\documentclass[10pt]{article} 

\usepackage[dvipsnames]{xcolor}

\usepackage[accepted]{tmlr}

\def\month{03}  
\def\year{2025} 
\def\openreview{\url{https://openreview.net/forum?id=h434zx5SX0}} 

\usepackage[utf8]{inputenc} 
\usepackage[T1]{fontenc}    
\usepackage{hyperref}       
\usepackage{url}            
\usepackage{booktabs}       
\usepackage{amsfonts}       
\usepackage{nicefrac}       
\usepackage{microtype}      

\usepackage{multirow}
\usepackage{enumitem}
\setlist[enumerate]{nosep}
\usepackage{bbm}
\usepackage{graphicx}

\usepackage{amsmath}
\usepackage{amssymb}
\usepackage{mathtools}
\usepackage{amsthm}
\usepackage{thmtools}

\usepackage{tabularx}
\usepackage{caption}

\theoremstyle{plain}
\newtheorem{theorem}{Theorem}[section]
\newtheorem{proposition}[theorem]{Proposition}
\newtheorem{lemma}[theorem]{Lemma}

\newtheorem{claim}[theorem]{Claim}
\theoremstyle{definition}
\newtheorem{definition}[theorem]{Definition}

\theoremstyle{remark}
\newtheorem{remark}[theorem]{Remark}

\newcommand{\ours}{\textsc{Sea}}
\newcommand{\lto}{\leftarrow}
\newcommand{\RR}{\mathbb{R}}

\newcommand{\vsq}{N\times N}
\newcommand{\indep}{\perp \!\!\! \perp}

\newcommand\numberthis{\addtocounter{equation}{1}\tag{\theequation}}

\title{Sample, estimate, aggregate:\\
A recipe for causal discovery foundation models}

\author{\name Menghua Wu \email {rmwu\{at\}mit.edu} \\
\addr Department of Computer Science, Massachusetts Institute of Technology
\ANDD
\name Yujia Bao \email yujia.bao\{at\}accenture.com \\
\addr Center for Advanced AI, Accenture
\ANDD
\name Regina Barzilay \email regina\{at\}csail.mit.edu \\
\addr Department of Computer Science, Massachusetts Institute of Technology
\ANDD
\name Tommi S. Jaakkola \email tommi\{at\}csail.mit.edu \\
\addr Department of Computer Science, Massachusetts Institute of Technology
}

\begin{document}

\maketitle

\begin{abstract}
Causal discovery, the task of inferring causal structure from data, has the potential to uncover mechanistic insights from biological experiments, especially those involving perturbations.
However, causal discovery algorithms over larger sets of variables tend to be brittle against misspecification or when data are limited.
For example, single-cell transcriptomics measures thousands of genes, but the nature of their relationships is not known, and there may be as few as tens of cells per intervention setting.
To mitigate these challenges, we propose a foundation model-inspired approach: a supervised model trained on large-scale, synthetic data to predict causal graphs from summary statistics --- like the outputs of classical causal discovery algorithms run over subsets of variables and other statistical hints like inverse covariance.
Our approach is enabled by the observation that typical errors in the outputs of a discovery algorithm remain comparable across datasets.
Theoretically, we show that the model architecture is well-specified,
in the sense that it can recover a causal graph consistent with graphs over subsets.
Empirically, we train the model to be robust to misspecification and distribution shift using diverse datasets.
Experiments on biological and synthetic data confirm that this model generalizes well beyond its training set, runs on graphs with hundreds of variables in seconds, and can be easily adapted to different underlying data assumptions.\footnote{
Our code is available at \url{https://github.com/rmwu/sea}.
}
\end{abstract}

\section{Introduction}
\label{intro}

A fundamental aspect of scientific research is to discover and validate causal hypotheses involving variables of interest.
Given observations of these variables, the goal of causal discovery algorithms is to extract such hypotheses in the form of directed graphs, in which edges denote causal relationships~\citep{causation-prediction-search}.
There are several challenges to their widespread adoption in basic science.
The core issue is that the correctness of these algorithms is tied to their assumptions on the data-generating processes, which are unknown in real applications.
In principle, one could circumvent this issue by exhaustively running discovery algorithms with different assumptions and comparing their outputs with surrogate measures that reflect graph quality~\citep{compatibility}.
However, this search would be costly: current algorithms must be optimized from scratch each time,
and they scale poorly to the graph and dataset sizes present in modern scientific big data~\citep{gwps}.

Causal discovery algorithms follow two primary approaches that differ in their treatment of the
causal graph.
Discrete search algorithms explore the super-exponential space of graphs by proposing and evaluating changes to a working graph~\citep{10.3389/fgene.2019.00524}.
While these methods are fast on small graphs, the combinatorial space renders them intractable for exploring larger structures.
An alternative is to frame causal discovery as a continuous optimization over weighted adjacency matrices.
These algorithms either fit a generative model to the data and extract the causal graph as a parameter~\citep{notears},
or train a supervised learning model on simulated data~\citep{petersen_ramsey_ekstrøm_spirtes_2023}.
However, these methods may be less robust beyond simple settings, and their optimization can be nontrivial~\citep{ng2024structure}.

In this work, we present \ours{}: Sample, Estimate, Aggregate, a supervised causal discovery framework that aims to perform well even when data-generating processes are unknown, and to easily incorporate prior knowledge when it is available.
We train a deep learning model to predict causal graphs from two types of statistical descriptors: the estimates of classical discovery algorithms over small subsets, and graph-level statistics.
Each classical discovery algorithm outputs a representation of a graph's equivalency class, and the types of errors that it makes may be comparable across datasets.
On the other hand, statistics like correlation or inverse covariance are fast to compute, and strong indicators for a graph's overall connectivity.
Theoretically, we illustrate a simple algorithm for recovering larger causal graphs that are consistent with estimates over subsets, and we show that there exists a set of model parameters that can map sets of subgraph estimates to the correct global graph.
Empirically, our training procedure forces the model to predict causal graphs across diverse synthetic data, including on datasets that are misaligned with the discovery algorithms' assumptions, or when insufficient subsets are provided.

Our experiments probe three qualities that we view a foundation model should fulfill, with thorough comparison to three classical baselines and five deep learning approaches.
Specifically, we assess the framework's ability to generalize to unseen and out-of-distribution data; to steer predictions based on prior knowledge; and to perform well in low-data regimes.
\ours{} attains the state-of-the-art results on synthetic and real causal discovery tasks, while providing 10-1000x faster inference.
To incorporate domain knowledge, we show that it is possible to swap classic discovery algorithms at inference time, 
for significant improvements on datasets that match the assumptions of the new algorithm.
Our models can also be finetuned at a fraction of the training cost to accommodate new graph-level statistics that capture different (e.g. nonlinear) relationships.
We extensively analyze \ours{} in terms of low-data performance, scalability, causal identifiability, and other design choices.
To conclude, while our experiments focus on specific algorithms and architectures, this work presents a blueprint for designing causal discovery foundation models, in which sampling heuristics, classical causal discovery algorithms, and summary statistics are the fundamental building blocks.
\section{Background and related work}
\label{related}

\subsection{Causal structure learning}


A \textbf{causal graphical model} is a directed graph $G=(V,E)$, where each node $i\in V$ corresponds to a random variable $X_i\in X$
and each edge $(i,j)\in E$ represents a causal relationship from $X_i\to X_j$.
There are a number of assumptions that relate data distribution $P_X$ to $G$~\citep{causation-prediction-search,gies}, which determine whether $G$ is uniquely identifiable, or identifiable up to an equivalence class.
In this work, we assume causal sufficiency -- that is, $V$ contains all the parents $\pi_i$ of every node $i$.
Causal graphical models
can be used, along with other information, to compute the downstream consequences of interventions.
An intervention on node $i$ refers to setting conditional $P(X_i\mid X_{\pi_i})$ to a different distribution $\tilde{P}(X_i\mid X_{\pi_i})$.
Our experiments cover the observational case (no interventions) and the case with perfect interventions on each node, i.e. for all nodes $i$, we have access to data where we set $P(X_i\mid X_{\pi_i}) \leftarrow \tilde{P}(X_i)$.


Given a dataset $D\sim P_X$, the goal of \textbf{causal structure learning} (causal discovery) is to recover $G$.
There are two main challenges.
First, the number of possible graphs is super-exponential in
the number of nodes $N$, so algorithms
must navigate this combinatorial search space efficiently.
Second, depending on data availability and the underlying data generation process, causal discovery algorithms may or may not be able to recover $G$ in practice. In fact, many algorithms are only analyzed in the infinite-data
regime and require at least thousands of data samples for reasonable empirical performance~\citep{causation-prediction-search, dcdi}.

\textbf{Discrete search algorithms}
encompass diverse strategies for traversing the combinatorial space of possible graphs.
Constraint-based algorithms
are based on conditional independence tests, whose discrete results inform of the presence or absence of edges, and
whose statistical power depends directly on dataset size~\citep{10.3389/fgene.2019.00524}.
These include the observational FCI and PC algorithms~\citep{fci},
and the interventional JCI algorithm~\citep{mooij2020joint}.
Score-based methods define 
a continuous score that guides the search through the discrete space of valid graphs, where the true graph lies at the optimum.
Examples include GES~\citep{ges}, GIES~\citep{gies}, CAM~\citep{cam}, and IGSP~\citep{igsp}.
Finally, semi-parametric methods such as LiNGAM~\citep{lingam} or additive noise models~\citep{hoyer2008nonlinear} exploit asymmetries implied by the model class to identify graph connectivity and causal ordering.

\textbf{Continuous optimization approaches} recast the
combinatorial space of graphs into a continuous space
of weighted adjacency matrices.
Many of these works train a generative model to learn
the empirical data distribution,
which is parameterized through the adjacency matrix~\citep{notears, grandag, dcdi}.
Others focus on properties related to the empirical
data distribution, such as a relationship between
the underlying graph and the Jacobian of the learned model~\citep{jacobian}, or between the Hessian of the
data log-likelihood and the topological order~\citep{diffan}.
Finally, amortized inference approaches~\citep{ke2022learning,avici,petersen_ramsey_ekstrøm_spirtes_2023} frame causal discovery as a supervised learning problem,
where a neural network is trained to predict (synthetic) graphs from (synthetic) datasets.
To incorporate new information, current supervised methods must simulate new datasets and re-train.
In addition, these models that operate on raw observations scale poorly to larger datasets.
Since this direction is most similar to our own, we include further comparisons in \ref{subsec:baselines}.

\begin{figure*}[t]
\begin{center}
\centerline{\includegraphics[width=\linewidth, page=2]{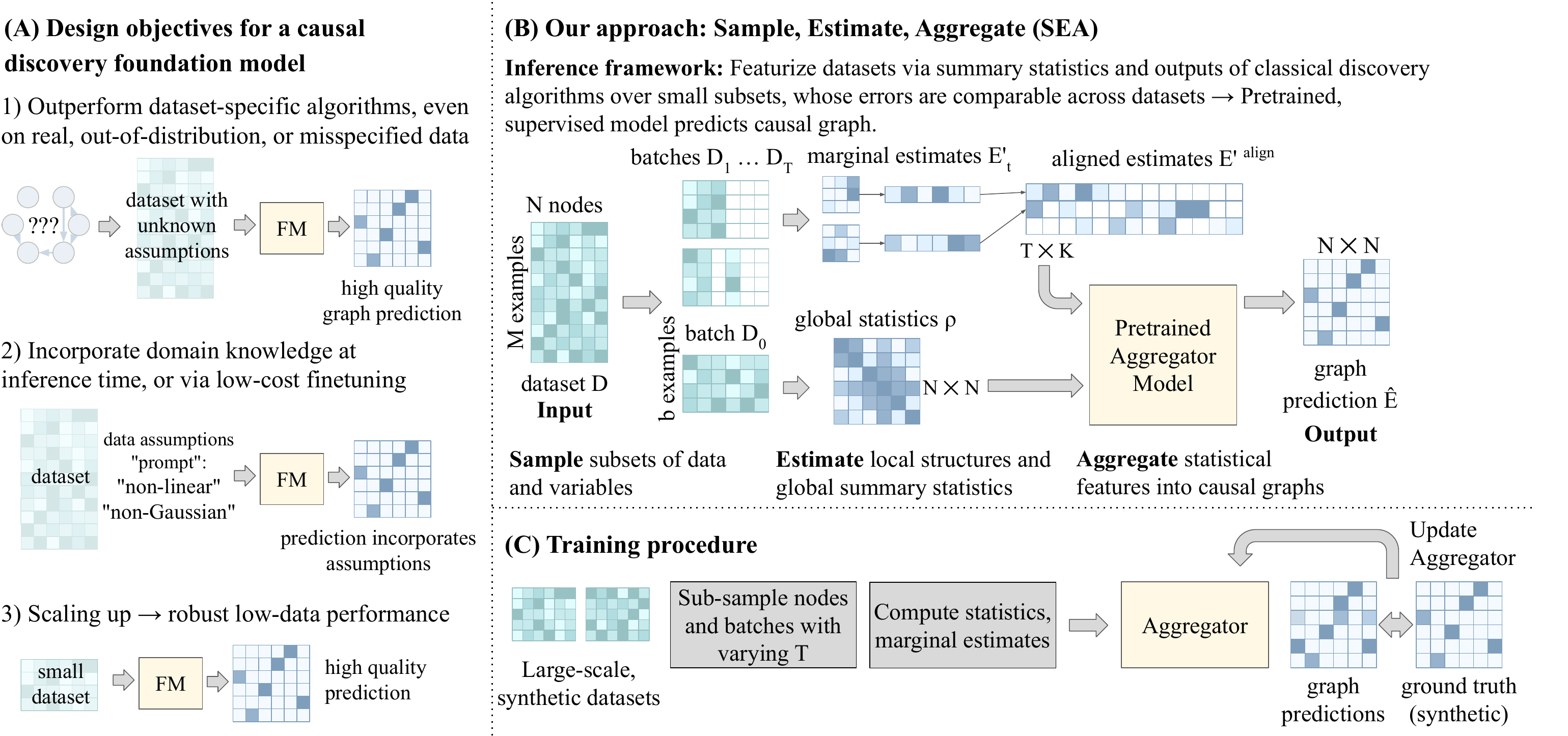}}
\vspace{-0.1in}
\caption{
Overview of our goals and approach.
(A) Criteria we aim to fulfill.
(B-C) Inference and training procedure.
Green: raw data. Blue: graph / features. Yellow: Learned. Gray: Stochastic, but not learned.}
\label{fig:overview}
\end{center}
\vspace{-0.4in}
\end{figure*}

\subsection{Foundation models}
\label{sec:fm}

The concept of foundation models
has revolutionized the machine learning workflow
in a variety of disciplines:
instead of training domain-specific models from scratch, we start from a pretrained, general-purpose model~\citep{bommasani2021opportunities}.
This work describes a blueprint for designing ``foundation models'' in the context of causal discovery.
The precise definition of a foundation model varies by application, but we aim to fulfill the following properties (Figure~\ref{fig:overview}A), enjoyed by modern text and image foundation models~\citep{clip, gpt}.
\begin{enumerate}
\item A foundation model should enable us to outperform domain-specific models trained from scratch, even if the former has never seen similar tasks during training~\citep{radford2019language}.
In the context of causal discovery, we would like to train a model that outperforms any individual algorithm on real, misspecified, and/or out-of-distribution datasets.
\item It should be possible to explicitly steer a foundation model's behavior towards better performance on new tasks, either directly at inference time, e.g. ``prompting''~\citep{reynolds2021prompt}, or at low cost compared to pretraining~\citep{ouyang2022training}.
Here, we would like to easily change our causal discovery algorithm's ``assumptions'' regarding the data, e.g. by incorporating the knowledge of non-linearity, non-Gaussianity.
\item Scaling up a foundation model should lead to improved performance in few-shot or data-poor regimes~\citep{gpt}.
This aspect we analyze empirically.
\end{enumerate}
In the following sections, we will revisit these desiderata from both the design and experimental perspectives.

\section{Methods}
\label{methods}

Sample, Estimate, Aggregate (\ours{}) is a supervised causal discovery framework built upon the intuition that summary statistics and marginal estimates (outputs of a classical causal discovery algorithm on subsets of nodes) provide useful hints towards global causal structure.
In the following sections, we describe how these statistics are efficiently estimated from data (Section~\ref{subsec:framework}), and how we train a neural network to predict causal graphs from these inputs (Section~\ref{subsec:training}).
We expand upon the model architecture in Section~\ref{subsec:model}, and conclude with theoretical motivation for this architecture in Section~\ref{subsec:theory}.

\subsection{Inference procedure}
\label{subsec:framework}

Given a new dataset $D\in\RR^{M\times N}$, we sample small batches of nodes and observations; estimate global summary statistics and local subgraphs; and aggregate these information with a trained neural network (Figure \ref{fig:overview}B).

\textbf{Sample:} takes as input dataset $D$; and outputs data batches $\{ D_0, D_1, \dots , D_T \}$ and node subsets $\{ S_1,\dots, S_T\}$.
\begin{enumerate}
    \item Sample $T+1$ batches of $b \ll M$ observations uniformly at random from $D$.
    \item Initialize selection scores $\alpha\in\RR^{\vsq}$ (e.g. correlation or inverse covariance, computed over $D_0$ or $D$).
    \item Sample $T$ node subsets of size $k$. Each subset $S_t$ is constructed iteratively as follows.
    \begin{enumerate}
    \item The initial node is sampled with probability proportional to $\sum_{j\in V} \alpha_{i,j}$.
    \item Each subsequent node is added with probability proportional to $\sum_{j\in S_t} \alpha_{i,j}$ (prioritizing connections to nodes that have already been sampled),
    until $\| S_t\| = k$.
    \item We update $\alpha$, down-weighting $\alpha_{i,j}$ proportional to the number of times $i,j$ have been selected.
    \end{enumerate}
\end{enumerate}
We include further details and analyze alternative strategies for sampling nodes in~\ref{subsec:sampling}.

\textbf{Estimate:} takes as inputs data batches, node subsets, and (optionally) intervention targets; and outputs global statistics $\rho$ and marginal estimates $\{E'_1,\dots, E'_T\}$.
\begin{enumerate}
    \item Compute global statistics $\rho\in\mathbb{R}^{\vsq}$ over $D_0$ (e.g. correlation or inverse covariance).
    \item Run discovery algorithm $f$ to obtain marginal estimates $f(D_t[S_t]) = E'_t$ for $t = 1\dots T$.
\end{enumerate}
We use $D_t[S_t]$ to denote the observations in $D_t$ that correspond only to the variables in $S_t$.
Each estimate $E'_t$ is a $k\times k$ adjacency matrix, corresponding to the $k$ nodes in $S_t$.

\textbf{Aggregate:} takes as inputs global statistics, marginal estimates, and node subsets.
A trained aggregator model outputs the predicted global causal graph $\hat E \in (0,1)^{N\times N}$ (Section~\ref{subsec:model}).

\subsection{Training procedure}
\label{subsec:training}

The training procedure mirrors the inference procedure (Figure~\ref{fig:overview}C).
We assume access to pairs of simulated datasets and graphs, $(D, G)$,
where each dataset $D$ is generated by a parametric model, whose dependencies are given by graph $G$ (Section~\ref{data-settings}).
Datasets $D$ are summarized into global statistics and marginal estimates via the sampling and estimation steps.
The resultant features are input to the aggregator (neural network), which is supervised by graphs $G$.

We trained two aggregator models, which employed the \textsc{Fci} algorithm with the Fisherz test and \textsc{Gies} algorithm with the Bayesian information criterion~\citep{bic}.
Both estimation algorithms were chosen for speed, but they differ in their assumptions, discovery strategies, and output formats.
Though we only trained two models, alternate estimation algorithms that produce the same output type\footnote{PC~\citep{causation-prediction-search} also predicts CPDAGs~\citep{cpdag}, so its outputs may be input to the \textsc{Gies} model, while FCI's PAG outputs cannot. 
} may be used at inference time (experiments in Section~\ref{subsec:finetune}).
The training dataset contains both data that are correctly and incorrectly specified (Section \ref{data-settings}), so the aggregator is forced to predict the correct graph regardless.
In addition, each training instance samples a random number of marginal estimates, which might not cover every edge.
As a result, the aggregator must extrapolate to unseen edges using the available estimates and the global statistics. For example, if two variables have low correlation, and they are located in different neighborhoods of the already-identified graph, it may be reasonable to assume that they are unconnected.

\begin{figure*}[t]
\begin{center}
\vskip -0.15in
\centerline{\includegraphics[width=\textwidth, page=1]{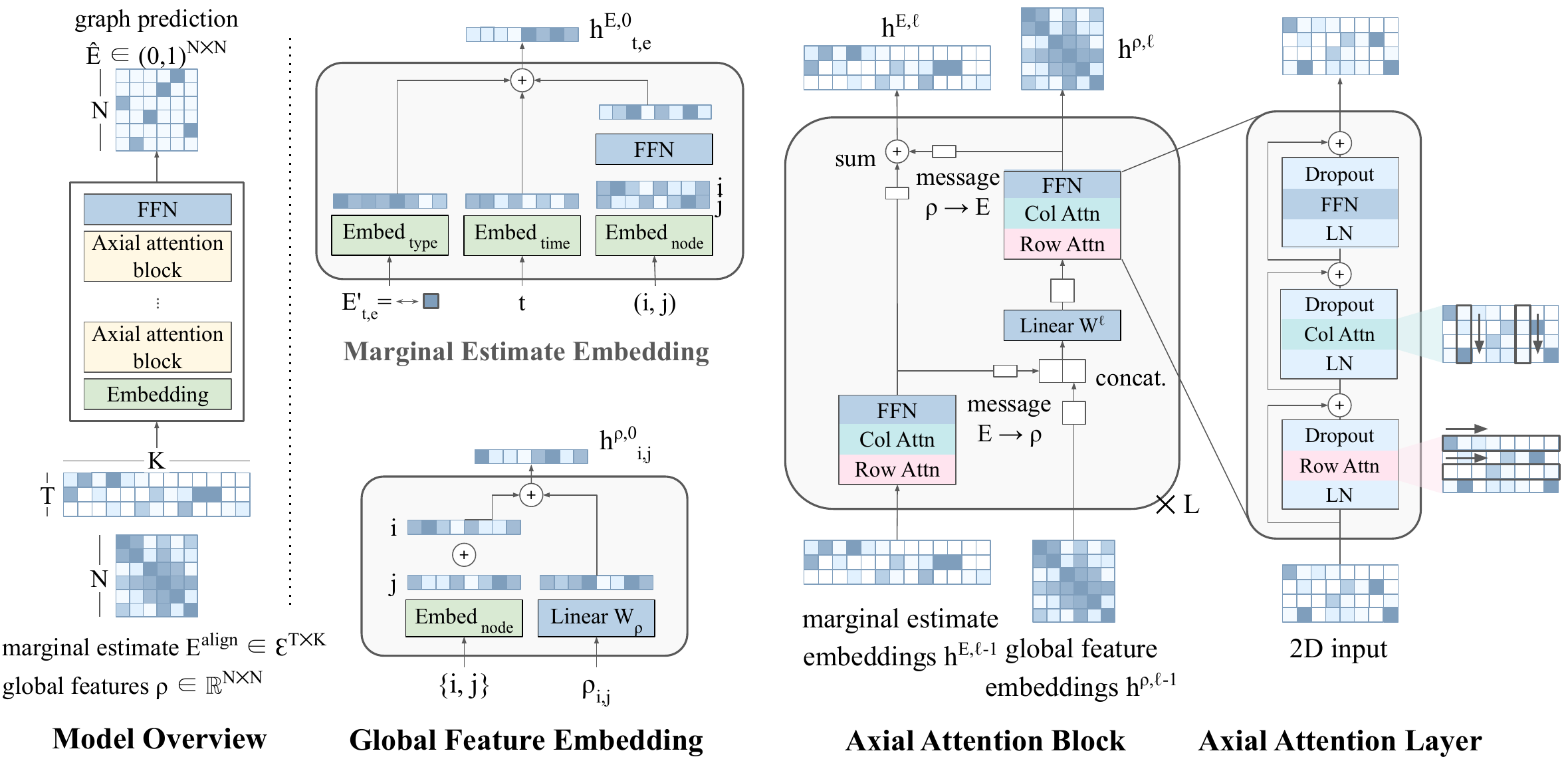}}
\vskip -0.1in
\caption{Aggregator architecture. Marginal graph estimates and global statistics are embedded into the model dimension. 1D positional embeddings are added along both rows and columns.
Embedded features pass through a series
of axial attention blocks, which attend to the marginal and global features.
Final layer global features pass through
a feedforward network to predict the causal graph.}
\label{model}
\end{center}
\vskip -0.35in
\end{figure*}

\subsection{Model architecture}
\label{subsec:model}

The aggregator is a neural network that takes as input: global statistics $\rho\in\RR^{\vsq}$,
marginal estimates $E'_{1\dots T} \in \mathcal{E}^{T\times k\times k}$,
and node subsets $S_{1\dots T}\in[N]^{T\times k}$ (Figure~\ref{model}), where
$\mathcal{E}$ is the set of output edge types for
the causal discovery algorithm $f$.\footnote{E.g. ``no relationship,'' ``$X$ causes $Y$,'' ``$X$ is not a descendent of $Y$''}

We project global statistics into the model dimension
via a learned linear projection matrix $W_\rho: \RR\to\RR^d$,
and we embed edge types via a learned embedding $\text{ebd}_\mathcal{E}:\mathcal{E}\to\mathbb{R}^d$.
To collect estimates of the same edge over all subsets, we align entries of $E'_{1\dots T}$ into
$E^{' \textrm{align}}_T \in \mathcal{E}^{T\times K}$
\begin{equation}
    E^{' \textrm{align}}_{t,e=(i,j)} = \begin{cases}
        E'_{t,i,j} & \textrm{ if } i \in S_t, j \in S_t\\
        0 & \textrm{ otherwise}
    \end{cases}
\end{equation}
where $t$ indexes into the subsets, $e=1\dots K$ indexes into the set of unique edges, i.e. the union of pairs $(i,j)$ over all $E'_t$.
We add learned 1D positional embeddings along both dimensions of each input,
\begin{align*}
    \textrm{pos-ebd}(\rho_{i,j}) =~& \textrm{ebd}_\text{node}(i') + \textrm{ebd}_\text{node}(j') \\
    \textrm{pos-ebd}(E^{' \textrm{align}}_{t,e}) =~&
        \textrm{ebd}_\text{time}(t)
        + \textrm{FFN}([\textrm{ebd}_\text{node}(i'), \textrm{ebd}_\text{node}(j')]) 
\end{align*}
where $i',j'$ index into a random permutation on $V$
for invariance to node permutation and graph size.\footnote{
The random permutation $i' = \sigma(V)_i$
allows us to avoid updating positional embeddings
of lower order positions more than higher order ones,
due to the mixing of graph sizes during training.}
Due to the (a)symmetries of their inputs,
$\textrm{pos-ebd}(\rho_{i,j})$ is symmetric, while $\textrm{pos-ebd}(E^{' \textrm{align}}_{t,e})$ considers
the node ordering.
In summary, the inputs to our axial attention blocks are
\begin{align}
    h^\rho_{i,j} &=
    (W_\rho \rho)_{i,j}
    + \textrm{pos-ebd}(\rho_{i,j}) \\
    h^E_{t,e} &=
    \textrm{ebd}_\mathcal{E}(E^{' \textrm{align}}_{t,e})
    + \textrm{pos-ebd}(E^{' \textrm{align}}_{t,e})
\end{align}
for $i,j\in[N]^2$, $t\in[T]$, $e\in[K]$.
Note that attention is permutation invariant, so positional embeddings are \emph{required} for the model to know which edges belong to the same subset, or what each edge's endpoints endpoints are.

\paragraph{Axial attention} An axial attention block
contains two axial attention layers (marginal estimates, global statistics)
and a feed-forward network
(Figure~\ref{model}, right).
Given a 2D input,
an axial attention layer
attends first along the rows, then along the columns.
For example, on a matrix of size \texttt{(R,C,d)},
one pass of the axial attention layer is equivalent to
running standard self-attention along \texttt{C}
with batch size \texttt{R}, followed by the reverse.
For marginal estimates, \texttt{R} is the number of subsets $T$, and \texttt{C} is the number of unique edges $K$.
For global statistics, \texttt{R} and \texttt{C} are both the total number of vertices $N$.

Following~\cite{msa-transformer},
each self-attention mechanism is preceded by layer normalization
and followed by dropout, with residual connections to the input,
\begin{equation}
    x = x + \text{Dropout}(\text{Attn}(\text{LayerNorm}(x))).
\end{equation}

We pass messages between the marginal and global layers to propagate information.
Let $\phi_{E,\ell}$ be marginal layer $\ell$,
let $\phi_{\rho,\ell}$ be global layer $\ell$,
and let $h^{\cdot,\ell}$ denote the hidden representations out of layer $\ell$.
The marginal to global message $m^{E\to\rho}\in\RR^{N\times N\times d}$ contains representations of each edge averaged over subsets,
\begin{align}
    m^{E\to\rho,\ell}_{i,j} &= \begin{cases}
        \frac{1}{T_e} \sum_{t}
            h^{E,\ell}_{t,e=(i,j)} & \textrm{ if }
            \exists S_t, i,j\in S_t \\
        \epsilon & \textrm{ otherwise.}
    \end{cases}
\end{align}
where $T_e$ is the number of $S_t$ containing $e$, and missing entries are padded to learned constant $\epsilon$.
The global to marginal message $m^{\rho\to E}\in\RR^{K\times d}$ is simply the hidden representation itself,
\begin{align}
    m^{\rho\to E,\ell}_{t,e=(i,j)} &= 
        h^{\rho,\ell}_{i,j}.
\end{align}
We update representations based on these messages as follows.
\begin{align}
    h^{E,\ell} &= \phi_{E,\ell}(h^{E,\ell-1})
    && \text{\small (marginal feature)} \\
    h^{\rho,\ell-1} &\lto W^{\ell} \left[h^{\rho,\ell-1}, m^{E\to\rho,\ell}\right] 
    && \text{\small (marginal to global)}\\
    h^{\rho,\ell} &= \phi_{\rho,\ell}(h^{\rho,\ell-1})
    && \text{\small (global feature)}\\
    h^{E,\ell} &\lto h^{E,\ell} + m^{\rho\to E,\ell}
    && \text{\small (global to marginal)}
\end{align}
$W^{\ell}\in\mathbb{R}^{2d\times d}$ is a learned linear
projection, and $[\cdot]$ denotes concatenation.

\paragraph{Graph prediction}

For each pair of vertices $i\ne j\in V$,
we predict $e=0,1$, or $2$ for no edge, $i\to j$, and $j\to i$.
We do not additionally enforce that our predicted graphs are acyclic, similar in spirit to~\cite{enco}.
Given the output of the final axial attention block
$h^\rho$, we compute logits
\begin{equation}
    z_{\{i,j\}} = \text{FFN}\left(
     \left[ h^\rho_{i,j}, h^\rho_{j,i} \right] \right)
     \in \mathbb{R}^3
\end{equation}
which correspond to probabilities after softmax normalization.
The overall output $\hat E\in\{0, 1\}^{\vsq}$
is supervised by the ground truth $E$.
Our model is trained with cross entropy loss and L2 regularization.

\paragraph{Implementation details}
\label{subsec:details}

Unless otherwise noted, inverse covariance is used for the global statistic and selection score, due to its relationship to partial correlation.
We sample batches of size $b=500$ 
over $k=5$ nodes each (analysis in~\ref{sec:main_ablations}).
Our model was implemented with 4 layers with 8 attention heads and hidden dimension 64. Our model was trained using the AdamW optimizer with a learning rate of 1e-4~\citep{adamw}.
See~\ref{subsec:pretraining} for additional details about hyperparameters.

\paragraph{Complexity}
The aggregator should be be invariant to node labeling, while maintaining the order of sampled subsets, so attention-based architectures were a natural choice~\citep{attention}.
If we concatenated $\rho$ and $E'_{1\dots T}$ into a length $N^2T$ input, quadratic-scaling attention would cost $O(N^4T^2)$.
Instead, we opted for axial attention blocks, which attend along each of the three axes separately in $O(N^3 T + N^2T^2)$.
Both are parallelizable on GPU,
but the latter is more efficient, especially on larger $N$.

\subsection{Theoretical interpretation}
\label{subsec:theory}

\textbf{Marginal graph resolution~} It is well-established that estimates of causal graphs over subsets of variables can be ``merged'' into consistent graphs over their union~\citep{compatibility,ion, cdmini}.
In \ref{resolver-toy-example} and \ref{subsec:resolve-proof}, we describe a simple algorithm towards this task, based on the intuition that edges absent from the global graph should be absent from at least one marginal estimate, and that v-structures present in the global graph are present in the marginal estimates.
We then prove Theorem~\ref{thm:resolve-exists}, which states that the axial attention architecture is well-specified as a model class.
That is, there exists a setting of its weights that can map marginal estimates into global graphs.
Our construction in \ref{subsec:power-proof} follows the same reasoning steps as the simple algorithm in \ref{subsec:resolve-proof}, and it provides realistic bounds on the model size.

\begin{restatable}{theorem}{thmresolve}
\label{thm:resolve-exists}
    Let $G=(V,E)$ be a directed acyclic graph
    with maximum degree $d$.
    For $S\subseteq V$,
    let $E'_S$ denote the marginal estimate
    over $S$.
    Let $\mathcal{S}_d$ denote the superset
    that contains all subsets $S\subseteq V$
    of size at most $d$.
    Given $\{ E'_S \}_{S \in \mathcal{S}_{d+2}}$,
    a stack of $L$ axial attention blocks
    has the capacity to recover $G$'s skeleton and
    v-structures in $O(N)$ width,
    and propagate orientations on paths of $O(L)$ length.
\end{restatable}

There are two practical considerations that motivate a framework like \ours{}, instead of directly running classic reconciliation algorithms.
First, many of these algorithms rely on specific characterizations of the data-generating process, e.g. linear non-Gaussian~\citep{cdmini}.
While our proof does not constrain the causal mechanisms or exogenous noise, it assumes that the marginal estimates are correct.
These assumptions may not hold on real data.
However, the failure modes of any particular causal discovery algorithm may be similar across datasets and can be corrected using statistics that capture richer information.
For example, an algorithm that assumes linearity will make (predictably) asymmetric mistakes on non-linear data and underestimate the presence of edges.
However, we may be able to recover nonlinear relationships with statistics like distance correlation~\citep{dcor}.
By training a deep learning model to reconcile marginal estimates and interpret global statistics, we are less sensitive to artifacts of sampling and discretization (e.g. p-value thresholds, statistics $\lessgtr$ 0).
The second consideration is that checking a combinatorial number of subsets is wasteful on smaller graphs and infeasible on larger graphs.
In fact, if we only leverage marginal estimates, we must check at least $O(N^2)$ subsets to cover each edge at least once.
To this end, the classical Independence Graph algorithm~\citep{causation-prediction-search} motivates statistics such as inverse covariance to initialize the undirected skeleton and reduce the number of independence tests required.
This allows us to use marginal estimates more efficiently, towards answering orientation questions.
We verify this latter consideration in Section~\ref{sec:main_ablations}, where we empirically quantify the number of estimates a global statistic is ``worth.''


\textbf{On identifiability~} The primary goal of this paper is to develop a practical framework for causal discovery, especially when data assumptions are unknown.
Instead of focusing on the identifiability of any particular setting,
we provide these interpretations of our model's outputs, and show empirically that our model respects classic identifiability theory (Section~\ref{sec:identifiability}).
The model will always output an orientation for all edges, but the graph can be interpreted as one member of an equivalence class.
Metrics can be computed with respect to either the ground truth graph (if identifiable) or the inferred equivalence class, e.g. the implied CPDAG.
When data do not match the estimation algorithm's assumptions, performance is inherently an empirical question, and
we show empirically that our model still does well (Section~\ref{subsec:main-results}).

\section{Experimental setup}
\label{sec:experiment-setup}

Our experiments aim to address the three desiderata proposed in Section~\ref{sec:fm} -- namely, generalization, adaptability, and emergent few-shot behavior.
These experiments span both real and synthetic data.
Real experiments quantify the practical utility of this framework, while synthetic experiments allow us to probe and compare each design choice in a controlled setting.

\subsection{Datasets}
\label{data-settings}

We pretrained \ours{} models on 6,480 synthetic datasets, which constitute approximately 280 million individual observations, each of 10-100 variables.\footnote{
3 mechanisms, 3 graph sizes, 4 sparsities, 2 topologies, $1000N$ examples, 90 datasets $\to$ 280,800,000 examples.
For a sense of scale, single cell foundation models are trained on 300K~\citep{rosen2024toward} to 30M cells~\citep{cui2024scgpt}.
}
To assess generalization and robustness,
we evaluate on unseen in-distribution and out-of-distribution synthetic datasets, as well as two real biological datasets~\citep{sachs,gwps}, using the versions from~\cite{igsp,causalbench}.
To probe for emergent few-shot behavior, we down-sample both the training and testing sets.
We also include experiments on simulated mRNA datasets with unseen datasets in Appendix~\ref{subsec:sergio} \citep{sergio}.

The training datasets were constructed by 1) sampling Erd\H{o}s-R\'enyi and scale free graphs with $N=10, 20, 100$ nodes and $E=N, 2N, 3N, 4N$ expected edges; 2) sampling random instantiations of causal mechanisms (Linear, NN with additive/non-additive Gaussian noise); and 3) iteratively sampling observations in topological order (details in Appendix~\ref{subsec:causal-mechanisms}).
From every causal graphical model (steps 1-2), we generated two datasets, each with $1000N$ points: either all observational, or split equally among regimes (observational and perfect single-node interventions on all nodes).
All models that can accommodate interventions were run on the interventional datasets, with complete knowledge of the intervention target identities.
The remaining were run on the observational datasets.
We generated 90 training, 5 validation, and 5 testing datasets for each combination. For testing, we also sampled out-of-distribution datasets with 1) Sigmoid and Polynomial mechanisms with Gaussian noise; and 2) Linear with additive non-Gaussian noise.

\subsection{Metrics}

We report standard causal discovery metrics. These include both discrete and continuous metrics, as neural networks can be notoriously uncalibrated~\citep{calibration},
and arbitrary discretization thresholds may impact the robustness of findings~\citep{ng2024structure,emergence}.
For all continuous metrics, we exclude the diagonal since several baselines manually set it to zero~\citep{dcdi, dcdfg}.

\textbf{SHD:} Structural Hamming distance is the minimum number of edge edits required to match two graphs~\citep{shd} (predicted and true DAGs, or the implied CPDAGs).
Discretization thresholds are as published or default to 0.5.

\textbf{mAP:} Mean average precision computes the area under precision-recall curve per edge and averages over the graph (predicted and true DAGs, or undirected skeletons). The random guessing baseline depends on the positive rate.

\textbf{AUC:} Area under the ROC curve~\citep{auc} computed per edge (binary prediction) and averaged over the graph.
For each edge, 0.5 indicates random guessing, while 1 indicates perfect performance.

\textbf{Orientation accuracy:} We compute the accuracy of edge orientations as
\begin{equation}
    \text{OA} = \frac{
        \sum_{(i,j)\in E} \mathbbm{1} \{ P(i,j) > P(j,i) \}
    }{
        \| E \|
    }.
\end{equation}
Since OA is normalized by $\|E\|$, it is invariant to the positive rate.
In contrast to orientation F1~\citep{deci},
it is also invariant to the assignment of forward/reverse edges as 1/0.

\subsection{Baselines}

We compare against several deep learning and classical baselines.
All baselines were trained and/or run from scratch on each testing dataset using their published code and hyperparameters, except \textsc{Avici} (their recommended checkpoint was trained on their synthetic train and test sets after publication, Appendix~\ref{subsec:baselines}).

\textbf{DCDI} \citep{dcdi} extracts the causal graph as a parameter of a generative model. The \textsc{G} and \textsc{Dsf} variants use Gaussian or deep sigmoidal flow likelihoods, respectively.
\textbf{DCD-FG} \citep{dcdfg} follows \textsc{DCDI-G}, but factorizes the graph into a product of two low-rank matrices for scalability.
\textbf{\textsc{DiffAn}} \citep{diffan} uses the trained model's Hessian to obtain a topological ordering, followed by a classical pruning algorithm.
\textbf{\textsc{AVICI}} \citep{avici} uses an amortized inference approach to estimate $P(G\mid D)$ over a class of data-generating mechanisms via variational inference.
\textbf{\textsc{VarSort}} (a.k.a. ``sort and regress'')~\citep{varsort} sorts nodes by marginal variance and sparsely regresses nodes based on their predecessors.
This naive baseline is intended to reveal artifacts of synthetic data generation.
\textbf{\textsc{FCI}, \textsc{GIES}} quantify the predictive power of FCI and GIES estimates, when run over all nodes.
\textsc{VarSort}, \textsc{Fci}, and \textsc{Gies} were run using non-parametric bootstrapping~\citep{bootstrap}, with 100 estimates of the full graph (based on 1000 examples each), where the final prediction for each edge is its frequency of appearing as directed.
Since these methods are treated as oracles, the bootstrapping strategy was selected to maximize test performance. Visualizations of the bootstrapped graph can be found in Figures~\ref{fig:gies-resolve} and \ref{fig:fci-resolve}.

\section{Results}
We highlight representative results in each section, with additional experiments and analyses in Appendix \ref{sec:more-experiments}.
\begin{enumerate}
\item Section~\ref{subsec:main-results} examines the case where we have no prior knowledge about the data. Our models achieve high performance out-of-the-box, even when the data are misspecified or out-of-domain.
\item Section~\ref{subsec:finetune} focuses on the case where we do know (or can estimate) the class of causal mechanisms or exogenous noise. We show that adapting our pretrained models with this information at zero/low cost leads to substantial improvement and exceeds the best baseline trained from scratch.
\item Section~\ref{sec:identifiability} analyzes \ours{} predictions in context of classic identifiability theory.
In particular, we focus on the linear Gaussian case, and show that \ours{} approaches ``oracle'' performance (with respect to the MEC), while simply running a classic discovery algorithm cannot, on our finite datasets.
\item Section~\ref{sec:main_ablations} contains a variety of ablation studies. In particular, \ours{} exhibits impressive low-data performance, requiring only 400 samples to perform well on $N=100$ datasets.
We also ablate estimation hyperparameters and the contribution of marginal/global features.
\end{enumerate}

\begin{table*}[t]
\setlength\tabcolsep{3.5 pt}
\caption{Synthetic experiments.
Mean/std over 5 distinct Erd\H{o}s-R\'enyi graphs, with metrics relative to ground truth DAG.
\textsc{DiffAn}, \textsc{VarSort}, \textsc{Fci}, \textsc{Sea(Fci)} run on observational data only. Evaluation w.r.t. CPDAG and undirected skeleton in Tables~\ref{table:cpdag} and \ref{table:skeleton},
with additional baselines and ablations in Appendix~\ref{sec:more-experiments}.
$\dagger$ indicates o.o.d. setting.
$*$ indicates non-parametric bootstrapping.
Runtimes with 1 CPU and 1 V100 GPU.
}
\vspace{-0.15in}
\label{table:synthetic}
\begin{center}
\begin{small}
\begin{tabular}{ccl rr rr rr rr rr r}
\toprule
$N$ & $E$ & Model 
& \multicolumn{2}{c}{Linear} 
& \multicolumn{2}{c}{NN add.}
& \multicolumn{2}{c}{Sigmoid$^\dagger$}
& \multicolumn{2}{c}{Polynomial$^\dagger$} 
& Overall
\\
\cmidrule(l{\tabcolsep}){4-5}
\cmidrule(l{\tabcolsep}){6-7}
\cmidrule(l{\tabcolsep}){8-9}
\cmidrule(l{\tabcolsep}){10-11}
\cmidrule(l{\tabcolsep}){12-12}
&&& \multicolumn{1}{c}{mAP $\uparrow$} & \multicolumn{1}{c}{SHD $\downarrow$} & \multicolumn{1}{c}{mAP $\uparrow$} & \multicolumn{1}{c}{SHD $\downarrow$} & \multicolumn{1}{c}{mAP $\uparrow$} & \multicolumn{1}{c}{SHD $\downarrow$} & \multicolumn{1}{c}{mAP $\uparrow$} & \multicolumn{1}{c}{SHD $\downarrow$} & \multicolumn{1}{c}{Time(s) $\downarrow$} \\
\midrule

\multirow{9}{*}{20} & \multirow{9}{*}{20} & \textsc{Dcdi-G} & $0.59 \scriptstyle \pm .12$& $6.4 \scriptstyle \pm .9$& $0.78 \scriptstyle \pm .07$& $\textbf{3.0} \scriptstyle \pm .7$& $0.36 \scriptstyle \pm .06$& $42.7 \scriptstyle \pm .3$& $0.42 \scriptstyle \pm .08$& $10.4 \scriptstyle \pm .4$& 4735.7\\
&& \textsc{Dcdi-Dsf} & $0.66 \scriptstyle \pm .16$& $5.2 \scriptstyle \pm .3$& $0.69 \scriptstyle \pm .18$& $4.2 \scriptstyle \pm .5$& $0.37 \scriptstyle \pm .04$& $43.2 \scriptstyle \pm .4$& $0.26 \scriptstyle \pm .08$& $15.7 \scriptstyle \pm .2$& 3569.1\\
&& \textsc{DiffAn} & $0.19 \scriptstyle \pm .09$& $40.2 \scriptstyle \pm 4.4$& $0.16 \scriptstyle \pm .10$& $38.6 \scriptstyle \pm 3.1$& $0.29 \scriptstyle \pm .11$& $19.2 \scriptstyle \pm .6$& $0.09 \scriptstyle \pm .03$& $49.7 \scriptstyle \pm 4.6$& 434.3\\
&& \textsc{Avici} & $0.48 \scriptstyle \pm .17$& $17.2 \scriptstyle \pm .1$& $0.59 \scriptstyle \pm .09$& $10.8 \scriptstyle \pm .1$& $0.42 \scriptstyle \pm .13$& $17.2 \scriptstyle \pm .8$& $0.24 \scriptstyle \pm .08$& $18.4 \scriptstyle \pm .1$& 2.0\\
\cmidrule(l{\tabcolsep}){3-12}
&& \textsc{VarSort*} & $0.81 \scriptstyle \pm .08$& $10.0 \scriptstyle \pm .4$& $0.81 \scriptstyle \pm .15$& $6.6 \scriptstyle \pm .7$& $0.50 \scriptstyle \pm .13$& $16.1 \scriptstyle \pm .7$& $0.33 \scriptstyle \pm .13$& $17.1 \scriptstyle \pm .1$& 0.4\\
&& \textsc{Fci*} & $0.66 \scriptstyle \pm .07$& $19.0 \scriptstyle \pm .3$& $0.42 \scriptstyle \pm .19$& $17.4 \scriptstyle \pm .2$& $0.56 \scriptstyle \pm .08$& $18.5 \scriptstyle \pm .5$& $0.41 \scriptstyle \pm .14$& $18.9 \scriptstyle \pm .3$& 22.2\\
&& \textsc{Gies*} & $0.84 \scriptstyle \pm .08$& $7.4 \scriptstyle \pm .0$& $0.79 \scriptstyle \pm .07$& $9.0 \scriptstyle \pm .1$& $0.71 \scriptstyle \pm .10$& $12.5 \scriptstyle \pm .7$& $0.62 \scriptstyle \pm .09$& $13.7 \scriptstyle \pm .7$& 482.1\\
\cmidrule(l{\tabcolsep}){3-12}
&& \ours{} \textsc{(Fci)} & $\textbf{0.96} \scriptstyle \pm .03$& $3.2 \scriptstyle \pm .6$& $0.91 \scriptstyle \pm .04$& $5.0 \scriptstyle \pm .8$& $\textbf{0.85} \scriptstyle \pm .09$& $\textbf{6.7} \scriptstyle \pm .1$& $\textbf{0.69} \scriptstyle \pm .09$& $\textbf{9.8} \scriptstyle \pm .2$& 4.2\\
&& \ours{} \textsc{(Gies)} & $\textbf{0.97} \scriptstyle \pm .02$& $\textbf{3.0} \scriptstyle \pm .9$& $\textbf{0.94} \scriptstyle \pm .03$& $3.4 \scriptstyle \pm .4$& $0.84 \scriptstyle \pm .07$& $8.1 \scriptstyle \pm .8$& $\textbf{0.69} \scriptstyle \pm .12$& $10.1 \scriptstyle \pm .9$& 3.0\\

\midrule\midrule

\multirow{5}{*}{100} & \multirow{5}{*}{400} & \textsc{Dcd-Fg} & $0.05 \scriptstyle \pm .00$& $\scriptstyle 3068 \pm 131$& $0.07 \scriptstyle \pm .00$& $\scriptstyle 3428 \pm 154$& $0.13 \scriptstyle \pm .02$& $\scriptstyle 3601 \pm 272$& $0.12 \scriptstyle \pm .03$& $\scriptstyle 3316 \pm 698$& 1838.2\\
&& \textsc{Avici} & $0.12 \scriptstyle \pm .02$& $391 \scriptstyle \pm 8$& $0.17 \scriptstyle \pm .01$& $407 \scriptstyle \pm 19$& $0.10 \scriptstyle \pm .02$& $398 \scriptstyle \pm 11$& $0.03 \scriptstyle \pm .00$& $402 \scriptstyle \pm 19$& 9.3\\
\cmidrule(l{\tabcolsep}){3-12}
&& \textsc{VarSort*} & $0.80 \scriptstyle \pm .02$& $224 \scriptstyle \pm 10$& $0.18 \scriptstyle \pm .03$& $1139 \scriptstyle \pm 269$& $0.51 \scriptstyle \pm .05$& $350 \scriptstyle \pm 15$& $0.27 \scriptstyle \pm .04$& $380 \scriptstyle \pm 17$& 5.1\\
\cmidrule(l{\tabcolsep}){3-12}
&& \ours{} \textsc{(Fci)} & $0.84 \scriptstyle \pm .02$& $162 \scriptstyle \pm 12$& $0.04 \scriptstyle \pm .00$& $403 \scriptstyle \pm 16$& $0.63 \scriptstyle \pm .03$& $247 \scriptstyle \pm 17$& $0.34 \scriptstyle \pm .04$& $\textbf{325} \scriptstyle \pm 22$& 19.2\\
&& \ours{} \textsc{(Gies)} & $\textbf{0.91} \scriptstyle \pm .01$& $\textbf{116} \scriptstyle \pm 7$& $\textbf{0.27} \scriptstyle \pm .10$& $\textbf{364} \scriptstyle \pm 34$& $\textbf{0.69} \scriptstyle \pm .03$& $\textbf{218} \scriptstyle \pm 21$& $\textbf{0.38} \scriptstyle \pm .04$& $328 \scriptstyle \pm 22$& 3.1\\

\bottomrule
\end{tabular}
\end{small}
\end{center}
\vskip -0.2in
\end{table*}

\subsection{SEA generalizes to out-of-distribution, misspecified, and real datasets}
\label{subsec:main-results}

Table~\ref{table:synthetic} summarizes our controlled experiments on synthetic data.
\ours{} exceeds all baselines in the Linear case, which matches the models' assumptions exactly (causal discovery algorithms and inverse covariance).
In the misspecified (NN) or misspecified \emph{and} out-of-distribution settings (Sigmoid, Polynomial), \ours{} also attains the best performance in the vast majority of cases, even though \textsc{Dcdi} and \textsc{Avici} both have access to the raw data.
Furthermore, our models outperform \textsc{VarSort} in every single setting, while most baselines are unable to do so consistently.
This indicates that our models do not simply overfit to spurious features of the synthetic data generation process.

Table~\ref{table:k562} illustrates that we exceed baselines on single cell gene expression data from CausalBench~\citep{causalbench, gwps}.
Furthermore, when we increase the subset size to $b=2000$, we achieve very high precision (0.838) over 2834 predicted edges.
\ours{} runs within 5s on this dataset of 162k cells and $N=622$ genes, while the fastest naive baseline takes 5 minutes and the slowest deep learning baseline takes 9 hours (run in parallel on subsets of genes).
\subsection{SEA adapts to new data assumptions with zero to minimal finetuning}
\label{subsec:finetune}

We illustrate two strategies that allow us to use pretrained \ours{} models with different implicit assumptions.
First, if two causal discovery algorithms share the same output format, they can be used interchangeably for marginal estimation.
On observational, linear \emph{non}-Gaussian data, replacing the \textsc{Ges} algorithm with \textsc{Lingam}~\citep{lingam} is beneficial without any other change (Table~\ref{table:uniform}).
The same improvement can be observed on Polynomial and Sigmoid non-additive data, when running \textsc{Fci} with a polynomial kernel conditional independence test (\textsc{Kci}, \citet{kci}) instead of the Fisherz test, which assumes linearity (Table~\ref{table:polymix}).
In principle, different algorithms might make different mistakes, so this strategy could lead to out-of-distribution inputs for the pretrained aggregator.
While comparing across algorithms is not a primary focus of this work and requires further theoretical study,
we notice similar performance for alternate estimation algorithms with linear Gaussian assumptions (Table~\ref{table:frankenstein}), regardless of discovery strategy (GRaSP~\citep{grasp}).
The gap is larger for LiNGAM, which assumes non-Gaussianity (Table~\ref{table:frankenstein2}).

\begin{table}[t]
\begin{tabularx}{\linewidth}{XX}

\setlength\tabcolsep{3 pt}
\caption{
Results on K562 single cell data,
with STRING database (physical) as ground truth.
Baselines taken from~\cite{causalbench}.
}
\vspace{-0.1in}
\setlength\tabcolsep{3 pt}
\begin{small}\begin{center}
\label{table:k562}
\begin{tabular}[b]{l rrrr}
\toprule
Model
& P $\uparrow$ & R $\uparrow$
& F1 $\uparrow$
& Time(s) $\downarrow$\\
\midrule

\textsc{GRNboost} & 0.070 & \textbf{0.710} & 0.127 & 316 \\
\textsc{Gies} & 0.190 & 0.020 & 0.036 & 2350 \\
\midrule
\textsc{NoTears} & 0.080 & 0.620 & 0.142 & 32883 \\
\textsc{Dcdi-G} & 0.180 & 0.030 & 0.051 & 16561 \\
\textsc{Dcdi-Dsf} & 0.140 & 0.040 & 0.062 & 5709 \\
\textsc{Dcd-Fg} & 0.110 & 0.070 & 0.086 & 6368 \\
\midrule
\ours{} \textsc{(G)+Corr} & 0.491 & 0.109 & \textbf{0.179} & \textbf{4} \\
with $b=2000$ & \textbf{0.838} & 0.093 & 0.167 & 5 \\

\bottomrule
\end{tabular}
\end{center}\end{small}
&
\setlength\tabcolsep{3 pt}
\caption{
Performance on Sachs (\ref{sec:additional-real}) varies depending on implicit (\textsc{Avici} training set) and explicit (\ours{} variants) assumptions.
}
\vspace{-0.1in}
\label{table:sachs}
\setlength\tabcolsep{3 pt}
\begin{small}\begin{center}
\begin{tabular}[b]{l rrr}
\toprule
Model & mAP $\uparrow$ & AUC $\uparrow$ & SHD $\downarrow$ \\
\midrule

\textsc{Dcdi-Dsf} & $0.20$& $0.59$& $20.0$\\
\midrule
\textsc{Avici-L} & $0.35$& $0.78$& $20.0$\\
\textsc{Avici-R} & $0.29$& $0.65$& $18.0$\\
\textsc{Avici-L+R} & $\textbf{0.59}$& $\textbf{0.83}$& $14.0$\\
\midrule
\ours{} \textsc{(F)} & $0.23$& $0.54$& $24.0$\\
\textsc{+Kci} & $0.33$& $0.63$& $14.0$\\
\textsc{+Corr} & $0.41$& $0.70$& $15.0$\\
\textsc{+Kci+Corr} & $0.49$& $0.71$& $\textbf{13.0}$\\

\bottomrule
\end{tabular}
\end{center}\end{small}
\end{tabularx}
\vspace{-0.35in}
\end{table}
\begin{table}[t]
\begin{tabularx}{\linewidth}{XX}

\caption{
Adapting \ours{} to
linear
non-Gaussian (Uniform) noise.
\textsc{Lingam} run without finetuning;
\ours \textsc{(G)} finetuned for distance correlation.
}
\vspace{-0.1in}
\begin{small}\begin{center}
    
\setlength\tabcolsep{3 pt}
\begin{tabular}[b]{l rrrr}

\toprule
Model &
\multicolumn{2}{c}{N=10, E=10}&
\multicolumn{2}{c}{N=20, E=20} 
\\
\cmidrule(l{\tabcolsep}){2-3}
\cmidrule(l{\tabcolsep}){4-5}
&mAP $\uparrow$ &SHD $\downarrow$
& mAP $\uparrow$ &SHD $\downarrow$
\\
\midrule

\textsc{Dcdi-Dsf} &
$0.34$ &
$22.3$ &
$0.32$ &
$63.0$
\\
\textsc{Lingam*} & $0.34$& $7.2$& $0.30$& $18.8$\\
\midrule
\ours{} \textsc{(G)} & $0.26$& $12.7$& $0.12$& $46.6$\\
\textsc{+lingam} & $0.52$& $10.1$& $0.22$& $39.7$\\
\textsc{+dcor} & $0.44$& $8.0$& $0.21$& $33.1$\\
\textsc{+ling+dcor} & $\textbf{0.76}$& $\textbf{4.6}$& $\textbf{0.67}$& $\textbf{14.2}$\\

\bottomrule
\end{tabular}
\label{table:uniform}

\end{center}
\end{small}
&
\begin{small}
\setlength\tabcolsep{3 pt}
\caption{
Adapting \ours{} to
polynomial, sigmoid non-additive (N=10, E=10).
\textsc{Fci} run with \textsc{Kci} test;
\ours \textsc{(F)} finetuned for distance correlation.
}
\vspace{-0.1in}
\begin{center}
\begin{tabular}[b]{l rrrr}

\toprule
Model &
\multicolumn{2}{c}{Polynomial} 
& \multicolumn{2}{c}{Sigmoid}
\\
\cmidrule(l{\tabcolsep}){2-3}
\cmidrule(l{\tabcolsep}){4-5}
&mAP $\uparrow$ &SHD $\downarrow$ & mAP $\uparrow$ &SHD $\downarrow$ \\
\midrule

\textsc{Dcdi-Dsf} & $0.39$& $9.8$& $0.81$& $13.6$\\
\textsc{Fci*} & $0.12$& $10.6$& $0.53$& $8.1$\\
\midrule
\ours{} \textsc{(F)} & $0.22$& $10.6$& $0.59$& $4.8$\\
\textsc{+kci} & $0.30$& $10.6$& $0.59$& $5.5$\\
\textsc{+dcor} & $0.45$& $9.6$& $\textbf{0.90}$& $\textbf{2.1}$\\
\textsc{+kci+dcor} & $\textbf{0.52}$& $\textbf{8.2}$& $0.86$& $3.4$\\

\bottomrule
\end{tabular}
\end{center}
\label{table:polymix}
\end{small}
\end{tabularx}
\vspace{-0.35in}
\end{table}

\begin{figure}[t]
\captionof{table}{\ours{} respects identifiability theory.
Observational setting, \emph{standardized} (-std) $N=10, E=10$ linear Gaussian test datasets with $>1$ graph in Markov equivalence class (MEC).
Top: oracle performance based on true MEC (see left).
Bottom: trained \ours{} approaches oracle performance, while FCI is very noisy.
}
\label{table:identifiability}
\vspace{-0.1in}
\begin{minipage}{0.39\linewidth}
\begin{center}
\includegraphics[width=\linewidth]{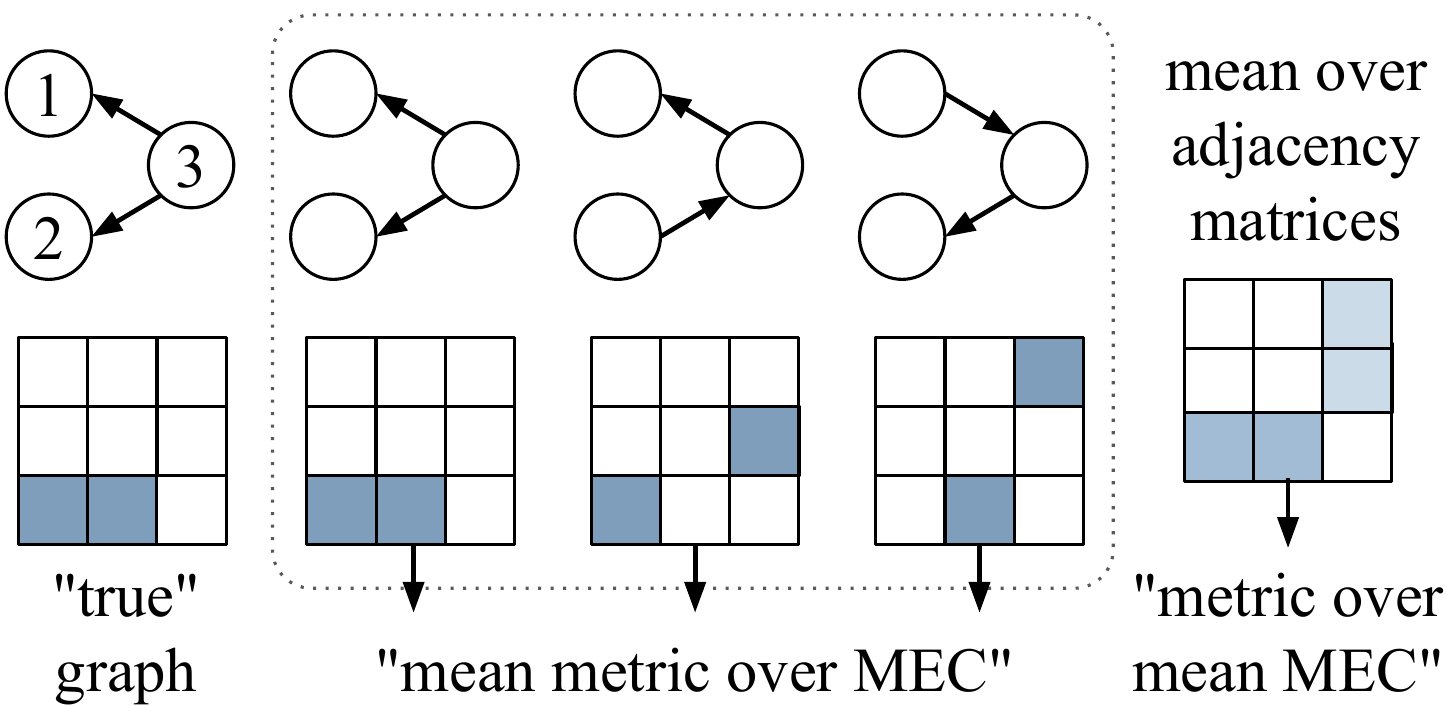}
\end{center}
\end{minipage}
\begin{minipage}{0.6\linewidth}
\setlength\tabcolsep{3.5 pt}
\begin{center}
\begin{small}
\begin{tabular}{l rrrr}
\toprule
Model & mAP($\uparrow$) & AUC($\uparrow$) & SHD($\downarrow$) & OA($\uparrow$)\\
\midrule
metric over mean MEC & $0.88 \pm \scriptstyle .10$ & $0.98 \pm \scriptstyle .03$ & $ 2.0 \pm \scriptstyle 1.0$ & $0.74 \pm \scriptstyle .22$ \\
mean metric over MEC & $0.74 \pm \scriptstyle .21$ & $0.91 \pm \scriptstyle .07$ & $ 1.2 \pm \scriptstyle .69$ & $0.84 \pm \scriptstyle .13$ \\
\midrule
\ours{}\textsc{(Fci)}-std & $0.83 \pm \scriptstyle .16$ & $0.97 \pm \scriptstyle .04$ & $ 3.3 \pm \scriptstyle 2.3$ & $0.69 \pm \scriptstyle .21$ \\
\ours{}\textsc{(Fci)+Corr}-std & $0.84 \pm \scriptstyle .14$ & $0.96 \pm \scriptstyle .03$ & $ 5.0 \pm \scriptstyle 4.5$ & $0.85 \pm \scriptstyle .14$ \\
\textsc{Fci}-std & $0.49 \pm \scriptstyle .28$ & $0.75 \pm \scriptstyle .16$ & $ 9.3 \pm \scriptstyle 2.8$ & $0.49 \pm \scriptstyle .29$ \\
\bottomrule
\end{tabular}
\end{small}
\end{center}
\end{minipage}
\vskip -0.05in
\end{figure}

\begin{figure}[t]
\centering
\begin{minipage}{0.5\linewidth}
\setlength\tabcolsep{3 pt}
\captionof{table}{\ours{} is generally insensitive to swapping estimation other algorithms with linear Gaussian assumptions, at \emph{inference} time.
Results on $N=10$ observational setting.
FCI cannot be used with \ours{}(\textsc{g}) since FCI outputs
a PAG, not a CPDAG.
}
\vspace{-0.1in}
\label{table:frankenstein}
\begin{center}
\begin{small}
\begin{tabular}{l l rr rr rr rr}
\toprule

\multirow{2}{1.5cm}{Inference estimator}
&
\multicolumn{4}{c}{\ours{} (\textsc{Fci})}
&
\multicolumn{4}{c}{\ours{} (\textsc{Gies})}
\\
& Lin. & NN & Sig. & Poly. 
& Lin. & NN & Sig. & Poly. \\
\midrule
FCI & $0.98$& $0.88$& $0.83$& $0.62$& 
--- & --- & --- & ---
\\
\midrule
PC& $0.93$& $0.85$& $0.86$& $0.64$& $0.96$& $0.89$& $0.82$& $0.58$\\
GES& $0.94$& $0.85$& $0.80$& $0.60$& $0.95$& $0.88$& $0.81$& $0.57$\\
GRaSP& $0.93$& $0.85$& $0.80$& $0.61$& $0.95$& $0.88$& $0.81$& $0.57$\\
\bottomrule
\end{tabular}
\end{small}
\end{center}
\end{minipage}
\begin{minipage}{0.49\linewidth}
\centering
\includegraphics[width=0.95\linewidth]{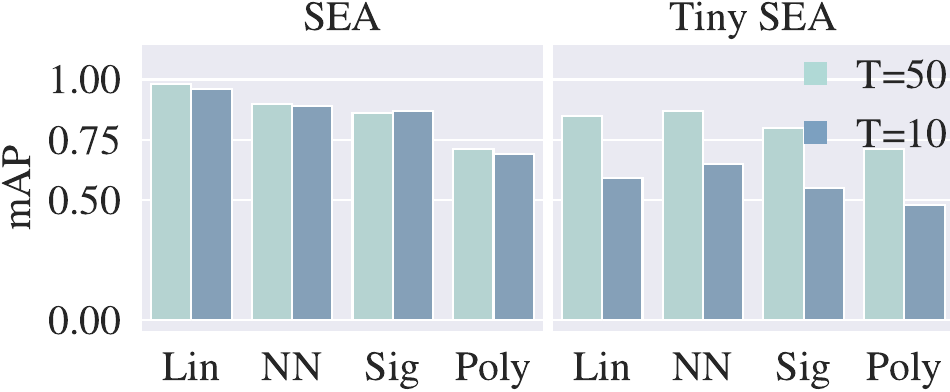}
\vspace{-0.05in}
\captionof{figure}{Few-shot learning behavior emerges as training set increases.
``Tiny'' \ours{} trained on
1/4 of the data is comparable to the full model on $N=10$ datasets when given $T=50$ batches, but is less robust with only $T=10$.
}
\label{fig:tinysea}
\end{minipage}
\vspace{-0.2in}
\end{figure}
\begin{figure}[t]
\centering
\includegraphics[width=\linewidth]{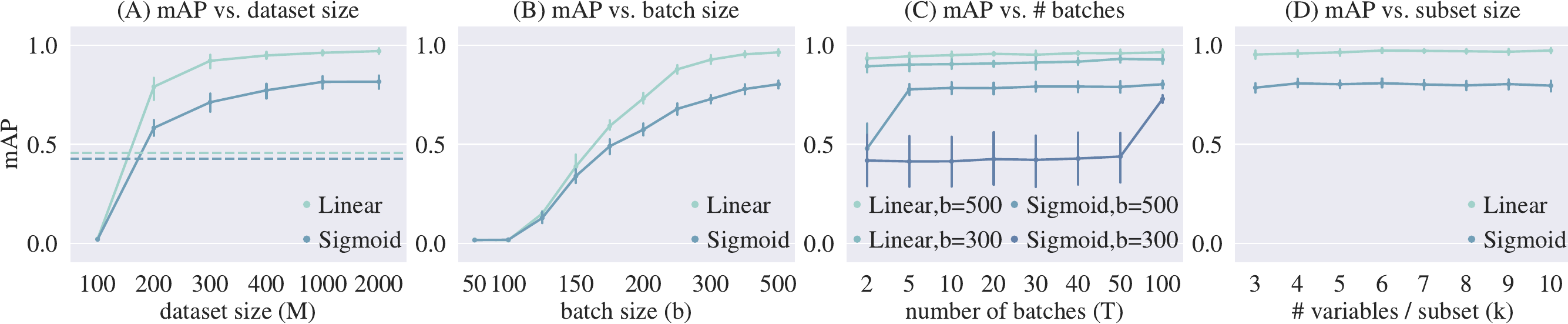}
\vspace{-0.25in}
\caption{Ablations with \ours{} (\textsc{Gies}) for estimation parameters on $N=100, E=100$.
Error bars indicate 95\% confidence interval across the 5 i.i.d. datasets of each setting.
All parameters are set to the defaults (Section~\ref{subsec:model}) unless otherwise noted.
(A) Dashed: inverse covariance at $M=500$.
(C) Variance is unusually high for Sigmoid $b=300$ until $T=100$, indicating that larger batches result in more stable results.
}
\label{fig:estimate-ablation}
\vspace{-0.1in}
\end{figure}

\begin{table*}[t]
\setlength\tabcolsep{3.5 pt}
\caption{
Ablating marginal and global features on \ours{} \textsc{(Gies)}.
Top: We set marginal and global representations to 0 (lack of $E'$/$\rho$ is out-of-distribution) and observe that the pretrained model is more robust to removing $E'$, perhaps since we sample varying $T$ during training.
Bottom: Re-train \ours{} \textsc{(Gies)} on $N=10$ datasets, with and without global features (lack of $\rho$ is in-distribution).
We observe that global features are ``worth'' $T\approx40$ estimates of $k=5$ variables each.
}
\vspace{-0.15in}
\label{table:synthetic-ablation-main}
\begin{center}
\begin{small}
\begin{tabular}{l rr rr rr rr rr}
\toprule
Model 
& \multicolumn{2}{c}{Linear} 
& \multicolumn{2}{c}{NN add.}
& \multicolumn{2}{c}{NN.}
& \multicolumn{2}{c}{Sigmoid}
& \multicolumn{2}{c}{Polynomial} \\
\cmidrule(l{\tabcolsep}){2-3}
\cmidrule(l{\tabcolsep}){4-5}
\cmidrule(l{\tabcolsep}){6-7}
\cmidrule(l{\tabcolsep}){8-9}
\cmidrule(l{\tabcolsep}){10-11}
& \multicolumn{1}{c}{mAP $\uparrow$} & \multicolumn{1}{c}{SHD $\downarrow$} & \multicolumn{1}{c}{mAP $\uparrow$} & \multicolumn{1}{c}{SHD $\downarrow$} & \multicolumn{1}{c}{mAP $\uparrow$} & \multicolumn{1}{c}{SHD $\downarrow$} & \multicolumn{1}{c}{mAP $\uparrow$} & \multicolumn{1}{c}{SHD $\downarrow$} &
\multicolumn{1}{c}{mAP $\uparrow$} & \multicolumn{1}{c}{SHD $\downarrow$}
\\
\midrule

\ours{} \textsc{(Gies)} 
& $0.99 \scriptstyle \pm .01$& $1.2 \scriptstyle \pm .7$& $0.94 \scriptstyle \pm .06$& $2.6 \scriptstyle \pm .8$& $0.91 \scriptstyle \pm .07$& $3.2 \scriptstyle \pm .3$& $0.85 \scriptstyle \pm .12$& $4.0 \scriptstyle \pm .5$& $0.70 \scriptstyle \pm .11$& $5.8 \scriptstyle \pm .6$\\
$h^\rho \leftarrow 0$ & $0.30 \scriptstyle \pm .17$& $29.2 \scriptstyle \pm .4$& $0.27 \scriptstyle \pm .18$& $29.4 \scriptstyle \pm .8$& $0.19 \scriptstyle \pm .09$& $29.0 \scriptstyle \pm .0$& $0.35 \scriptstyle \pm .17$& $27.4 \scriptstyle \pm .4$& $0.31 \scriptstyle \pm .15$& $27.1 \scriptstyle \pm .9$\\
$h^E \leftarrow 0$ & $0.85 \scriptstyle \pm .09$& $6.4 \scriptstyle \pm .7$& $0.82 \scriptstyle \pm .11$& $10.2 \scriptstyle \pm .9$& $0.78 \scriptstyle \pm .07$& $13.2 \scriptstyle \pm .2$& $0.63 \scriptstyle \pm .21$& $10.2 \scriptstyle \pm .8$& $0.55 \scriptstyle \pm .19$& $13.4 \scriptstyle \pm .6$\\

\midrule

$T=2$ & $0.20 \scriptstyle \pm .04$& $31.2 \scriptstyle \pm .5$& $0.25 \scriptstyle \pm .06$& $29.2 \scriptstyle \pm .4$& $0.33 \scriptstyle \pm .10$& $27.8 \scriptstyle \pm .1$& $0.19 \scriptstyle \pm .07$& $30.2 \scriptstyle \pm .9$& $0.24 \scriptstyle \pm .09$& $28.1 \scriptstyle \pm .9$\\
$T=10$ & $0.62 \scriptstyle \pm .16$& $8.0 \scriptstyle \pm .8$& $0.69 \scriptstyle \pm .11$& $9.6 \scriptstyle \pm .6$& $0.66 \scriptstyle \pm .13$& $11.2 \scriptstyle \pm .9$& $0.62 \scriptstyle \pm .20$& $9.2 \scriptstyle \pm .6$& $0.50 \scriptstyle \pm .22$& $8.9 \scriptstyle \pm .2$\\
$T=50$, no $\rho$
& $0.63 \scriptstyle \pm .13$& $6.8 \scriptstyle \pm .1$& $0.53 \scriptstyle \pm .07$& $6.2 \scriptstyle \pm .9$& $0.68 \scriptstyle \pm .20$& $6.0 \scriptstyle \pm .1$& $0.58 \scriptstyle \pm .15$& $7.1 \scriptstyle \pm .1$& $0.50 \scriptstyle \pm .14$& $7.1 \scriptstyle \pm .5$\\

\bottomrule
\end{tabular}
\end{small}
\end{center}
\vskip -0.15in
\end{table*}

Another strategy is to ``finetune'' the aggregator, either fully or using low-cost methods like LoRA~\citep{hu2022lora}.
Specifically, we keep the same training set and classification objective, while changing the input's featurization, e.g. a different global statistic.
Here, we show that finetuning our models for distance correlation (\textsc{Dcor}) is beneficial in both Tables~\ref{table:uniform} and \ref{table:polymix}, and the combination of strategies results in the highest performance overall, surpassing the best baseline trained from scratch (\textsc{Dcdi-Dsf}).

On real data from unknown distributions, these two strategies enable the ability to run causal discovery with different assumptions, which may be coupled with unsupervised methods for model selection~\citep{compatibility}.
Table~\ref{table:sachs} illustrates this idea using the Sachs proteomics dataset.
\ours{} can be run directly with a different estimation algorithm (FCI with polynomial kernel ``\textsc{Kci}''), or finetuned for around 4-6 hours on 1 A6000 and $<4$ GB of GPU memory (correlation ``\textsc{Corr}'').
In contrast, methods like \textsc{Avici} must simulate new datasets based on each new assumption and re-train/finetune on these data (reportedly around 4 days).

\subsection{SEA respects identifiability theory}
\label{sec:identifiability}

While the identifiability of specific causal models is not a primary focus of this work, we show that \ours{} still respects classic identifiability theory.
Specifically, while linear Gaussian models are known to be unidentifiable, Table~\ref{table:synthetic} might suggest that both \ours{} and \textsc{Dcdi} perform quite well on these data -- better than would be expected if graphs were only identifiable up to their Markov equivalence classes.
This empirical ``identifiability'' may be the consequence of two findings. Common synthetic data generation schemes tend to result in marginal variances that reflect topological order~\citep{varsort}, and in additive noise models, it has been shown that marginal variances that are the ``same or weakly monotone increasing in the [topological] ordering'' result in uniquely identifiable graphs~\citep{anm-variances}.
Data standardization can eliminate these artifacts of synthetic data generation.
In Table~\ref{sec:identifiability}, we see that after standardizing linear Gaussian data, our model performs no better than randomly selecting a graph from the Markov equivalence class (enumerated via \texttt{pcalg}~\citep{pcalg}).
The classic FCI algorithm is unable to reach this upper bound, suggesting that the amortized inference framework allows us to perform better in finite datasets.

\subsection{Ablation studies}
\label{sec:main_ablations}

In addition to high performance and flexibility, one of the hallmarks of foundation models is their ability to act as few-shot learners when scaling up~\citep{gpt}.
We first confirm that \ours{} is indeed data-efficient, requiring only around 300-400 examples for performance to converge on datasets of $N=100$ variables,
and outperforms inverse covariance (computed with 500 examples) at only 200 examples (Figure~\ref{fig:estimate-ablation}A).
To probe for how this behavior emerges, we trained a ``tiny'' version of \ours{} (\textsc{Gies}) on approximately a quarter of the training data ($N=10,20$ datasets, 64.8 million examples).
The tiny model performs nearly as well as the original on $N=10$ datasets when provided $T=50$ batches, but exhibits much poorer few-shot behavior with only $T=10$ batches (Figure~\ref{fig:tinysea}).
This demonstrates that \ours{} is able to ingest large amounts of data, leading to promising few-shot behavior.\footnote{
Due to computational limitations, we were unable to train larger models, as our existing training set requires several hundred GB in memory, and our file system does not support fast dynamic loading.}

We also ablate each parameter of the estimation step to inform best practices.
The trade-off between the number and size of batches may be relevant to estimation algorithms that scale poorly with the number of examples, e.g. kernel-based methods~\citep{kci}.
When given $T=100$ batches, \ours{} reaches reasonable performance at around 250-300 examples per batch (Figure~\ref{fig:estimate-ablation}B).
Figure~\ref{fig:estimate-ablation}C further illustrates that on the harder Sigmoid datasets, 5 batches of size $b=500$ are roughly equivalent to 100 batches of size $b=300$.
Finally, increasing the number of variables in each subset has minimal impact (Figure~\ref{fig:estimate-ablation}D), which is encouraging, as there is no need to incur the exponentially-scaling runtimes associated with larger subsets.

Finally, we analyze the impact of removing marginal estimates or global statistics (Table~\ref{table:synthetic-ablation-main}).
First, we take a fully pretrained \ours{} \textsc{(Gies)} and set the corresponding hidden representations to 0.
Performance drops more when $h^\rho$ is set to 0, indicating that our \emph{pretrained} aggregator relies more on global statistics, though a sizable gap emerges in both situations.
Then, we \emph{re-train} \ours{} \textsc{(Gies)} on the $N=10$ datasets, with and without global statistics, so that lack of $\rho$ is in-distribution for the latter model, and the training sets are comparable.
Here, we see that the ``no $\rho$'' version with $T=50$ estimates is on par with the original architecture with $T=10$ estimates,
so the global statistic is equivalent to $\sim40$ estimates.
This roughly aligns with the theory that global statistics can expedite the skeleton discovery process (Section~\ref{subsec:theory}), as the number of estimates required to discover the skeleton of a $N=10$ graph is approximately ${10 \choose 2} = 45$ (Prop.~\ref{prop:bound-skeleton}).


\section{Discussion}

Interventional experiments have formed the basis of scientific discovery throughout history,
and in recent years, advances in the life sciences have led to datasets of unprecedented scale and resolution~\citep{gwps,nadig}.
The goal of these perturbation experiments is to extract causal relationships between biological entities, such as genes or proteins.
However, the sheer size, sparsity, and noise level of these data pose significant challenges to existing causal discovery algorithms.
Moreover, these real datasets do not fit cleanly into causal frameworks that are designed around fixed sets of data assumptions, either explicit~\citep{fci} or implicit~\citep{avici}.
In this work, we approached these challenges through a causal discovery ``foundation model.''
Central to this concept were three goals.
First, this model should generalize to unseen datasets whose data-generating mechanisms are unknown, and potentially out-of-distribution.
Second, it should be easy to steer the model's predictions with inductive biases about the data.
Finally, scaling up the model should lead to data-efficiency.
We proposed \ours{}, a framework that yields causal discovery foundation models.
\ours{} was motivated by the idea that classical statistics and discovery algorithms provide powerful descriptors of data that are fast to compute and robust across datasets.
Given these statistics, we trained a deep learning model to reproduce faithful causal graphs.
Theoretically, we demonstrated that it is possible to infer causal graphs consistent with correct marginal estimates, and that our model is well-specified with respect to this task.
Empirically, we implemented two proofs of concept of \ours{} that perform well across a variety of causal discovery tasks, easily incorporate inductive biases when they are available, and exhibit excellent few-shot behavior when scaled up.

While \ours{} provides a high-level framework for supervised causal discovery,
there are several empirical limitations of the two implementations describe in this paper.
These include: 1) an arbitrary, hard-coded maximum of 1000 variables, 2) poor generalization to synthetic cyclic data, 3) erring on the side of sparsity on real data, 4) numeric instability of inverse covariance on larger graphs, and as a result, 5) training requires full precision.
The first three aspects may be addressed by modifying the architecture and/or synthetic training datasets,
while the latter two can be addressed with more numerically stable statistics, like correlation.

More broadly, the success of supervised causal discovery algorithms derives from the fidelity of the data simulation procedure.
In this paper, we present proofs of concept for the modeling framework, but we do not solve the problem of simulating realistic data.
Of the data that may exist in the real world, the training data used here represent only a small, perhaps unrealistic, fraction.
To achieve any semblance of trustworthiness on real applications, it is crucial to study the characteristics of each domain in detail -- including common graph topologies and functional forms~\citep{aguirre2024gene}; the degree and nature of missing data~\citep{hicks2018missing}; and sources of measurement error or other covariates~\citep{tran2020benchmark}.
When possible, we recommend that any insights be triangulated with other sources of knowledge.
For example, while this work does not directly provide an inference framework, the predicted structure could be used to parametrize a generative model, whose likelihood could be evaluated on held-out interventions~\citep{bacadi}.
It is also important to check whether the inferred relationships compound upon any existing biases in the data.
This is particularly important for sensitive domains like healthcare or legal applications.

This work also opens several directions for further investigation.
The framework we describe utilizes a single causal discovery algorithm and a single global statistic.
Classic causal discovery algorithms leverage diverse insights for identifying causal structure, e.g. the non-Gaussianity of noise~\citep{lingam} vs. conditional independence~\citep{fci}.
Thus, different discovery algorithms or summary statistics may reveal different aspects of the causal structure.
Learning to resolve these potentially conflicting views remains unexplored, both from experimental and theoretical perspectives.
Furthermore, this work only shows that the axial attention architecture is well-specified as a model class and probes generalization empirically.
This motivates theoretical studies into supervised causal discovery with regards to what information can be provably learned, e.g. in the style of of PAC (probably approximately correct) learning~\citep{allen2019learning}.

In summary, we hope that this work will inspire a new avenue of research into causal discovery algorithms that are applicable to and informed by real applications.

\section*{Acknowledgements}

We thank Bowen Jing, Felix Faltings, Sean Murphy, and Wenxian Shi for helpful discussions regarding the writing; as well as Jiaqi Zhang, Romain Lopez, Caroline Uhler, and Stephen Bates for helpful feedback regarding the framing of this project.
Finally, we thank our action editor Bryon Aragam and our anonymous reviewers for their invaluable suggestions towards improving this paper during the review process.

This material is based upon work supported by the National Science Foundation Graduate Research Fellowship under Grant No. 1745302. We would like to acknowledge support from the NSF Expeditions grant (award 1918839: Collaborative Research: Understanding the World Through Code), Machine Learning for Pharmaceutical Discovery and Synthesis (MLPDS) consortium, and the Abdul Latif Jameel Clinic for Machine Learning in Health.

\newpage
\bibliography{references}
\bibliographystyle{tmlr}

\newpage
\appendix
\section{Theoretical motivations}
\label{sec:proofs}

Our theoretical contributions focus on two primary directions.

\begin{enumerate}
    \item We formalize the notion of marginal estimates used in this paper, and prove that given sufficient marginal estimates, it is possible to recover a pattern faithful to the global causal graph. We provide lower bounds on the number of marginal estimates required for such a task, and motivate global statistics as an efficient means to reduce this bound.
    \item We show that our proposed axial attention has the capacity to recapitulate the reasoning required for marginal estimate resolution.
    We provide realistic, finite bounds on the width and depth required for this task.
\end{enumerate}
Before these formal discussions, we start with a toy example to provide intuition regarding marginal estimates and constraint-based causal discovery algorithms.

\subsection{Toy example: Resolving marginal graphs}
\label{resolver-toy-example}

Consider the Y-shaped graph with four nodes in Figure \ref{toy-resolution}.
Suppose we run the PC algorithm on all subsets of three nodes, and we would like to recover the result of the PC algorithm on the full graph.
We illustrate how one might resolve the marginal graph estimates.
The PC algorithm consists of the following steps~\citep{causation-prediction-search}.
\begin{enumerate}
    \item Start from the fully connected, undirected graph
    on $N$ nodes.
    \item Remove all edges $(i,j)$ where $X_i \indep X_j$.
    \item For each edge $(i,j)$
    and subsets $S \subseteq [N]\setminus\{i,j\}$ of increasing
    size $n=1,2,\dots,d$,
    where $d$ is the maximum degree in $G$,
    and all $k\in S$
    are connected to either $i$ or $j$: if $X_i \indep X_j \mid S$,
    remove edge $(i,j)$.
    \item For each triplet
    $(i,j,k)$, such that only edges
    $(i,k)$ and $(j,k)$ remain,
    if $k$ was not in the set $S$
    that eliminated edge $(i,j)$,
    then orient the ``v-structure''
    as $i \rightarrow k \leftarrow j$.
    \item (Orientation propagation)
    If $i\to j$, edge $(j,k)$ remains,
    and edge $(i,k)$ has been removed,
    orient $j \to k$.
    If there is a directed path
    $i\rightsquigarrow j$
    and an undirected edge $(i,j)$,
    then orient $i\to j$.
\end{enumerate}
\begin{figure*}[ht]
\begin{center}
\centerline{\includegraphics[width=\textwidth]{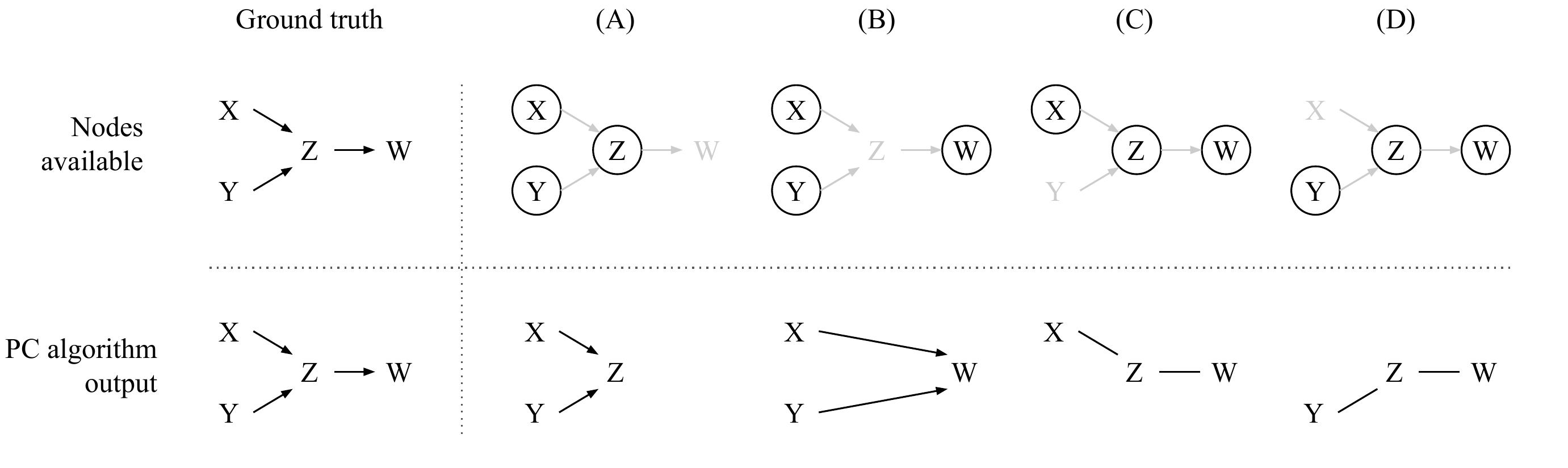}}
\caption{Resolving marginal graphs.
Subsets of nodes revealed to the PC algorithm
(circled in row 1)
and its outputs (row 2).}
\label{toy-resolution}
\end{center}
\vskip -0.2in
\end{figure*}
In each of the four cases,
the PC algorithm estimates the respective graphs as follows.
\begin{enumerate}[label=(\Alph*)]
    \item We remove edge $(X,Y)$ via (2)
    and orient the v-structure.
    \item We remove edge $(X,Y)$ via (2)
    and orient the v-structure.
    \item We remove edge $(X,W)$ via (3)
    by conditioning on $Z$.
    There are no v-structures, so the edges
    remain undirected.
    \item We remove edge $(Y,W)$ via (3)
    by conditioning on $Z$.
    There are no v-structures, so the edges
    remain undirected.
\end{enumerate}
The outputs (A-D) admit the full PC algorithm output
as the only consistent graph on four nodes.
\begin{itemize}
    \item $X$ and $Y$ are unconditionally
    independent, so no subset will reveal an edge between $(X,Y)$.
    \item There are no edges between $(X,W)$ and $(Y,W)$.
    Otherwise, (C) and (D) would yield the undirected triangle.
    \item $X,Y,Z$ must be oriented as
    $X \rightarrow Z \leftarrow Y$.
    Paths $X \to Z \to Y$ and
    $X \leftarrow Z \leftarrow Y$ would induce
    an $(X,Y)$ edge in (B).
    Reversing orientations
    $X \leftarrow Z \rightarrow Y$
    would contradict (A).
    \item $(Y,Z)$ must be oriented as
    $Y\to Z$. Otherwise, (A) would remain unoriented.
\end{itemize}

\subsection{Resolving marginal estimates into global graphs}
\label{subsec:resolve-proof}


Classical results have characterized the Markov equivalency class of directed acyclic graphs.
Two graphs are observationally equivalent if they have the same skeleton and v-structures~\citep{verma1990equivalence}.
Thus, a pattern $P$ is \emph{faithful} to a graph $G$ if and only if they share the same skeletons and v-structures~\citep{sgs}.

\begin{definition}
    Let $G=(V,E)$ be a directed acyclic graph. 
    A \emph{pattern} $P$ is a set of directed and undirected edges over $V$.
\end{definition}

\begin{definition}[Theorem 3.4 from \citet{causation-prediction-search}]
\label{thm:faithfulness}
    If pattern $P$ is \emph{faithful} to some directed acyclic graph, then $P$ is faithful to $G$ if and only if
    \begin{enumerate}
        \item for all vertices $X,Y$ of $G$,
        $X$ and $Y$ are adjacent if and only if $X$ and $Y$ are dependent conditional on every set of vertices of $G$ that does not include $X$ or $Y$; and
        \item for all vertices $X, Y, Z$, such that $X$ is adjacent to $Y$ and $Y$ is adjacent to $Z$ and $X$ and $Z$ are not adjacent, $X\to Y\lto Z$ is a subgraph of $G$ if and only if $X, Z$ are dependent conditional on every set containing $Y$ but not $X$ or $Z$.
    \end{enumerate}
\end{definition}

Given data faithful to $G$, a number of classical constraint-based algorithms produce patterns that are faithful to $G$. We denote this set of algorithms as $\mathcal{F}$.

\begin{theorem}[Theorem 5.1 from \cite{causation-prediction-search}]
\label{thm:pc-and-friends}
    If the input to the PC, SGS, PC-1, PC-2, PC$^*$, or IG algorithms faithful to directed acyclic graph $G$, the output is a pattern that represents the faithful indistinguishability class of $G$.
\end{theorem}

The algorithms in $\mathcal{F}$ are sound and complete \emph{if} there are no unobserved confounders.


Let $P_V$ be a probability
distribution that is Markov, minimal, and faithful to $G$.
Let $D\in\RR^{M\times N}\sim P_V$ be a dataset of $M$ observations over all $N=|V|$ nodes.

Consider a subset $S\subseteq V$. Let $D[S]$ denote the subset of $D$ over $S$,
\begin{equation}
    D[S] = \{ x_{i,v} : v\in S \}_{i=1}^N,
\end{equation}
and let $G[S]$ denote the subgraph of $G$
induced by $S$
\begin{equation}
    G[S] = (S, \{ (i,j) : i, j\in S, (i,j)\in E\}.
\end{equation}
If we apply any $f\in\mathcal{F}$ to $D[S]$, the results are \emph{not} necessarily faithful to $G[S]$, as now there may be latent confounders in $V\setminus S$ (by construction).
We introduce the term \emph{marginal estimate} to denote the resultant pattern that, while not faithful to $G[S]$, is still informative.

\begin{definition}[Marginal estimate]
    A pattern $E'$ is a \emph{marginal estimate} of $G[S]$ if and only if
    \begin{enumerate}
        \item for all vertices $X,Y$ of $S$,
        $X$ and $Y$ are adjacent if and only if $X$ and $Y$ are dependent conditional on every set of vertices of $S$ that does not include $X$ or $Y$; and
        \item for all vertices $X, Y, Z$, such that $X$ is adjacent to $Y$ and $Y$ is adjacent to $Z$ and $X$ and $Z$ are not adjacent, $X\to Y\lto Z$ is a subgraph of $S$ if and only if $X, Z$ are dependent conditional on every set containing $Y$ but not $X$ or $Z$.
    \end{enumerate}
\end{definition}


\begin{algorithm}[t]
   \caption{Resolve marginal estimates of $f\in\mathcal{F}$}
   \label{alg:resolve-pc}
\begin{algorithmic}[1]
   \STATE {\bfseries Input:} Data $\mathcal{D}_G$
   faithful to $G$
   \STATE Initialize $E' \leftarrow K_N$
   as the complete undirected graph
   on $N$ nodes.
   \FOR{$S \in \mathcal{S}_{d+2}$}
   \STATE Compute $E'_S = f(\mathcal{D}_{G[S]})$
   \FOR{$(i,j) \not\in E'_S$}
   \STATE Remove $(i,j)$ from $E'$
   \ENDFOR
   \ENDFOR
   \FOR{$E'_S \in \{E'_S\}_{\mathcal{S}_{d+2}}$}
   \FOR{v-structure $i \rightarrow j \leftarrow k$
   in $E'_S$}
   \IF{$\{i,j\}, \{j,k\}\in E'$ and $\{i,k\}\not\in E'$}
   \STATE Assign orientation
   $i \rightarrow j \leftarrow k$ in $E'$
   \ENDIF
   \ENDFOR
   \ENDFOR
   \STATE Propagate orientations in $E'$ (optional).
\end{algorithmic}
\end{algorithm}

\begin{proposition}
    Let $G=(V,E)$ be a directed acyclic graph
    with maximum degree $d$.
    For $S\subseteq V$,
    let $E'_S$ denote the marginal estimate
    over $S$.
    Let $\mathcal{S}_d$ denote the superset
    that contains all subsets $S\subseteq V$
    of size at most $d$.
    Algorithm~\ref{alg:resolve-pc} maps
    $\{ E'_S \}_{S \in \mathcal{S}_{d+2}}$ to a pattern $E'$ faithful to $G$.
   \label{prop:resolve-new}
\end{proposition}

On a high level, lines 3-8 recover the undirected
``skeleton'' graph of $E^*$,
lines 9-15 recover the v-structures,
and line 16 references step 5 in Section \ref{resolver-toy-example}.

\begin{remark}
    In the PC algorithm (\cite{causation-prediction-search}, \ref{resolver-toy-example}), its derivatives,
    and Algorithm \ref{alg:resolve-pc},
    there is no need to consider separating sets with cardinality
    greater than maximum degree $d$, since the
    maximum number of independence tests required to separate
    any node from the rest of the graph
    is equal to number of its parents plus its children
    (due to the Markov assumption).
\end{remark}

\begin{lemma}
\label{prop:skeleton}
    The undirected skeleton of $E^*$
    is equivalent to
    the undirected skeleton of $E'$
    \begin{equation}
        C^* \coloneqq \left\{
        \{i,j\} \mid (i,j) \in E^* \textrm{ or }
        (j,i) \in E^* \right\}
        = \left\{
        \{i,j\} \mid (i,j) \in E' \textrm{ or }
        (j,i) \in E' \right\}
        \coloneqq C'.
    \end{equation}
    That is, $\{i,j\} \in C^* \iff \{i, j\} \in C'$.
\end{lemma}

\begin{proof}
It is equivalent to show that
$\{i,j\} \not\in C^* \iff \{i, j\} \not\in C'$

$\Rightarrow$
If $\{ i,j \} \not\in C*$, then there must exist
a separating set $S$ in $G$ of at most size $d$
such that $i \indep j \mid S$.
Then $S \cup \{i,j\}$ is a set of at most size $d+2$,
where $\{i,j\} \not\in C'_{S \cup \{i,j\}}$.
Thus, $\{i,j\}$ would have been removed from $C'$
in line 6 of Algorithm \ref{alg:resolve-pc}.

$\Leftarrow$
If $\{ i,j \} \not\in C'$,
let $S$ be a separating set in $\mathcal{S}_{d+2}$ such that
$\{i,j\} \not\in C'_{S \cup \{i,j\}}$
and $i \indep j \mid S$.
$S$ is also a separating set in $G$,
and conditioning on $S$ removes $\{i,j\}$ from $C^*$.
\end{proof}

\begin{lemma}
\label{prop:orientation}
    A v-structure $i\rightarrow j\leftarrow k$
    exists in $E^*$ if and only if there exists
    the same v-structure in $E'$.
\end{lemma}

\begin{proof}
V-structures are oriented
$i\rightarrow j\leftarrow k$ in $E^*$ if there
is an edge between $\{i,j\}$ and $\{j,k\}$
but not $\{i,k\}$; and if $j$ was not in the
conditioning set that removed $\{i,k\}$.
Algorithm \ref{alg:resolve-pc} orients v-structures
$i\rightarrow j\leftarrow k$ in $E'$ if they are
oriented as such in any $E'_S$; and if
$\{i,j\}, \{j,k\}\in E', \{i,k\}\not\in E'$

$\Rightarrow$ Suppose for contradiction that
$i\rightarrow j\leftarrow k$ is oriented as a v-structure
in $E^*$, but not in $E'$.
There are two cases.
\begin{enumerate}
    \item No $E'_S$ contains the undirected path
    $i - j - k$. If either $i-j$ or $j-k$ are missing from any $E'_S$,
    then $E^*$ would not contain $(i,j)$ or $(k,j)$.
    Otherwise, if all $S$ contain $\{i,k\}$,
    then $E^*$ would not be missing $\{i,k\}$ (Lemma~\ref{prop:skeleton}).
    \item In every $E'_S$ that contains $i - j - k$,
    $j$ is in the conditioning set that removed
    $\{i,k\}$,
    i.e. $i \indep k \mid S, S \ni j$. This would violate the faithfulness property,
    as $j$ is neither a parent of $i$ or $k$ in 
    $E^*$, and the outputs of the PC algorithm
    are faithful to the equivalence class of $G$
    (Theorem 5.1 \cite{causation-prediction-search}).
\end{enumerate}

$\Leftarrow$ Suppose for contradiction that
$i\rightarrow j\leftarrow k$ is oriented as a v-structure
in $E'$, but not in $E^*$.
By Lemma~\ref{prop:skeleton},
the path $i-j-k$ must exist in $E^*$.
There are two cases.
\begin{enumerate}
    \item If $i\to j\to k$ or
    $i\leftarrow j \leftarrow k$,
    then $j$ must be in the conditioning set that removes
    $\{i,k\}$, so no $E'_S$ containing $\{i,j,k\}$
    would orient them as v-structures.
    \item If $j$ is the root of a fork
    $i \leftarrow j \rightarrow k$,
    then as the parent of both $i$ and $k$,
    $j$ must be in the conditioning set that removes
    $\{i,k\}$, so no $E'_S$ containing $\{i,j,k\}$
    would orient them as v-structures.
\end{enumerate}
Therefore, all v-structures in $E'$ are also
v-structures in $E^*$.
\end{proof}

\begin{proof}[Proof of Proposition \ref{prop:resolve-new}]
    Given data that is faithful to $G$,
    Algorithm \ref{alg:resolve-pc} produces
    a pattern $E'$ with the same connectivity
    and v-structures as $E^*$.
    Any additional orientations in both patterns
    are propagated
    using identical, deterministic procedures,
    so $E' = E^*$.
\end{proof}

This proof presents a deterministic but inefficient algorithm for resolving marginal subgraph estimates.
In reality, it is possible to recover the undirected skeleton and the v-structures of $G$ without checking all subsets $S\in\mathcal{S}_{d+2}$.

\begin{restatable}[Skeleton bounds]{proposition}{propBoundSkeleton}
\label{prop:bound-skeleton}
    Let $G=(V,E)$ be a directed acyclic graph
    with maximum degree $d$.
    It takes $O(N^2)$ marginal estimates over subsets of size $d+2$ to recover the undirected skeleton of $G$.
\end{restatable}

\begin{proof}
    Following Lemma~\ref{prop:skeleton},
    an edge $(i,j)$ is not present in $C$ if it is not present in any of the size $d+2$ estimates.
    Therefore, every pair of nodes $\{i,j\}$ requires only a single estimate of size $d+2$, so it is possible to recover $C$ in $\binom{N}{2}$ estimates.
\end{proof}

\begin{restatable}[V-structures bounds]{proposition}{propBoundV}
\label{prop:bound-v}
    Let $G=(V,E)$ be a directed acyclic graph
    with maximum degree $d$
    and $\nu$ v-structures.
    It is possible to identify all v-structures in $O(\nu)$ estimates over subsets of at most size $d+2$.
\end{restatable}

\begin{proof}
    Each v-structure $i\to j\lto k$ falls under two cases.
    \begin{enumerate}
        \item $i\indep k$ unconditionally.
        Then an estimate over $\{i,j,k\}$ will identify the v-structure.
        \item $i\indep k \mid S$,
        where $j\not\in S\subset V$.
        Then an estimate over $S \cup \{i,j,k\}$ will identify the v-structure.
        Note that $|S| \le d+2$ since the degree of $i$ is at least $|S|+1$.
    \end{enumerate}
    Therefore, each v-structure only requires one estimate, and it is possible to identify all v-structures in $O(\nu)$ estimates.
\end{proof}

There are three takeaways from this section.
\begin{enumerate}
    \item If we exhaustively run a constraint-based algorithm on all subsets of size $d+2$, it is trivial to recover the estimate of the full graph. However, this is no more efficient than running the causal discovery algorithm on the full graph.
    \item In theory, it is possible to recover the undirected graph in $O(N^2)$ estimates,
    and the v-structures in $O(\nu)$ estimates.
    However, we may not know the appropriate subsets ahead of time.
    \item In practice, if we have a surrogate for connectivity, such as the global statistics used in \ours{}, then we can vastly reduce the number of estimates used to eliminate edges from consideration, and more effectively focus on sampling subsets for orientation determination.
\end{enumerate}

\newpage
\subsection{Model specification}
\label{subsec:power-proof}

\paragraph{Model specification and universality}

In classical statistics, a statistical model can be expressed as a pair $(S, P_\theta)$ for $\theta\in\Theta$, where $S$ is the sample space, $P$ is the family of distributions on $S$, and $\Theta$ is the space of parameters~\citep{mccullagh2002statistical}.
Let $x_1,\dots,x_N$ be observations of i.i.d. random variables in $S$, and let $P^*$ denote their common distribution.
A statistical model is \emph{well-specified} if $P^* = P_\theta$ for some $\theta\in\Theta$.
Identifying the most appropriate class of models has long been an area of interest in statistics~\citep{akaike1973information}.
For simple models like linear regression, there are many ways to test for model specification, where the alternative may be that the data follow a non-linear relationship~\citep{davidson1981several}.

In machine learning, however, even the simplest architectures vary immensely. Whether a neural network is ``well-specified'' depends on many aspects, such as its width (number of hidden units), depth (number of layers), activation functions (sources of non-linearity), and more.
Thus, instead of testing whether each neural network architecture is well-specified for each experiment, it is more common to show universality (or lack thereof). That is, given a class of functions $\mathcal{F}$ and the space of parameters $\Theta$, for every $f\in\mathcal{F}$ does there exist a $\theta\in\Theta$ that allows the neural network to approximate $f$ to arbitrary accuracy?
For example, the most well-cited work for the universality of multi-layer perceptrons~\citep{mlp-universal} showed that:
\begin{quote}
``standard multilayer feedforward networks with as few as one hidden layer using arbitrary \textbf{squashing functions} are capable of approximating any Borel measurable function from one finite dimensional space to another to any desired degree of accuracy, \textbf{provided sufficiently many hidden units are available}.''
\end{quote}
Note that this definition \emph{excludes} neural networks that use rectified linear units (ReLU is unbounded), whose variants are ubiquitous today, as well as neural networks with bounded width.
Universality under alternate assumptions has been addressed by many later works, e.g. \citet{maiorov1999lower,shen2022optimal}.

With respect to modern Transformer architectures~\citep{attention}, the existing literature are similarly subject to constraints.
For example, \citet{transformer-turing} shows that Transformers equipped with positional encodings are Turing complete, given \emph{infinite precision}.
Alternatively, \citet{transformers-universal-approx} proves that Transformers can approximate arbitrary continuous ``sequence-to-sequence'' functions, but requires in the worst case, \emph{exponential depth}.
Finally, \citet{sanford2024representational} shows that one-dimensional self attention is \emph{unable} to detect arbitrary triples along a sequence, unless the depth scales linearly as the input size (sequence length).
This last case is particularly relevant to our setting, as v-structures are relations of three nodes, but even modern large language models ``only'' contain $\sim100$ layers (e.g. GPT-4 is reported to have 120 layers).
In general, these works present constructive proofs of either $\theta$~\citep{transformer-turing,transformers-universal-approx} or pathological cases that cannot be handled~\citep{sanford2024representational}.

In this paper, we do \emph{not} consider universality in a general sense, but rather limit our scope to the algorithms presented in Section~\ref{subsec:resolve-proof}.
This will allow us to make more realistic assumptions regarding the model class.


\paragraph{Causal identifiability for continuous discovery algorithms}

Identifiability is central to causal discovery, as it is important to understand the degree to which a causal model can or cannot be inferred from data.
In contrast to classic approaches, causal discovery algorithms that rely on continuous optimization must consider several additional aspects with respect to identifiability -- namely, their objective; optimization dynamics; and model specification.
We discuss representative works in light of these considerations.

First, the optima of the objectives must correspond to true graphs.
\textsc{NoTears}~\citep{notears} focuses on the linear SEM case with least-squares loss, and they cite earlier literature regarding the identifiability of this setup~\citep{van2013l0,loh2014high,aragam2015learning}.
\textsc{Dcdi}~\citep{dcdi} 
proves that the graph that maximizes their proposed score is $\mathcal{I}$-Markov equivalent to the ground truth, subject to faithfulness and other regularity conditions.
Since amortized causal discovery algorithms tend to be trained on data from a variety of data-generating processes, the identifiability of an arbitrary graph is less clear.
\citet{ke2022learning} cites \citet{eberhardt2006n} and claims that any graph is identifiable ``in the limit of an infinite amount of [single-node hard] interventional data samples.''
In our interventional setting, this also holds.
\textsc{Avici}~\citep{avici} focuses solely on inferring graphs from data, without addressing whether they are identifiable.
If we suppose that the graphs are identifiable, then in all three cases, the objectives can be written in terms of the KL divergence between the true and predicted edge distributions, which reaches its minimum when the predicted graph matches the true graph.

Second, the optimization process must not only converge to an optimum, but also (for amortized models) generalize to unseen data.
This is, in general, difficult to show.
Both \textsc{NoTears} and \textsc{Dcdi} acknowledge that due to the non-convexity of this optimization problem, the optimizer may converge to a stable point, which is not necessarily the global optimum.
\textsc{Avici} includes an explicit acyclicity constraint, so they additionally cite \citet{nandwani2019primal,jin2020local} for inspiration in constrained optimization of a neural network.
Regarding generalization, existing results on what can be ``provably'' learned by neural networks are limited to simple architectures and/or algorithms~\citep{allen2019learning, shao2022theory}.
Instead, both \textsc{Avici} and our work assess generalization empirically, by holding out certain causal mechanisms from training.

Finally, the implementation of each model must be well-specified.
Since \textsc{NoTears} assumes linearity, it directly optimizes the (weighted) adjacency, and the only concern is whether linearity holds.
\textsc{Dcdi} cites that deep sigmoidal flows are universal density approximators~\citep{huang2018neural}, which in turn invokes the classic result that a multilayer perceptron with sigmoid activations is a universal approximator~\citep{cybenko1989approximation}.
\textsc{Avici} does not discuss this aspect of their architecture, which is also based on sparse attention, but is otherwise quite different from ours.
The following section proves by construction that our model is well-specified.

\paragraph{Axial attention architecture is well-specified}

In the context of Section~\ref{subsec:resolve-proof},
we show that three axial attention blocks (model depth) are sufficient to recover the skeleton and v-structures in $O(N)$ width, and we require $O(L)$ to propagate orientations along paths of length $L$.
In the following section, we first formalize the notion
of a neural network architecture's capacity to
``implement'' an algorithm.
Then we prove Theorem~\ref{thm:resolve-exists} by construction.

\begin{definition}
Let $f$ be a map from
finite sets $Q$ to $F$,
and let $\phi$ be a map from
finite sets $Q_\Phi$ to $F_\Phi$.
We say $\phi$ \emph{implements}
$f$ if there exists 
injection $g_\textrm{in}: Q\to Q_\Phi$
and surjection
$g_\textrm{out}: F_\Phi \to F$
such that
\begin{equation}
    \forall q\in Q, g_\textrm{out}(\phi(g_\textrm{in}(q))) = f(q).
\end{equation}
\end{definition}

\begin{definition}
Let $Q, F, Q_\Phi, F_\Phi$ be finite sets.
Let $f$ be a map from
$Q$ to $F$,
and let $\Phi$ be a finite set of maps
$\{\phi:Q_\Phi\to F_\Phi\}$.
We say $\Phi$ has the \emph{capacity} to implement $f$ if and only if there exists at least one element $\phi\in\Phi$ that implements $f$.
\end{definition}

That is, a single model \emph{implements} an algorithm $f$ if for every input to $f$, the model outputs the corresponding output.
A class of models parametrized by $\phi\in\Phi$ has the \emph{capacity} to implement an algorithm if there exists at least one $\phi$ that implements $f$.

\thmresolve*

\begin{proof} We consider axial attention blocks with dot-product attention and omit layer normalization from our analysis, as is common in the Transformer universality literature~\cite{transformers-universal-approx}.
Our inputs $X\in\RR^{d\times R \times C}$ consist of $d$-dimension embeddings over $R$ rows and $C$ columns.
Since our axial attention only operates over one dimension at a time, we use $X_{\cdot,c}$ to denote a 1D sequence of length $R$, given a fixed column $c$, and $X_{r,\cdot}$ to denote a 1D sequence of length $C$, given a fixed row $r$.
A single axial attention layer (with one head) consists of two attention layers and a feedforward network,
\begin{align*}
    \text{Attn}_\text{row}(X_{\cdot,c}) &= X_{\cdot,c} +
        W_O W_V X_{\cdot,c} \cdot \sigma\left[
            (W_K X_{\cdot,c})^T W_Q X_{\cdot,c}
        \right],\numberthis \\
    X &\lto \text{Attn}_\text{row}(X) \\
    \text{Attn}_\text{col}(X_{r,\cdot}) &= X_{r,\cdot} +
        W_O W_V X_{r,\cdot} \cdot \sigma\left[
            (W_K X_{r,\cdot})^T W_Q X_{r,\cdot}
        \right], \numberthis\\
    X &\lto \text{Attn}_\text{col}(X) \\
    \text{FFN}(X) &= X +
        W_2 \cdot \text{ReLU}(W_1 \cdot 
        X + b_1 \mathbf{1}_L^T)
        + b_2 \mathbf{1}_L^T, \numberthis
\end{align*}
where $W_O\in\RR^{d\times d}, W_V, W_K, W_Q \in \RR^{d\times d}, W_2\in\RR^{d\times m}, W_1\in\RR^{m\times d}, b_2\in\RR^{d}, b_1\in\RR^m$,
and $m$ is the hidden layer size of the feedforward network.
For concision, we have omitted the $r$ and $c$ subscripts on the $W$s, but the row and column attentions use different parameters.
Any row or column attention can take on the identity mapping by setting $W_O, W_V, W_K, W_Q$ to $d\times d$ matrices of zeros.

A single axial attention \emph{block} consists of two axial attention layers $\phi_E$ and $\phi_\rho$,
connected via messages (Section \ref{subsec:model})
\begin{align*}
    h^{E,\ell} &= \phi_{E,\ell}(h^{E,\ell-1}) \\
    h^{\rho,\ell-1} &\lto W_{\rho,\ell} \left[h^{\rho,\ell-1}, m^{E\to\rho,\ell}\right] \\
    h^{\rho,\ell} &= \phi_{\rho,\ell}(h^{\rho,\ell-1}) \\
    h^{E,\ell} &\lto h^{E,\ell} + m^{\rho\to E,\ell}
\end{align*}
where $h^\ell$ denote the hidden representations of $E$ and $\rho$ at layer $\ell$,
and the outputs of the axial attention block are $h^{\rho,\ell}, h^{E,\ell}$.

We construct a stack of $L\ge 3$ axial attention blocks that implement Algorithm \ref{alg:resolve-pc}.

\paragraph{Model inputs} Consider edge estimate $E'_{i,j} \in\mathcal{E}$ in a graph of size $N$.
Let $e_i, e_j$ denote the endpoints of $(i,j)$.
Outputs of the PC algorithm can be expressed by three endpoints: $\{\varnothing, \bullet, \blacktriangleright\}$.
A directed edge from $i\to j$ has endpoints $(\bullet,\blacktriangleright)$,
the reversed edge $i\lto j$ has endpoints $(\blacktriangleright, \bullet)$,
an undirected edge has endpoints $(\bullet,\bullet)$, 
and the lack of any edge between $i,j$ has endpoints $(\varnothing, \varnothing)$.

Let $\textrm{one-hot}_N(i)$ denote the $N$-dimensional one-hot column vector where element $i$ is 1.
We define the embedding of $(i,j)$ as
a $d=2N+6$ dimensional vector,
\begin{equation}
\label{eq:constructed-embedding}
    g_\text{in}(E_{t,(i,j)}) = 
    h^{E,0}_{(i,j)} = \left[\begin{array}{c}
        \textrm{one-hot}_{3}(e_i) \\
        \hline
        \textrm{one-hot}_{3}(e_j) \\
        \hline
        \textrm{one-hot}_N(i) \\
        \hline
        \textrm{one-hot}_N(j)
    \end{array}\right].
\end{equation}
To recover graph structures from $h^
E$, we simply read off the indices of non-zero entries ($g_\text{out}$).
We can set $h^{\rho,0}$ to any $\RR^{d\times\vsq}$ matrix,
as we do not consider its values in this analysis
and discard it during the first step.

\begin{claim}(Consistency)
\label{claim:consistency}
    The outputs of each step
    \begin{enumerate}
        \item are consistent with (\ref{eq:constructed-embedding}), and
        \item are equivariant to the ordering of nodes in edges.
    \end{enumerate}
\end{claim}
For example, if $(i,j)$ is oriented as $(\blacktriangleright,\bullet)$, then we expect $(j,i)$ to be oriented $(\bullet, \blacktriangleright)$.

\paragraph{Step 1: Undirected skeleton}
We use the first axial attention block to recover the undirected skeleton $C'$.
We set all attentions to the identity,
set $W_{\rho,1} \in \RR^{2d\times d}$
to a $d\times d$ zeros matrix, stacked on top of a $d\times d$ identity matrix (discard $\rho$),
and set $\text{FFN}_E$ to the identity (inputs are positive).
This yields
\begin{equation}
    h^{\rho,0}_{i,j} = m^{E\to\rho,1}_{i,j} 
    = \left[\begin{array}{c}
        P_{e_i}(\varnothing) \\
        P_{e_i}(\bullet) \\
        P_{e_i}(\blacktriangleright) \\
        \hline
        \vdots \\
        \hline
        \textrm{one-hot}_N(i) \\
        \hline
        \textrm{one-hot}_N(j)
    \end{array}\right],
\end{equation}
where $P_{e_i}(\cdot)$ is the frequency
that endpoint $e_i = \cdot$ within the
subsets sampled.
FFNs with 1 hidden layer are universal approximators of continuous functions~\citep{mlp-universal},
so we use $\text{FFN}_\rho$ to map 
\begin{equation}
    \text{FFN}_\rho(X_{i,u,v}) = \begin{cases}
        0 & i \le 6 \\
        0 & i > 6, X_{1, u, v} = 0 \\
        -X_{i, u,v} & \text{ otherwise,}
    \end{cases}
\end{equation}
where $i\in[2N+6]$ indexes into the feature dimension,
and $u,v$ index into the rows and columns.
This allows us to remove edges not present
in $C'$ from consideration:
\begin{align*}
    m^{\rho\to E,1} &= h^{\rho,1} \\
    h^{E,1}_{i,j} &\lto h^{E,1}_{i,j} + m^{\rho\to E,1}_{i,j}
    = \begin{cases}
        0 & (i,j) \not\in C' \\
        h^{E,0}_{i,j} & \text{ otherwise.}
    \end{cases}\numberthis
\end{align*}
This yields $(i,j) \in C'$ if and only if $h^{\rho,1}_{i,j} \ne \mathbf{0}$.
We satisfy \ref{claim:consistency} since our inputs are valid PC algorithm outputs for which $P_{e_i}(\varnothing) = P_{e_j}(\varnothing)$.


\paragraph{Step 2: V-structures}
The second and third axial attention blocks recover v-structures.
We run the same procedure twice,
once to capture v-structures that point towards the first node in an ordered pair, and one to capture v-structures that point towards the latter node.

We start with the first row attention
over edge estimates, given a fixed subset $t$.
We set the key and query
attention matrices
\begin{align}
    \MoveEqLeft[1]
    W_K = k \cdot \left[\begin{array}{rrrrrrrr}
         0 & 0 & 1 \\
         &&& 0 & 1 & 0 \\
         & \vdots \\
         &&& &&& I_{N} \\
         &&& &&& & -I_{N}
    \end{array}\right]
    &
    W_Q = k \cdot \left[\begin{array}{rrrrrrrr}
         0 & 0 & 1 \\
         &&& 0 & 1 & 0 \\
         & \vdots \\
         &&& &&& I_{N} \\
         &&& &&& & I_{N}
    \end{array}\right]
\end{align}
where $k$ is a large constant,
$I_{N}$ denotes the size $N$ identity matrix, and all unmarked entries are 0s.

Recall that a v-structure is a pair of directed edges that share a target node.
We claim that two edges $(i,j), (u,v)$ form a v-structure in $E'$, pointing towards $i=u$, if this inner product takes on the maximum value
\begin{equation}
    \left\langle 
    (W_K h^{E,1})_{i,j},
    (W_Q h^{E,1})_{u,v}
    \right\rangle
    = 3.
\end{equation}
Suppose both edges $(i,j)$ and $(u,v)$ still remain in $C'$.
There are two components to consider.
\begin{enumerate}
    \item If $i=u$, then their shared node contributes $+1$ to the inner product (prior to scaling by $k$).
    If $j=v$, then the inner product accrues $-1$.
    \item
    Nodes that do not share the same endpoint contribute 0 to the inner product.
    Of edges that share one node, only endpoints that match $\blacktriangleright$ at the starting node, or $\bullet$ at the ending node
    contribute $+1$ to the inner product each.
    We provide some examples below.
    \begin{center}\begin{tabular}{cccr}
        $(e_i,e_j)$ & $(e_u,e_v)$ & contribution & note \\
        \midrule
        $(\blacktriangleright, \bullet)$ & $(\bullet, \blacktriangleright)$ & $0$ & no shared node\\
        $(\bullet, \blacktriangleright)$ & $(\bullet, \blacktriangleright)$ & $0$ & wrong endpoints \\
        $(\bullet, \bullet)$ & $(\bullet, \bullet)$ & $1$ & one correct endpoint \\
        $(\blacktriangleright, \bullet)$ & $(\blacktriangleright, \bullet)$ & $2$ & v-structure \\
    \end{tabular}\end{center}
\end{enumerate}
All edges with endpoints $\varnothing$
were ``removed'' in step 1, resulting in an inner product of zero, since their node embeddings were set to zero.
We set $k$ to some large constant (empirically, $k^2=1000$ is more than enough)
to ensure that after softmax scaling,
$\sigma_{e, e'} > 0$ only if
$e, e'$ form a v-structure.

Given ordered pair $e=(i,j)$,
let $V_{i}\subset V$ denote the set of nodes that form a v-structure with $e$ with shared node $i$.
Note that $V_i$ excludes $j$ itself, since setting of $W_K, W_Q$ exclude edges that share both nodes.
We set $W_V$ to the identity, and
we multiply by attention weights $\sigma$ to obtain
\begin{equation}
    (W_V h^{E,1} \sigma )_{e=(i,j)} = 
    \left[\begin{array}{c}
        \vdots \\
        \hline
        \text{one-hot}_N(i) \\
        \hline
        \alpha_j \cdot \text{binary}_N(V_j)
    \end{array}\right]
\end{equation}
where $\text{binary}_N(S)$ denotes the $N$-dimensional binary vector with ones at elements in $S$, and the scaling factor
\begin{equation}
    \alpha_j = (1/\|V_j\|) \cdot \mathbbm{1}\{ \|V_j\| > 0 \} \in [0,1]
\end{equation} results from softmax normalization.
We set
\begin{equation}
\label{eq:usual-wo}
    W_O = \left[\begin{array}{rr}
         \mathbf{0}_{N+6} \\
         & 0.5 \cdot I_{N}
    \end{array}\right]
\end{equation}
to preserve the original endpoint values,
and to distinguish between the edge's own
node identity and newly recognized v-structures.
To summarize, the output of this row attention layer is
\begin{equation*}
    \text{Attn}_\text{row}(X_{\cdot,c}) = X_{\cdot,c} +
        W_O W_V X_{\cdot,c} \cdot \sigma,
\end{equation*}
which is equal to its input $h^{E,1}$ plus additional positive values $\in(0,0.5)$ in the last $N$ positions that indicate the presence of v-structures that exist in the overall $E'$.

Our final step is to ``copy'' newly assigned edge directions into all the edges.
We set the $\phi_E$ column attention, $\text{FFN}_E$ and the $\phi_\rho$ attentions to the identity mapping.
We also set $W_{\rho,2}$ to a $d\times d$ zeros matrix, stacked on top of a $d\times d$ identity matrix.
This passes the output of the $\phi_E$ row attention, aggregated over subsets, directly to $\text{FFN}_{\phi,2}$.

For endpoint dimensions $\mathbf{e}=[6]$,
we let $\text{FFN}_{\phi,2}$ implement 
\begin{equation}
    \text{FFN}_{\rho,2}(X_{\mathbf{e},u,v}) = \begin{cases}
        [ 0, 0, 1, 0, 1, 0]^T - X_{\mathbf{e},u,v} &
        0 < \sum_{i>N+6} X_{i,u,v} < 0.5\\
        0 & \text{ otherwise.}
    \end{cases}
\end{equation}
Subtracting $X_{\mathbf{e},u,v}$ ``erases''
the original endpoints and replaces them with 
$(\blacktriangleright, \bullet)$ after the update
\begin{equation*}
    h^{E,1}_{i,j} \lto h^{E,1}_{i,j} + m^{\rho\to E,1}_{i,j}.
\end{equation*}
The overall operation translates to checking whether \emph{any} v-structure points towards $i$,
and if so, assigning edge directions accordingly.
For dimensions $i > 6$,
\begin{equation}
    \text{FFN}_{\rho,2}(X_{i,u,v}) = \begin{cases}
        - X_{i,u,v} & X_{i,u,v} \le 0.5 \\
        0 & \text{ otherwise,}
    \end{cases}
\end{equation}
effectively erasing the stored v-structures
from the representation
and remaining consistent to (\ref{eq:constructed-embedding}).

At this point, we have copied all v-structures once.
However, our orientations are not necessarily symmetric. For example, given v-structure $i\to j\lto k$, our model orients edges $(j,i)$ and $(j,k)$, but not $(i,j)$ or $(k,j)$.

The simplest way to symmetrize these edges (for the writer and the reader) is to run another axial attention block, in which we focus on
v-structures that point towards the second node of a pair.
The only changes are as follows.
\begin{itemize}
    \item For $W_K$ and $W_Q$, we swap columns 1-3 with 4-6, and columns 7 to $N+6$ with the last $N$ columns.
    \item $(h^{E,2}\sigma)_{i,j}$ sees the third and fourth blocks swapped.
    \item $W_O$ swaps the $N\times N$ blocks that correspond to $i$ and $j$'s node embeddings.
    \item $\text{FFN}_{\rho,3}$ sets the endpoint embedding to $[0,1,0,0,0,1]^T - X_{\mathbf{e},u,v}$ if $i=7, ... , N+6$ sum to a value between 0 and 0.5.
\end{itemize}
The result is $h^{E,3}$ with all v-structures oriented symmetrically, satisfying \ref{claim:consistency}.

\paragraph{Step 3: Orientation propagation}

To propagate orientations, we would like to identify cases $(i,j), (i,k)\in E', (j,k)\not\in E'$ with shared node $i$ and corresponding endpoints $(\blacktriangleright, \bullet), (\bullet, \bullet)$.
We use $\phi_E$ to identify triangles,
and $\phi_\rho$ to identify edges $(i,j), (i,k)\in E'$ with the desired endpoints, while ignoring triangles.

\paragraph{Marginal layer}

The row attention in $\phi_E$ fixes a subset $t$ and varies the edge $(i,j)$.

Given edge $(i,j)$, we want to extract all $(i,k)$ that share node $i$.
We set the key and query
attention matrices to
\begin{equation}
\label{eq:graph-layer-orientation}
    W_K, W_Q = k \cdot \left[\begin{array}{rrrrrrrr}
         0 & 1 & 1 & 0 & 1 & 1\\
         \vdots \\
         &&&&&& I_{N} \\
         &&&&&&& \pm I_{N} \\
    \end{array}\right].
\end{equation}
We set $W_V$ to the identity to obtain
\begin{equation}
    (W_V h^{E} \sigma )_{e=(i,k)} = 
    \left[\begin{array}{c}
        \vdots \\
        \hline
        \vdots \\
        \hline
        \text{one-hot}_N(i) \\
        \hline
        \alpha_k \cdot \text{binary}_N(V_k) 
    \end{array}\right],
\end{equation}
where $V_k$ is the set of nodes $k$ that share any edge with $i$.
To distinguish between $k$ and $V_k$,
we again set $W_o$ to the same as in (\ref{eq:usual-wo}).
Finally, we set $\text{FFN}_E$ to the identity and pass $h^E$ directly to $\phi_\rho$.
To summarize, we have $h^E$
equal to its input, with values $\in(0,0.5)$ in the last $N$ locations indicating 1-hop neighbors of each edge.

\paragraph{Global layer}

Now we would like to identify cases $(i,k), (j,k)$ with corresponding endpoints $(\bullet, \blacktriangleright), (\bullet, \bullet)$.
We set the key and query
attention matrices
\begin{align}
\label{eq:feature-layer-orientation}
    \MoveEqLeft[1]
    W_K = k \cdot \left[\begin{array}{rrrrr}
         0 & 0 & 1 \\
         \vdots \\
         &&& I_{N} \\
         &&&& I_{N}
    \end{array}\right]
    &
    W_Q = k \cdot \left[\begin{array}{rrrrrrrr}
         0 & 1 & -1 & 0 & 1 & -1\\
         \vdots \\
         &&&&&& I_{N} \\
         &&&&&&& -I_{N}
    \end{array}\right].
\end{align}

The key allows us to check that endpoint $i$ is directed, and the query allows us to check that $(i,k)$ exists in $C'$, and does not already point elsewhere.
After softmax normalization, for sufficiently large $k$, we obtain $\sigma_{(i,j), (i,k)} > 0$ if and only if $(i,k)$ should be oriented $(\bullet, \blacktriangleright)$,
and the inner product attains the maximum possible value
\begin{equation}
    \left\langle 
    (W_K h^{\rho})_{i,j},
    (W_Q h^{\rho})_{i,k}
    \right\rangle = 2.
\end{equation}
We consider two components.
\begin{enumerate}
    \item If the endpoints match our desired endpoints,
    we gain a $+1$ contribution to the inner product.
    \item A match between the first nodes contributes $+1$. If the second node shares any overlap (either same edge, or a triangle), then a negative value would be added to the overall inner product.
\end{enumerate}
Therefore, we can only attain the maximal inner product if only one edge is directed,
and if there exists no triangle.

We set $W_o$ to the same as in (\ref{eq:usual-wo}), and we add $h^\rho$ to the input of the next $\phi_E$.
To summarize, we have $h^\rho$
equal to its input, with values $\in(0,0.5)$ in the last $N$ locations indicating incoming edges.

\paragraph{Orientation assignment}

Our final step is to assign our new edge orientations.
Let the column attention take on the identity mapping.
For endpoint dimensions $\mathbf{e}=(4,5,6)$,
we let $\text{FFN}_\rho$ implement 
\begin{equation}
\label{eq:orientation-ffn}
    \text{FFN}_\rho(X_{\mathbf{e},u,v}) = \begin{cases}
        [0, 0, 1]^T - X_{\mathbf{e},u,v} &
        0 < \sum_{i>N+6} X_{i,u,v} < 0.5 \\
        0 & \text{ otherwise.}
    \end{cases}
\end{equation}
This translates to checking whether any incoming edge points towards $v$,
and if so, assigning the new edge direction accordingly.
For dimensions $i > 6$,
\begin{equation}
    \text{FFN}_\rho(X_{i,u,v}) = \begin{cases}
        0 & X_{i,u,v} \le 0.5 \\
        X_{i,u,v} & \text{ otherwise,}
    \end{cases}
\end{equation}
effectively erasing the stored assignments
from the representation.
Thus, we are left with $h^{E,\ell}$ that
conforms to the same format as the initial embedding in (\ref{eq:constructed-embedding}).

To symmetrize these edges, we run another axial attention block, in which we focus on
paths that point towards the second node of a pair.
The only changes are as follows.
\begin{itemize}
    \item For $\phi_E$ layer $W_K$ and $W_Q$ (\ref{eq:graph-layer-orientation}), we swap $I_N$ and $\pm I_N$.
    \item For $\phi_\rho$ layer $W_K$ and $W_Q$
    (\ref{eq:feature-layer-orientation}), we swap $I_N$ and $\pm I_N$.
    \item $W_O$ swaps the $N\times N$ blocks that correspond to $i$ and $j$'s node embeddings.
    \item For $\text{FFN}_\rho$ (\ref{eq:orientation-ffn}), we let $\mathbf{e}=(1,2,3)$ instead.
\end{itemize}
The result is $h^{E}$ with symmetric 1-hop orientation propagation, satisfying \ref{claim:consistency}.
We may repeat this procedure $k$ times to capture $k$-hop paths.

To summarize, we used axial attention block 1 to recover the undirected skeleton $C'$, blocks 2-3 to identify and copy v-structures in $E'$, and all subsequent $L-3$ layers to propagate orientations on paths up to $\lfloor (L-3)/2 \rfloor$ length.
Overall, this particular construction requires $O(N)$ width for $O(L)$ paths.

\end{proof}

\paragraph{Final remarks}

Information theoretically, it should be possible to encode the same information in $\log N$ space, and achieve $O(\log N)$ width. For ease of construction, we have allowed for wider networks than optimal.
On the other hand, if we increase the width and encode each edge symmetrically, e.g. $(e_i, e_j, e_j, e_i \mid i,j,j,i)$,
we can reduce the number of blocks by half, since we no longer need to run each operation twice. However, attention weights scale quadratically, so we opted for an asymmetric construction.

Finally, a strict limitation of our model is that it only considers 1D pairwise interactions. In the graph layer, we cannot compare different edges' estimates at different times in a single step.
In the feature layer, we cannot compare $(i,j)$ to $(j,i)$ in a single step either.
However, the graph layer does enable us to compare all edges at once (sparsely),
and the feature layer looks at a time-collapsed version of the whole graph.
Therefore, though we opted for this design for computational efficiency, we have shown that it is able to capture significant graph reasoning.

\subsection{Robustness and stability}
\label{subsec:misspec}

We discuss the notion of stability informally, in the context of \citet{causation-prediction-search}.
There are two cases in which our framework may receive erroneous inputs: low/noisy data settings, and functionally misspecified situations.
We consider our framework's empirical robustness to these cases, in terms of recovering the skeleton and orienting edges.


In the case of noisy data, edges may be erroneously added, removed, or misdirected from marginal estimates $E'$.
Our framework provides two avenues to mitigating such noise.
\begin{enumerate}
    \item We observe that global statistics can be estimated reliably in low data scenarios.
    For example, Figure~\ref{fig:estimate-ablation} suggests that 300 examples suffice to provide a robust estimate over 100 variables in our synthetic settings.
    Therefore, even if the marginal estimates are erroneous, the neural network can learn the skeleton from the global statistics.
    \item Most classical causal discovery algorithms are not stable with respect to edge orientation assignment.
    That is, an error in a single edge may propagate throughout the graph.
    Empirically, we observe that the majority vote of \textsc{Gies} achieves reasonable accuracy even without any training, while \textsc{Fci} suffers in this assessment (Table~\ref{table:synthetic-edge}).
    However both \ours{} \textsc{(Gies)} and \ours{} \textsc{(Fci)} achieve high edge accuracy. Therefore, while the underlying algorithms may not be stable with respect to edge orientation, our pretrained aggregator seems to be robust.
\end{enumerate}


It is also possible that our global statistics and marginal estimates make misspecified assumptions regarding the data generating mechanisms.
The degree of misspecification can vary case by case, so it is hard to provide any broad guarantees about the performance of our algorithm, in general.
However, we can make the following observation.

If two variables are independent, $X_i\indep X_j$, they are independent, e.g. under linear Gaussian assumptions.
If $X_i, X_j$ exhibit more complex functional dependencies, they may be erroneously deemed independent.
Therefore, any systematic errors are necessarily one-sided, and the model can learn to recover the connectivity based on global statistics.

\section{Experimental details}
\label{sec:experiment-appendix}

\subsection{Synthetic data generation}
\label{subsec:causal-mechanisms}

Synthetic datasets were generated using code from \textsc{Dcdi}~\citep{dcdi}, which extended the Causal Discovery Toolkit data generators to interventional data~\citep{cdt}.

We considered the following causal mechanisms.
Let $y$ be the node in question, let $X$ be its parents, let $E$ be an independent noise variable (details below), and let $W$ be randomly initialized weight matrices.
\begin{itemize}
    \item Linear: $y = X W + E$.
    \item Polynomial: $y = W_0 + X W_1 + X^2 W_2 + \times E$
    \item Sigmoid: $y = \sum_{i=1}^d W_i \cdot \text{sigmoid}(X_i) +\times E$
    \item Randomly initialized neural network (NN): $y = \text{Tanh}((X, E) W_\text{in}) W_\text{out}$
    \item Randomly initialized neural network, additive (NN additive): $y = \text{Tanh}(X W_\text{in}) W_\text{out} + E$
\end{itemize}
Root causal mechanisms, noise variables, and interventional distributions maintained the \textsc{Dcdi} defaults.
\begin{itemize}
    \item Root causal mechanisms were set to $\text{Uniform}(-2,2)$.
    \item Noise was set to $E\sim 0.4\cdot \mathcal{N}(0,\sigma^2)$ where $\sigma^2\sim\text{Uniform}(1,2)$.
    \item Interventions were applied to all nodes (one at a time) by setting their causal mechanisms to $\mathcal{N}(0,1)$.
\end{itemize}

Ablation datasets with $N>100$ nodes contained 100,000 points each (same as $N=100$).
We set random seeds for each dataset using the hash of the output filename.

\subsection{Related work and baselines}
\label{subsec:baselines}

We considered the following baselines.
All baselines were run using official implementations
published by the authors.

\textbf{\textsc{AVICI}} \citep{avici}
is the most similar method to this work, though there are significant differences in both the causal discovery strategy and the implementation.
Both works simulate diverse datasets for training, which differ in graph topology, causal mechanism, and type of exogenous noise; and both are attention-based architectures~\citep{attention}.
The primary difference is that \textsc{Avici} operates over raw data, while we operate over summary statistics.
While access to raw data may allow for richer modeling of relationships, it also increases the computational cost on large datasets.
In fact, \textsc{Avici}'s model complexity scales quadratically as the number of samples \emph{and} quadratically as the number of nodes. Our model does not explicitly depend on the number of data samples, but scales cubically as the number of nodes, as we attend over richer, pairwise features.

Both \textsc{Avici} and \ours{} are equivariant to the ordering of nodes.
\textsc{Avici} is additionally invariant to the ordering of samples. 
Due to the two-track design of our aggregator, we also include message passing operations between the two input types (global statistics, marginal features), while \textsc{Avici} directly stacks Transformer layers.
Finally, \textsc{Avici} explicitly regularizes for acyclicity.
We follow the ENCO~\citep{enco} formulation (edges $i\to j$ and $j\to i$ cannot co-exist), since it is easier to optimize.
Empirically, our predictions are still 99\% acyclic (Table~\ref{tab:acyclic}).

\textsc{Avici} was run on all test datasets using the authors' pretrained \texttt{scm-v0} model, recommended for ``arbitrary real-valued data.''
Note that this model is different from the models described in their paper (denoted \textsc{Avici-L} and \textsc{Avici-R}), as it was trained on \emph{all} of their synthetic data, including test sets.
We sampled 1000 observations per dataset uniformly at random, with their respective interventions (the maximum number of synthetic samples used in their original paper), except for Sachs, which used the entire dataset (as in their paper).
Though the authors provided separate weights for synthetic mRNA data, we were unable to use it since we did not preserve the raw gene counts in our simulated mRNA datasets.

\textbf{\textsc{Dcdi}} \citep{dcdi} was trained
on each of the $N=10, 20$ datasets using their published hyperparameters.
We denote the Gaussian and Deep Sigmoidal Flow versions as
\textsc{DCDI-G} and \textsc{DCDI-DSF} respectively.
\textsc{DCDI} could not scale to graphs with $N=100$
due to memory constraints (did not fit on a 32GB V100 GPU).

\textbf{\textsc{DCD-FG}} \citep{dcdfg} was trained
on all of the test datasets using their published hyperparameters.
We set the number of factors to $5, 10, 20$ for each of $N=10, 20, 100$, based on their ablation studies.
Due to numerical instability on $N=100$,
we clamped augmented Lagrangian multipliers
$\mu$ and $\gamma$ to 10
and stopped training if elements of the 
learned adjacency matrix reached \texttt{NaN} values.
After discussion with the authors, we also tried adjusting the $\mu$ multiplier from 2 to 1.1, but the model did not converge within 48 hours.

\textbf{\textsc{DECI}} \citep{deci} was trained on
all of the test datasets using their published hyperparameters.
However, on all $N=100$ cases, the model failed to produce any meaningful results (adjacency matrices nearly all remained 0s with AUCs of 0.5). Thus, we only report results on $N=10, 20$.

\textbf{\textsc{DiffAN}} \citep{diffan} was trained on
the each of the $N=10, 20$ datasets using their published hyperparameters.
The authors write that ``most hyperparameters are hard-coded into
[the] constructor of the \textsc{DiffAN} class and we verified they work across
a wide set of datasets.''
We used the original, non-approximation version of their algorithm
by maintaining \texttt{residue=True} in their codebase.
We were unable to consistently run \textsc{DiffAN}
with both R and GPU support within a Docker container,
and the authors did not respond to questions regarding reproducibility,
so all models were trained on the CPU only.
We observed approximately a 10x speedup in the $<5$ cases that
were able to complete running on the GPU.

\textbf{\textsc{InvCov}} computes inverse
covariance over 1000 examples. This does \emph{not} orient edges, but it is a strong connectivity baseline. We discretize based on ground truth positive rate.

\textbf{\textsc{Corr}} and \textbf{\textsc{D-Corr}} are computed similarly, using global correlation and distance correlation, respectively (See \ref{subsec:global-stat} for details).

\subsection{Metrics}
\label{subsec:metrics-full}

Throughout our evaluation, we compute metrics with respect to the ground truth graph.
This means that in the observational setting, the ``oracle'' value of each metric will vary depending on the size of the equivalence class (e.g. if multiple graphs are observationally equivalent, the expected SHD is $>0$; see Section~\ref{sec:identifiability} for more analysis).
While it is more correct to evaluate (observational) models with respect to the CPDAG that is implied by the predicted DAG, there were several reasons we also chose to evaluate with respect to the ground truth DAG.

\begin{itemize}
\item This practice is quite common among continuous causal discovery algorithms, regardless of whether interventional data are available~\citep{avici,dcdi,bacadi} or not~\citep{diffan}.
This may be primarily because their predicted graphs contain continuous probabilities, and metrics that reflect uncertainty (mAP, AUROC) are difficult to translate to the equivalence class.
\item It has been reported that common simulators produce ``identifiable'' linear Gaussian datasets, if the data are unstandardized (hence the VarSort baseline, \citet{varsort}). This would mean that the MEC is actually smaller than expected. While it is possible to normalize all the data, it is unknown whether there are additional artifacts of the data generating process that may influence the (empirical) identifiability of the graphs, and thus, the ``true'' size of the equivalence class.
\item Our models and baselines are evaluated on the same graphs, but not all can incorporate interventional information. For comparability, we evaluated all predictions against the same ground truth, while noting that the equivalence class is larger in the observational case.
\item \textsc{Fci}, \textsc{Gies}, and \textsc{VarSort} don't necessarily produce acyclic graphs, due to bootsrapping. In Appendix~\ref{sec:more-experiments}, we also include the symmetric summary statistics as baselines, to quantify the ``cost'' of not directing edges, and of not using marginal estimates.
In these cases, we cannot compute a corresponding CPDAG.
\item Practically, it was somewhat expensive to compute the CPDAG, especially for larger dense graphs ($N=100, E=400$ and beyond).
\end{itemize}

\subsection{Training and hardware details}
\label{subsec:pretraining}

Hyperparameters and architectural choices
were selected by training the model
on 20\% of the the training and validation data
for approximately 50k steps (several hours).
We considered the following parameters in sequence.
\begin{itemize}
    \item learned positional embedding vs.
    sinusoidal positional embedding
    \item number of layers $\times$ number of heads:
    $\{4, 8\} \times \{4, 8\}$
    \item learning rate $\eta = \{1e-4, 5e-5, 1e-5\}$
\end{itemize}

For our final model, we selected learned positional embeddings, 4 layers, 8 heads, and learning rate $\eta=1e-4$.

The models were trained across 2 NVIDIA RTX A6000 GPUs
and 60 CPU cores.
We used the GPU exclusively for running the aggregator, and retained all classical algorithm execution on the CPUs (during data loading).
The total pretraining time took approximately
14 hours for the final FCI model and 16 hours for the final GIES model.

For finetuning, we used rank $r=2$ adapters on the axial attention model's key, query, and feedforward weights~\citep{hu2022lora}.
We trained until convergence on the validation set (no improvement for 100 epochs), which took 4-6 hours with 40 CPUs and around 10 hours with 20 CPUs.
We used a single NVIDIA RTX A6000 GPU, but the bottleneck was CPU availability.

For the scope of this paper, our models and datasets
are fairly small.
We did not scale further due to hardware constraints.
Our primary bottlenecks to scaling up lay in availability of CPU cores and networking speed
across nodes, rather than GPU memory or utilization.
The optimal CPU:GPU ratio for \ours{} ranges from 20:1 to 40:1.

We are able to run inference comfortably over $N=500$ graphs with $T=500$ subsets of $k=5$ nodes each, on a single 32GB V100 GPU.
For runtime analysis, we used a batch size of 1, with 1 data worker per dataset.
Our runtime could be further improved if we amortized the GPU utilization across batches.

\subsection{Choice of classical causal discovery algorithm}
\label{subsec:classical-alg}

For training, we selected FCI~\citep{fci} as the underlying discovery algorithm
in the observational setting over GES~\citep{ges}, GRaSP~\citep{grasp}, and LiNGAM~\citep{lingam}
due to its speed and superior downstream performance.
We hypothesize this may be due to its richer
output (ancestral graph) providing
more signal to the Transformer model.
We also tried Causal Additive Models~\citep{cam}, but its runtime was too slow for consistent GPU utilization.
Observational algorithm implementations were provided by the causal-learn library~\citep{causallearn}.
The code for running these alternative classical algorithms is available in our codebase.

We selected GIES as the discovery algorithm
in the interventional setting because an efficient Python implementation was readily available at \url{https://github.com/juangamella/gies}.

We tried incorporating implementations from the Causal Discovery Toolbox via a Docker image~\citep{cdt}, but there was excessive overhead associated with calling an R subroutine and reading/writing the inputs/results from disk.

Finally, we considered other independence tests for richer characterization, such as kernel-based methods. However, due to speed, we chose to remain with the default Fisherz conditional independence test for FCI, and BIC for GIES~\citep{bic}.

\subsection{Sampling procedure}
\label{subsec:sampling}

\paragraph{Selection scores:}
We consider three strategies for computing
selection scores $\alpha$.
We include an empirical comparison of these strategies in Table~\ref{table:sampling}.
\begin{enumerate}
    \item Random selection: $\alpha$ is an $\vsq$ matrix of ones.
    \item Global-statistic-based selection: $\alpha = \rho$.
    \item Uncertainty-based selection: $\alpha = \hat H(E_t$),
    where $H$ denotes the information entropy
    \begin{equation}
        \alpha_{i,j} = -\sum_{e\in\{0,1,2\}} p(e) \log p(e).
    \end{equation}
\end{enumerate}

Let $c_{i,j}^t$ be the number of times edge $(i,j)$
was selected in $S_1\dots S_{t-1}$, and let
$\alpha^t = \alpha / \sqrt{c_{i,j}^t}$.
We consider two strategies for selecting $S_t$ based on $\alpha_t$.

\paragraph{Greedy selection:} 
Throughout our experiments, we used a greedy algorithm for subset selection.
We normalize probabilities to 1 before constructing each Categorical.
Initialize
\begin{equation}
S_t \leftarrow
    \{ i : i\sim\text{Categorical}(
    \alpha_1^t \dots \alpha_N^t ) \}.
\end{equation}
where $\alpha_i^t = \sum_{j\ne i \in V} \alpha_{i,j}^t$.
While $\lvert S_t \rvert < k$, update
\begin{equation}
S_t \leftarrow S_t \cup 
    \{ j : j \sim\text{Categorical}(
        \alpha_{1,S_t}^t \dots \alpha_{N,S_t}^t )
    )
\end{equation}
where
\begin{equation}
    \alpha_{j,S_t} = \begin{cases}
        \sum_{i\in S_t} \alpha_{i,j}^t & j \not\in S_t \\
        0 & \text{otherwise}.
    \end{cases}
\end{equation}

\paragraph{Subset selection:}
We also considered the following subset-level selection procedure, and observed minor performance gain for significantly longer runtime (linear program takes around 1 second per batch).
Therefore, we opted for the greedy method instead.

We solve the following integer linear program to select a subset $S_t$ of size $k$ that maximizes $\sum_{i\in S_t} \alpha_{i,j}^t$.
Let $\nu_{i}\in\{ 0, 1 \}$ denote the selection of node $i$, and
let $\epsilon_{i,j}\in\{0,1\}$ denote the selection of edge $(i,j)$.
Our objective is to
\begin{equation*}
\begin{array}{ll@{}ll}
\text{maximize}  & \sum_{i,j} a_{i,j}^t \cdot \epsilon_{i,j} & \\
\text{subject to}& \sum_i \nu_i = k & \text{ \small subset size}\\
                 & \epsilon_{i,j} \ge \nu_i + \nu_j - 1
                 & \text{ \small node-edge consistency}\\
                 & \epsilon_{i,j} \le \nu_i \\
                 & \epsilon_{i,j} \le \nu_j, \\
                 & \nu_{i}\in\{ 0, 1 \} \\
                 & \epsilon_{i,j}\in\{ 0, 1 \}
\end{array}
\end{equation*}
for $i,j \in V\times V$, $i\in V$.
$S_t$ is the set of non-zero indices in $\nu$.

\begin{table*}[ht]
\vskip -0.1in
\setlength\tabcolsep{1.5 pt}
\caption{Comparison between heuristics-based sampler (random and inverse covariance) vs. model confidence-based sampler.
The suffix \textsc{-L} indicates the greedy confidence-based sampler.
Each setting encompasses 5 distinct Erd\H{o}s-R\'enyi graphs.
The symbol $\dagger$ indicates that \ours{}
was not pretrained on this setting.
Bold indicates best of all models considered (including baselines not pictured).}
\label{table:sampling}
\begin{center}
\begin{small}
\begin{tabular}{ccl rrr rrr rrr rrr rrr}
\toprule
$N$ & $E$ & Model 
& \multicolumn{3}{c}{Linear} 
& \multicolumn{3}{c}{NN add.}
& \multicolumn{3}{c}{NN non-add.}
& \multicolumn{3}{c}{Sigmoid$^\dagger$}
& \multicolumn{3}{c}{Polynomial$^\dagger$} \\
\cmidrule(l{\tabcolsep}){4-6}
\cmidrule(l{\tabcolsep}){7-9}
\cmidrule(l{\tabcolsep}){10-12}
\cmidrule(l{\tabcolsep}){13-15}
\cmidrule(l{\tabcolsep}){16-18}
&&& \multicolumn{1}{c}{mAP$\uparrow$} & \multicolumn{1}{c}{OA$\uparrow$} & \multicolumn{1}{c}{shd$\downarrow$} &\multicolumn{1}{c}{mAP$\uparrow$} & \multicolumn{1}{c}{OA$\uparrow$} & \multicolumn{1}{c}{shd$\downarrow$} &\multicolumn{1}{c}{mAP$\uparrow$} & \multicolumn{1}{c}{OA$\uparrow$} & \multicolumn{1}{c}{shd$\downarrow$} &\multicolumn{1}{c}{mAP$\uparrow$} & \multicolumn{1}{c}{OA$\uparrow$} & \multicolumn{1}{c}{shd$\downarrow$} &\multicolumn{1}{c}{mAP$\uparrow$} & \multicolumn{1}{c}{OA$\uparrow$} & \multicolumn{1}{c}{shd$\downarrow$} \\
\midrule
\multirow{4}{*}{10} & \multirow{4}{*}{10} & \ours{}\textsc{-f} & $0.97$& $0.92$& $1.6$& $\textbf{0.95}$& $\textbf{0.92}$& $2.4$& $\textbf{0.92}$& $0.94$& $2.8$& $0.83$& $0.76$& $3.7$& $0.69$& $0.71$& $6.7$\\
&& \ours{}\textsc{-g} & $\textbf{0.99}$& $\textbf{0.94}$& $1.2$& $0.94$& $0.88$& $2.6$& $0.91$& $0.93$& $3.2$& $0.85$& $0.84$& $4.0$& $0.70$& $0.79$& $\textbf{5.8}$\\
\cmidrule(l{\tabcolsep}){3-18}
&& \ours{}\textsc{-f-l} & $0.97$& $0.93$& $\textbf{1.0}$& $0.95$& $0.87$& $2.4$& $\textbf{0.92}$& $\textbf{0.98}$& $3.4$& $0.84$& $0.77$& $3.9$& $\textbf{0.70}$& $0.79$& $\textbf{5.8}$\\
&& \ours{}\textsc{-g-l} & $0.98$& $0.93$& $1.4$& $0.94$& $0.91$& $2.8$& $0.91$& $0.94$& $4.0$& $\textbf{0.88}$& $\textbf{0.84}$& $\textbf{3.6}$& $0.70$& $\textbf{0.80}$& $\textbf{5.8}$\\
\midrule\midrule
\multirow{4}{*}{10} & \multirow{4}{*}{40} & \ours{}\textsc{-f} & $0.90$& $0.87$& $14.4$& $0.91$& $0.94$& $11.2$& $0.87$& $0.86$& $16.0$& $\textbf{0.81}$& $0.85$& $22.7$& $0.81$& $0.92$& $33.4$\\
&& \ours{}\textsc{-g} & $\textbf{0.94}$& $\textbf{0.91}$& $\textbf{12.8}$& $0.91$& $\textbf{0.95}$& $\textbf{10.4}$& $\textbf{0.89}$& $\textbf{0.89}$& $17.2$& $0.81$& $\textbf{0.87}$& $24.5$& $0.89$& $0.93$& $29.5$\\
\cmidrule(l{\tabcolsep}){3-18}
&& \ours{}\textsc{-f-l} & $0.91$& $0.90$& $15.6$& $\textbf{0.91}$& $0.92$& $15.8$& $0.88$& $0.86$& $14.2$& $0.81$& $0.84$& $23.2$& $0.82$& $0.93$& $33.8$\\
&& \ours{}\textsc{-g-l} & $0.93$& $0.91$& $13.4$& $0.91$& $0.93$& $\textbf{10.4}$& $0.88$& $0.85$& $16.2$& $0.79$& $0.83$& $25.5$& $\textbf{0.90}$& $\textbf{0.94}$& $28.3$\\
\midrule\midrule
\multirow{4}{*}{20} & \multirow{4}{*}{20} & \ours{}\textsc{-f} & $0.97$& $0.92$& $3.2$& $0.94$& $0.97$& $3.2$& $0.84$& $0.93$& $7.2$& $0.84$& $0.85$& $7.6$& $\textbf{0.71}$& $\textbf{0.80}$& $10.2$\\
&& \ours{}\textsc{-g} & $0.97$& $0.89$& $3.0$& $\textbf{0.94}$& $0.95$& $3.4$& $0.83$& $0.94$& $7.8$& $0.84$& $0.83$& $8.1$& $0.69$& $0.78$& $10.1$\\
\cmidrule(l{\tabcolsep}){3-18}
&& \ours{}\textsc{-f-l} & $0.97$& $\textbf{0.92}$& $2.8$& $0.93$& $0.95$& $3.8$& $\textbf{0.85}$& $0.94$& $6.8$& $\textbf{0.85}$& $0.85$& $\textbf{7.5}$& $0.67$& $0.78$& $\textbf{9.9}$\\
&& \ours{}\textsc{-g-l} & $\textbf{0.97}$& $0.90$& $\textbf{2.6}$& $\textbf{0.94}$& $\textbf{0.98}$& $3.4$& $0.83$& $\textbf{0.97}$& $7.0$& $0.84$& $0.84$& $7.9$& $0.67$& $0.79$& $10.6$\\
\midrule\midrule
\multirow{4}{*}{20} & \multirow{4}{*}{80} & \ours{}\textsc{-f} & $0.86$& $\textbf{0.93}$& $29.6$& $0.55$& $0.90$& $73.6$& $0.72$& $\textbf{0.93}$& $51.8$& $\textbf{0.77}$& $0.85$& $42.8$& $0.61$& $0.89$& $61.8$\\
&& \ours{}\textsc{-g} & $\textbf{0.89}$& $0.92$& $\textbf{26.8}$& $\textbf{0.58}$& $0.88$& $71.4$& $0.73$& $0.92$& $50.6$& $0.76$& $0.84$& $45.0$& $\textbf{0.65}$& $\textbf{0.89}$& $60.1$\\
\cmidrule(l{\tabcolsep}){3-18}
&& \ours{}\textsc{-f-l} & $0.86$& $0.92$& $32.0$& $0.55$& $\textbf{0.90}$& $74.0$& $0.74$& $0.93$& $49.2$& $0.76$& $\textbf{0.87}$& $\textbf{41.8}$& $0.59$& $0.88$& $62.3$\\
&& \ours{}\textsc{-g-l} & $\textbf{0.89}$& $0.92$& $28.4$& $0.58$& $0.89$& $71.6$& $0.75$& $0.92$& $49.4$& $0.75$& $0.85$& $45.7$& $\textbf{0.65}$& $0.88$& $60.6$\\
\bottomrule
\end{tabular}
\end{small}
\end{center}
\vskip -0.2in
\end{table*}

The final algorithm used the greedy selection strategy, with the first half of batches sampled according to global statistics, and the latter half sampled randomly, with visit counts shared.
This strategy was selected heuristically, and we did not observe significant improvements or drops in performance when switching to other strategies (e.g. all greedy statistics-based, greedy uncertainty-based, linear program uncertainty-based, etc.)

Table~\ref{table:sampling} compares the heuristics-based greedy sampler (inverse covariance + random) with the model uncertainty-based greedy sampler.
Runtimes are plotted in Figure~\ref{fig:runtime-l}.
The latter was run on CPU only, since it was non-trivial to access the GPU within a PyTorch data loader.
We ran a forward pass to obtain an updated selection score every 10 batches, so this accrued over 10 times the number of forward passes, all on CPU.
With proper engineering, this model-based sampler is expected to be much more efficient than reported.
Still, it is faster than nearly all baselines.

\begin{figure*}[hb]
    \centering
    \includegraphics[height=2in]{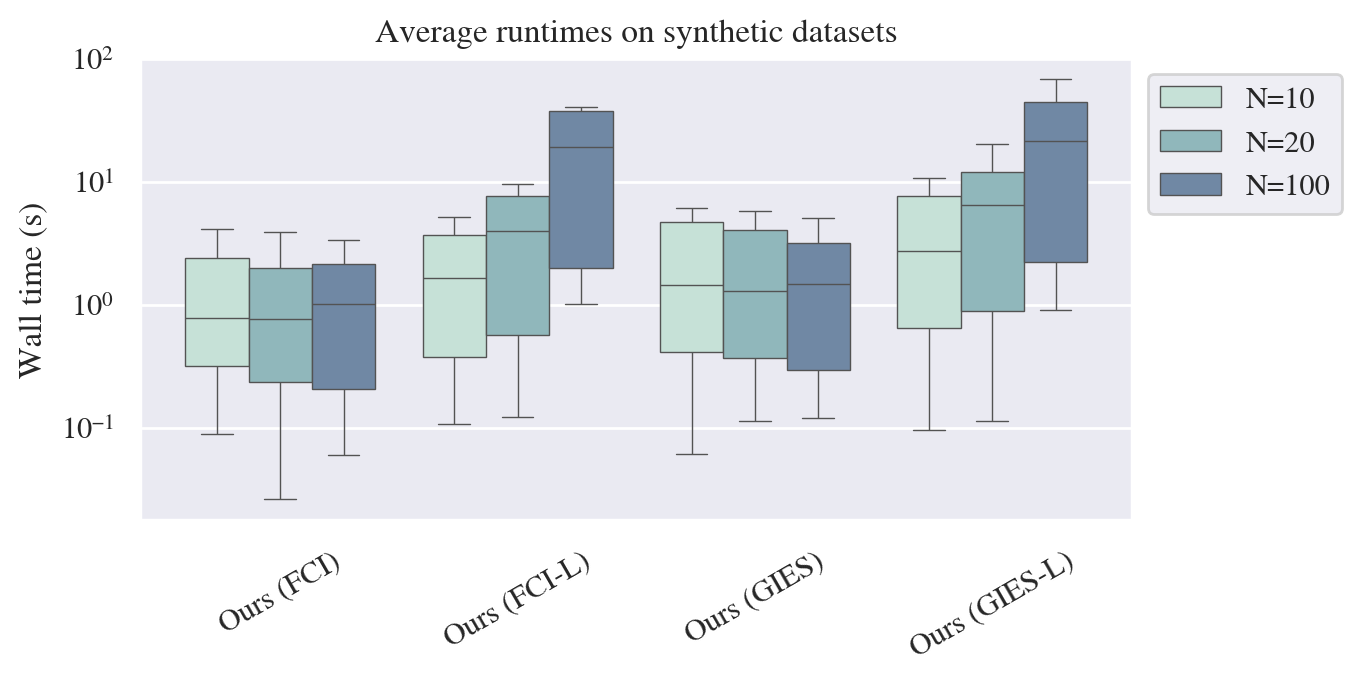}
    \caption{Runtime for heuristics-based greedy sampler vs. model uncertainty-based greedy sampler (suffix \textsc{-l}). For sampling, the model was run on CPU only, due to the difficulty of invoking GPU in the PyTorch data sampler.}
    \label{fig:runtime-l}
\end{figure*}



\section{Additional analyses}
\label{sec:more-experiments}

\subsection{Choice of global statistic}
\label{subsec:global-stat}

We selected inverse covariance as our global feature due to its ease of computation and its relationship to partial correlation.
For context, we also provide the performance analysis of several alternatives.
Tables \ref{table:synthetic-naiveC} and \ref{table:synthetic-naiveD} compare the results of different graph-level statistics on our synthetic datasets.
Discretization thresholds for SHD were obtained by computing the $p^\text{th}$ quantile of the computed values,
where $p = 1-(E/N)$.
This is not entirely fair, as no other baseline receives the same calibration, but these ablation studies only seek to compare state-of-the-art causal discovery methods with the ``best'' possible (oracle) statistical alternatives.

\textbf{\textsc{Corr}} refers to global correlation,
\newcommand{\expect}[1]{\mathbb{E}\left(#1\right)}
\begin{equation}
    \rho_{i,j}
    = \frac{\expect{X_i X_j} - \expect{X_i}\expect{X_j}}{\sqrt{\expect{X_i^2} - \expect{X_i}^2} \cdot \sqrt{\expect{X_j^2} - \expect{X_j}^2}}.
\end{equation}

\textbf{\textsc{D-Corr}} refers to distance correlation, computed between all pairs of variables.
Distance correlation captures both linear and non-linear dependencies, and $\textsc{D-Corr}(X_i, X_j) = 0$
if and only if $X_i \indep X_j$.
Please refer to~\cite{dcor} for the full derivation.
Despite its power to capture non-linear dependencies, we opted not to use \textsc{D-Corr} because it is quite slow to compute between all pairs of variables.

\textbf{\textsc{InvCov}} refers to inverse covariance, computed globally,
\begin{equation}
    \rho
    = \expect{(X - \expect{X}) (X - \expect{X})^T}^{-1}.
\end{equation}
For graphs $N < 100$, inverse covariance was computed directly using NumPy.
For graphs $N\ge100$, inverse covariance was computed using Ledoit-Wolf shrinkage at inference time~\cite{ledoit-wolf}.
Unfortunately we only realized this after training our models, so swapping to Ledoit-Wolf leads to some distribution shift (and drop in performance) on \ours{} results for large graphs.

\subsection{Visualization of denoising statistical features}

Figure~\ref{fig:gies-resolve} illustrates the input features and our predictions for a $N=10, E=10$ linear graph.
Compared to the inputs, our method is able to produce a much cleaner graph.
\textsc{Gies} may orient edges the wrong direction in some of the bootstrap samples, but \ours{} can generally identify the right direction.
This example also illustrates how \ours{} can triangulate between \textsc{Gies} and the global statistic, to avoid naive predictions of edges wherever the global statistic has a high value.
In the example marked in teal, though the global statistic has a moderate value, the edge is absent from a substantial number of estimates (if we consider both directions).
This reflects the theory in \ref{subsec:resolve-proof} that edges absent from some marginal estimate should be absent from the final skeleton.
\ours{} was able to (rather confidently) reject the edge in the final graph.

In contrast to \textsc{Gies}, \textsc{Fci} tends to identify much fewer edges (Figure~\ref{fig:fci-resolve}).
Note that this does not hurt performance on metrics, as continuous metrics consider ``sliding'' thresholds, while discrete metrics are computed with respect to the true edge rate (for the oracle \textsc{Fci} baseline).
However, it does highlight that \textsc{Fci} has a much lower signal-to-noise ratio when used alone.

\begin{figure*}[ht]
\centering
\includegraphics[width=0.7\linewidth]{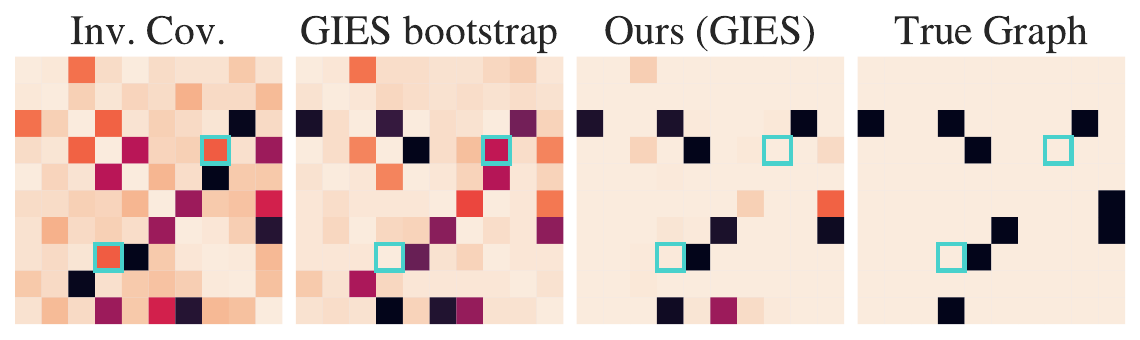}
\caption{Left to right: magnitude value of global features with diagonals zeroed; \textsc{Gies} run over all variables via non-parametric bootstrapping (frequency of edge); \ours{} predictions; and the ground truth.
\ours{} is able to denoise the input features much better than naive aggregation schemes.}
\label{fig:gies-resolve}
\end{figure*}

\begin{figure*}[ht]
\centering
\includegraphics[width=0.9\linewidth]{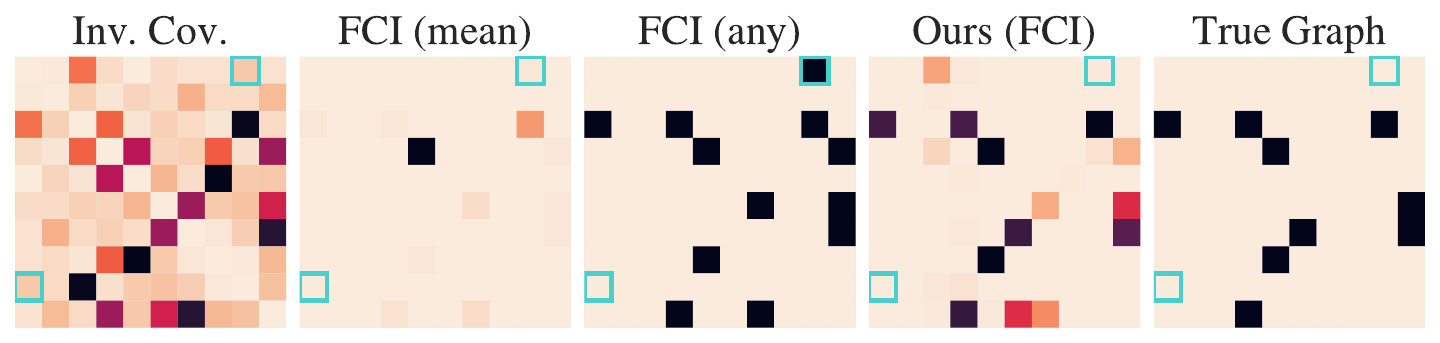}
\caption{Left to right: magnitude value of global features with diagonals zeroed; \textsc{Fci} run over all variables via non-parametric bootstrapping, aggregation via frequency of edge  vs. binarized (appear or not); \ours{} predictions; and the ground truth.
\textsc{Fci} produces sparser outputs than \textsc{Gies}, so \textsc{Fci} alone is less able to distinguish between noisy and genuine edges.}
\label{fig:fci-resolve}
\end{figure*}

\begin{table}[t]
\vskip -0.1in
\setlength\tabcolsep{3 pt}
\caption{Synthetic experiments, edge direction accuracy (higher is better).
All standard deviations were within $0.2$.
The symbol $\dagger$ indicates that \ours{}
was not pretrained on this setting.
}
\label{table:synthetic-edge}
\begin{center}
\begin{small}
\begin{tabular}{ccl r r r r r}
\toprule
$N$ & $E$ & Model 
& Linear & NN add & NN & Sig.$^\dagger$ & Poly.$^\dagger$ \\
\midrule
\multirow{10}{*}{10} & \multirow{10}{*}{10} &
\textsc{Dcdi-G} & $0.74$& $0.80$& $0.85$& $0.41$& $0.44$\\
&& \textsc{Dcdi-Dsf} & $0.79$& $0.62$& $0.68$& $0.38$& $0.39$\\
&& \textsc{Dcd-Fg} & $0.50$& $0.47$& $0.70$& $0.43$& $0.54$\\
&& \textsc{DiffAn} & $0.61$& $0.55$& $0.26$& $0.53$& $0.47$\\
&& \textsc{Deci} & $0.50$& $0.43$& $0.62$& $0.63$& $0.75$\\
&& \textsc{Avici} & $0.80$& $\textbf{0.92}$& $0.83$& $0.81$& $0.75$\\
\cmidrule(l{\tabcolsep}){3-8}
&& \textsc{Fci*} & $0.52$& $0.43$& $0.41$& $0.55$& $0.40$\\
&& \textsc{Gies*} & $0.76$& $0.49$& $0.69$& $0.67$& $0.63$\\
\cmidrule(l{\tabcolsep}){3-8}
&& \ours{} \textsc{(Fci)} & $0.92$& $\textbf{0.92}$& $\textbf{0.94}$& $0.76$& $0.71$\\
&& \ours{} \textsc{(Gies)} & $\textbf{0.94}$& $0.88$& $0.93$& $\textbf{0.84}$& $\textbf{0.79}$\\
\midrule\midrule
\multirow{10}{*}{20} & \multirow{10}{*}{80} &
\textsc{Dcdi-G} & $0.47$& $0.43$& $0.82$& $0.40$& $0.24$\\
&& \textsc{Dcdi-Dsf} & $0.50$& $0.49$& $0.78$& $0.41$& $0.28$\\
&& \textsc{Dcd-Fg} & $0.58$& $0.65$& $0.75$& $0.62$& $0.48$\\
&& \textsc{DiffAn} & $0.46$& $0.28$& $0.36$& $0.45$& $0.21$\\
&& \textsc{Deci} & $0.30$& $0.47$& $0.35$& $0.48$& $0.57$\\
&& \textsc{Avici} & $0.57$& $0.67$& $0.74$& $0.63$& $0.62$\\
\cmidrule(l{\tabcolsep}){3-8}
&& \textsc{Fci*} & $0.19$& $0.19$& $0.22$& $0.33$& $0.23$\\
&& \textsc{Gies*} & $0.56$& $0.73$& $0.59$& $0.62$& $0.61$\\
\cmidrule(l{\tabcolsep}){3-8}
&& \ours{} \textsc{(Fci)} & $\textbf{0.93}$& $\textbf{0.90}$& $\textbf{0.93}$& $\textbf{0.85}$& $\textbf{0.89}$\\
&& \ours{} \textsc{(Gies)} & $0.92$& $0.88$& $0.92$& $0.84$& $\textbf{0.89}$\\
\midrule\midrule
\multirow{4}{*}{100} & \multirow{4}{*}{400} &\textsc{Dcd-Fg} & $0.46$& $0.60$& $0.70$& $0.67$& $0.53$\\
&&\textsc{Avici} & $0.61$& $0.68$& $0.72$& $0.54$& $0.42$\\
\cmidrule(l{\tabcolsep}){3-8}
&& \ours{} \textsc{(Fci)} & $0.93$& $0.90$& $0.91$& $\textbf{0.87}$& $0.82$\\
&& \ours{} \textsc{(Gies)} & $\textbf{0.94}$& $\textbf{0.91}$& $\textbf{0.92}$& $\textbf{0.87}$& $\textbf{0.84}$\\
\bottomrule
\end{tabular}
\end{small}
\end{center}
\vskip -0.2in
\end{table}
\begin{table*}[ht]
\setlength\tabcolsep{2.5 pt}
\caption{Comparison of global statistics (continuous metrics). All standard deviations within 0.1.}
\label{table:synthetic-naiveC}
\begin{center}
\begin{small}
\begin{tabular}{ccl rr rr rr rr rr}
\toprule
$N$ & $E$ & Model 
& \multicolumn{2}{c}{Linear} 
& \multicolumn{2}{c}{NN add.}
& \multicolumn{2}{c}{NN non-add.}
& \multicolumn{2}{c}{Sigmoid}
& \multicolumn{2}{c}{Polynomial} \\
\cmidrule(l{\tabcolsep}){4-5}
\cmidrule(l{\tabcolsep}){6-7}
\cmidrule(l{\tabcolsep}){8-9}
\cmidrule(l{\tabcolsep}){10-11}
\cmidrule(l{\tabcolsep}){12-13}
&&& \multicolumn{1}{c}{mAP $\uparrow$} & \multicolumn{1}{c}{AUC $\uparrow$} & \multicolumn{1}{c}{mAP $\uparrow$} & \multicolumn{1}{c}{AUC $\uparrow$} & \multicolumn{1}{c}{mAP $\uparrow$} & \multicolumn{1}{c}{AUC $\uparrow$} & \multicolumn{1}{c}{mAP $\uparrow$} & \multicolumn{1}{c}{AUC $\uparrow$} & \multicolumn{1}{c}{mAP $\uparrow$} & \multicolumn{1}{c}{AUC $\uparrow$} \\
\midrule
\multirow{3}{*}{10} & \multirow{3}{*}{10} &
\textsc{Corr} & $0.45$& $0.87$& $0.41$& $0.86$& $0.41$& $0.85$& $0.46$& $0.86$& $0.45$& $0.85$\\
&& \textsc{D-Corr} & $0.42$& $0.86$& $0.41$& $0.87$& $0.40$& $0.87$& $0.43$& $0.86$& $0.45$& $0.89$\\
&& \textsc{InvCov} & $0.49$& $0.87$& $0.45$& $0.86$& $0.36$& $0.81$& $0.44$& $0.86$& $0.45$& $0.83$\\
\midrule
\multirow{3}{*}{10} & \multirow{3}{*}{40} &
\textsc{Corr} & $0.47$& $0.53$& $0.47$& $0.52$& $0.46$& $0.52$& $0.48$& $0.53$& $0.48$& $0.54$\\
&& \textsc{D-Corr} & $0.46$& $0.53$& $0.46$& $0.51$& $0.46$& $0.54$& $0.48$& $0.53$& $0.47$& $0.54$\\
&& \textsc{InvCov} & $0.50$& $0.57$& $0.48$& $0.52$& $0.47$& $0.53$& $0.47$& $0.50$& $0.48$& $0.52$\\
\midrule
\multirow{3}{*}{100} & \multirow{3}{*}{100} &
\textsc{Corr} & $0.42$& $0.99$& $0.25$& $0.94$& $0.25$& $0.93$& $0.42$& $0.98$& $0.35$& $0.91$\\
&& \textsc{D-Corr} & $0.41$& $0.99$& $0.25$& $0.96$& $0.26$& $0.96$& $0.41$& $0.98$& $0.37$& $0.94$\\
&& \textsc{InvCov} & $0.40$& $0.99$& $0.22$& $0.94$& $0.16$& $0.87$& $0.40$& $0.97$& $0.36$& $0.90$\\
\midrule
\multirow{3}{*}{100} & \multirow{3}{*}{400} &
\textsc{Corr} & $0.19$& $0.80$& $0.10$& $0.63$& $0.14$& $0.72$& $0.27$& $0.84$& $0.20$& $0.72$\\
&& \textsc{D-Corr} & $0.19$& $0.80$& $0.10$& $0.63$& $0.14$& $0.75$& $0.26$& $0.84$& $0.21$& $0.74$\\
&& \textsc{InvCov} & $0.25$& $0.91$& $0.09$& $0.62$& $0.14$& $0.77$& $0.27$& $0.86$& $0.20$& $0.67$\\
\bottomrule
\end{tabular}
\end{small}
\end{center}
\vskip -0.1in
\end{table*}
\begin{table*}[ht]
\setlength\tabcolsep{3 pt}
\caption{Comparison of global statistics (SHD).
Discretization thresholds for SHD were obtained by computing the $p^\text{th}$ quantile of the computed values,
where $p = 1-(E/N)$.}
\label{table:synthetic-naiveD}
\begin{center}
\begin{small}
\begin{tabular}{ccl r r r r r}
\toprule
$N$ & $E$ & Model 
& Linear & NN add. & NN non-add. & Sigmoid & Polynomial \\
\midrule
\multirow{3}{*}{10} & \multirow{3}{*}{10} &
\textsc{Corr} & $10.6 \scriptstyle \pm 2.8$& $10.2 \scriptstyle \pm 4.6$& $12.0 \scriptstyle \pm 1.9$& $11.1 \scriptstyle \pm 4.3$& $9.9 \scriptstyle \pm 2.8$\\
&& \textsc{D-Corr} & $10.4 \scriptstyle \pm 2.6$& $9.8 \scriptstyle \pm 4.7$& $12.2 \scriptstyle \pm 2.6$& $10.8 \scriptstyle \pm 3.3$& $10.2 \scriptstyle \pm 3.2$\\
&& \textsc{InvCov} & $11.0 \scriptstyle \pm 2.8$& $11.4 \scriptstyle \pm 5.5$& $13.6 \scriptstyle \pm 2.9$& $11.4 \scriptstyle \pm 4.1$& $10.9 \scriptstyle \pm 3.5$\\
\midrule
\multirow{3}{*}{10} & \multirow{3}{*}{40} &
\textsc{Corr} & $39.2 \scriptstyle \pm 2.4$& $38.0 \scriptstyle \pm 1.8$& $38.2 \scriptstyle \pm 0.7$& $38.8 \scriptstyle \pm 3.3$& $38.2 \scriptstyle \pm 2.0$\\
&& \textsc{D-Corr} & $38.8 \scriptstyle \pm 2.0$& $38.8 \scriptstyle \pm 1.5$& $37.0 \scriptstyle \pm 0.6$& $38.9 \scriptstyle \pm 3.2$& $38.0 \scriptstyle \pm 2.0$\\
&& \textsc{InvCov} & $35.8 \scriptstyle \pm 2.3$& $39.2 \scriptstyle \pm 1.5$& $37.6 \scriptstyle \pm 2.7$& $40.7 \scriptstyle \pm 2.2$& $38.4 \scriptstyle \pm 1.2$\\
\midrule
\multirow{3}{*}{100} & \multirow{3}{*}{100} &
\textsc{Corr} & $113.0 \scriptstyle \pm 4.9$& $132.2 \scriptstyle \pm 18.0$& $144.6 \scriptstyle \pm 5.2$& $106.5 \scriptstyle \pm 11.5$& $110.3 \scriptstyle \pm 6.1$\\
&& \textsc{D-Corr} & $113.8 \scriptstyle \pm 5.3$& $133.2 \scriptstyle \pm 17.9$& $144.2 \scriptstyle \pm 6.7$& $108.5 \scriptstyle \pm 11.9$& $109.5 \scriptstyle \pm 5.7$\\
&& \textsc{InvCov} & $124.4 \scriptstyle \pm 8.1$& $130.0 \scriptstyle \pm 17.2$& $158.8 \scriptstyle \pm 6.2$& $112.3 \scriptstyle \pm 14.8$& $106.3 \scriptstyle \pm 4.6$\\
\midrule
\multirow{3}{*}{100} & \multirow{3}{*}{400} &
\textsc{Corr} & $580.4 \scriptstyle \pm 24.5$& $666.0 \scriptstyle \pm 13.5$& $626.2 \scriptstyle \pm 23.4$& $516.5 \scriptstyle \pm 18.5$& $562.5 \scriptstyle \pm 20.1$\\
&& \textsc{D-Corr} & $578.2 \scriptstyle \pm 24.7$& $665.4 \scriptstyle \pm 15.4$& $626.6 \scriptstyle \pm 21.9$& $522.3 \scriptstyle \pm 17.6$& $557.2 \scriptstyle \pm 20.4$\\
&& \textsc{InvCov} & $557.0 \scriptstyle \pm 11.7$& $667.8 \scriptstyle \pm 15.4$& $639.0 \scriptstyle \pm 9.7$& $514.7 \scriptstyle \pm 23.1$& $539.4 \scriptstyle \pm 18.4$\\
\bottomrule
\end{tabular}
\end{small}
\end{center}
\vskip -0.1in
\end{table*}

\subsection{Results on simulated mRNA data}
\label{subsec:sergio}

\begin{table*}[ht]
\setlength\tabcolsep{6 pt}
\caption{Causal discovery results on simulated mRNA data. Each setting encompasses 5 distinct scale-free graphs. Data were generated via SERGIO~\cite{sergio}.}
\label{table:sergio}
\vskip 0.15in
\begin{center}
\begin{small}
\begin{tabular}{lrr rrrr rrrr}
\toprule
$N$ & $E$ & Model & mAP $\uparrow$ & AUC $\uparrow$ & SHD $\downarrow$ & OA $\uparrow$ \\
\midrule

\multirow{5}{*}{10} & \multirow{5}{*}{10} & \textsc{Dcdi-G} & $0.48 \scriptstyle \pm 0.1$& $0.73 \scriptstyle \pm 0.1$& $16.1 \scriptstyle \pm 3.3$& $0.59 \scriptstyle \pm 0.2$\\
&& \textsc{Dcdi-Dsf} & $0.63 \scriptstyle \pm 0.1$& $0.84 \scriptstyle \pm 0.1$& $18.5 \scriptstyle \pm 2.7$& $0.79 \scriptstyle \pm 0.2$\\
&& \textsc{Dcd-Fg} & $0.59 \scriptstyle \pm 0.2$& $0.82 \scriptstyle \pm 0.1$& $81.0 \scriptstyle \pm 0.0$& $0.79 \scriptstyle \pm 0.2$\\
&& \textsc{Avici} & $0.58 \scriptstyle \pm 0.2$& $0.85 \scriptstyle \pm 0.1$& $6.4 \scriptstyle \pm 4.7$& $0.72 \scriptstyle \pm 0.2$\\
\cmidrule(l{\tabcolsep}){3-7}
&& \ours{} \textsc{(Fci)} & $\textbf{0.92} \scriptstyle \pm 0.1$& $\textbf{0.98} \scriptstyle \pm 0.0$& $\textbf{1.9} \scriptstyle \pm 2.0$& $\textbf{0.92} \scriptstyle \pm 0.1$\\
\midrule\midrule

\multirow{5}{*}{10} & \multirow{5}{*}{20} & \textsc{Dcdi-G} & $0.32 \scriptstyle \pm 0.1$& $0.57 \scriptstyle \pm 0.1$& $26.2 \scriptstyle \pm 1.3$& $0.47 \scriptstyle \pm 0.2$\\
&& \textsc{Dcdi-Dsf} & $0.44 \scriptstyle \pm 0.1$& $0.64 \scriptstyle \pm 0.1$& $25.7 \scriptstyle \pm 1.3$& $0.63 \scriptstyle \pm 0.1$\\
&& \textsc{Dcd-Fg} & $0.43 \scriptstyle \pm 0.1$& $0.69 \scriptstyle \pm 0.1$& $73.0 \scriptstyle \pm 0.0$& $0.67 \scriptstyle \pm 0.2$\\
&& \textsc{Avici} & $0.22 \scriptstyle \pm 0.1$& $0.44 \scriptstyle \pm 0.2$& $16.8 \scriptstyle \pm 1.5$& $0.27 \scriptstyle \pm 0.3$\\
\cmidrule(l{\tabcolsep}){3-7}
&& \ours{} \textsc{(Fci)} & $\textbf{0.76} \scriptstyle \pm 0.1$& $\textbf{0.90} \scriptstyle \pm 0.1$& $\textbf{8.8} \scriptstyle \pm 1.5$& $\textbf{0.85} \scriptstyle \pm 0.1$\\
\midrule\midrule

\multirow{5}{*}{20} & \multirow{5}{*}{20} & \textsc{Dcdi-G} & $0.48 \scriptstyle \pm 0.1$& $0.86 \scriptstyle \pm 0.1$& $37.3 \scriptstyle \pm 2.8$& $0.65 \scriptstyle \pm 0.1$\\
&& \textsc{Dcdi-Dsf} & $0.45 \scriptstyle \pm 0.1$& $0.92 \scriptstyle \pm 0.0$& $51.9 \scriptstyle \pm 15.8$& $0.81 \scriptstyle \pm 0.1$\\
&& \textsc{Dcd-Fg} & $0.34 \scriptstyle \pm 0.2$& $0.87 \scriptstyle \pm 0.0$& $361 \scriptstyle \pm 0$& $0.66 \scriptstyle \pm 0.2$\\
&& \textsc{Avici} & $0.32 \scriptstyle \pm 0.2$& $0.78 \scriptstyle \pm 0.1$& $18.7 \scriptstyle \pm 4.9$& $0.66 \scriptstyle \pm 0.2$\\
\cmidrule(l{\tabcolsep}){3-7}
&& \ours{} \textsc{(Fci)} & $\textbf{0.54} \scriptstyle \pm 0.2$& $\textbf{0.94} \scriptstyle \pm 0.0$& $\textbf{16.6} \scriptstyle \pm 3.3$& $\textbf{0.83} \scriptstyle \pm 0.1$\\
\midrule\midrule

\multirow{5}{*}{20} & \multirow{5}{*}{40} & \textsc{Dcdi-G} & $0.31 \scriptstyle \pm 0.1$& $0.65 \scriptstyle \pm 0.1$& $54.7 \scriptstyle \pm 2.7$& $0.49 \scriptstyle \pm 0.1$\\
&& \textsc{Dcdi-Dsf} & $0.40 \scriptstyle \pm 0.1$& $0.71 \scriptstyle \pm 0.1$& $54.6 \scriptstyle \pm 4.4$& $0.63 \scriptstyle \pm 0.1$\\
&& \textsc{Dcd-Fg} & $0.36 \scriptstyle \pm 0.1$& $0.77 \scriptstyle \pm 0.1$& $343 \scriptstyle \pm 0$& $0.67 \scriptstyle \pm 0.1$\\
&& \textsc{Avici} & $0.17 \scriptstyle \pm 0.1$& $0.54 \scriptstyle \pm 0.1$& $37.1 \scriptstyle \pm 1.9$& $0.46 \scriptstyle \pm 0.1$\\
\cmidrule(l{\tabcolsep}){3-7}
&& \ours{} \textsc{(Fci)} & $\textbf{0.50} \scriptstyle \pm 0.1$& $\textbf{0.85} \scriptstyle \pm 0.1$& $\textbf{31.4} \scriptstyle \pm 4.9$& $\textbf{0.78} \scriptstyle \pm 0.1$\\

\bottomrule
\end{tabular}
\end{small}
\end{center}
\vskip -0.1in
\end{table*}

We generated mRNA data using the SERGIO simulator~\cite{sergio}.
We sampled datasets with the Hill coefficient set to $\{0.25, 0.5, 1, 2, 4 \}$ for training, and 2 for testing (2 was default).
We set the decay rate to the default 0.8, and the noise parameter to the default of 1.0.
We sampled 400 graphs for each of $N=\{10, 20\}$ and $E=\{N, 2N\}$.

These data distributions are quite different from typical synthetic datasets, as they simulate steady-state measurements and the data are lower bounded at 0 (gene counts).
Thus, we trained a separate model on these data using the \ours{} \textsc{(Fci)} architecture. Table~\ref{table:sergio} shows that \ours{} performs best across the board.

\subsection{Results and ablation studies on synthetic data}
\label{subsec:additional-results}

\begin{table*}[t]
\setlength\tabcolsep{3.5 pt}
\caption{
SHD on synthetic graphs, observational setting, between predicted vs. true DAG and inferred vs. true CPDAG.
Mean/std over 5 distinct Erd\H{o}s-R\'enyi graphs.
$\dagger$ indicates o.o.d. setting.
$*$ indicates non-parametric bootstrapping.
DAG results from Table~\ref{table:synthetic}.
}
\vspace{-0.15in}
\label{table:cpdag}
\begin{center}
\begin{small}
\begin{tabular}{ccl rr rr rr rr rr}
\toprule
$N$ & $E$ & Model 
& \multicolumn{2}{c}{Linear} 
& \multicolumn{2}{c}{NN add.}
& \multicolumn{2}{c}{Sigmoid$^\dagger$}
& \multicolumn{2}{c}{Polynomial$^\dagger$} 
\\
\cmidrule(l{\tabcolsep}){4-5}
\cmidrule(l{\tabcolsep}){6-7}
\cmidrule(l{\tabcolsep}){8-9}
\cmidrule(l{\tabcolsep}){10-11}
&&& \multicolumn{1}{c}{DAG $\downarrow$} & \multicolumn{1}{c}{CPDAG $\downarrow$} & \multicolumn{1}{c}{DAG $\downarrow$} & \multicolumn{1}{c}{CPDAG $\downarrow$} & \multicolumn{1}{c}{DAG $\downarrow$} & \multicolumn{1}{c}{CPDAG $\downarrow$} & \multicolumn{1}{c}{DAG $\downarrow$} & \multicolumn{1}{c}{CPDAG $\downarrow$}\\
\midrule

\multirow{4}{*}{10} & \multirow{4}{*}{10} & \textsc{DiffAn} & $14.0 \scriptstyle \pm .4$& $16.4 \scriptstyle \pm .4$& $13.6 \scriptstyle \pm 2.9$& $13.6 \scriptstyle \pm 3.4$& $12.0 \scriptstyle \pm .0$& $12.9 \scriptstyle \pm .3$& $15.0 \scriptstyle \pm .1$& $16.3 \scriptstyle \pm .6$\\
&& \textsc{VarSort*} & $6.0 \scriptstyle \pm .4$& $7.0 \scriptstyle \pm .3$& $4.0 \scriptstyle \pm .8$& $4.4 \scriptstyle \pm .2$& $7.6 \scriptstyle \pm .5$& $8.0 \scriptstyle \pm .1$& $9.3 \scriptstyle \pm .5$& $9.5 \scriptstyle \pm .8$\\
&& \textsc{Fci*} & $10.0 \scriptstyle \pm .3$& $11.0 \scriptstyle \pm .8$& $8.2 \scriptstyle \pm .8$& $8.4 \scriptstyle \pm .8$& $9.1 \scriptstyle \pm .5$& $10.3 \scriptstyle \pm .3$& $10.0 \scriptstyle \pm .3$& $10.1 \scriptstyle \pm .2$\\
&& \ours{} \textsc{(Fci)} & $\textbf{1.0} \scriptstyle \pm .1$& $\textbf{1.4} \scriptstyle \pm .7$& $\textbf{3.0} \scriptstyle \pm .2$& $\textbf{2.6} \scriptstyle \pm .3$& $\textbf{3.9} \scriptstyle \pm .9$& $\textbf{3.5} \scriptstyle \pm .3$& $\textbf{6.1} \scriptstyle \pm .7$& $\textbf{6.4} \scriptstyle \pm .1$\\
\midrule
\multirow{4}{*}{20} & \multirow{4}{*}{20} & \textsc{DiffAn} & $40.2 \scriptstyle \pm 4.4$& $41.8 \scriptstyle \pm 5.7$& $38.6 \scriptstyle \pm 3.1$& $39.8 \scriptstyle \pm 5.2$& $19.2 \scriptstyle \pm .6$& $21.6 \scriptstyle \pm .2$& $49.7 \scriptstyle \pm 4.6$& $52.7 \scriptstyle \pm 5.0$\\
&& \textsc{VarSort*} & $10.0 \scriptstyle \pm .4$& $11.0 \scriptstyle \pm .0$& $6.6 \scriptstyle \pm .7$& $7.2 \scriptstyle \pm .5$& $16.1 \scriptstyle \pm .7$& $16.8 \scriptstyle \pm .2$& $17.1 \scriptstyle \pm .1$& $17.7 \scriptstyle \pm .0$\\
&& \textsc{Fci*} & $19.0 \scriptstyle \pm .3$& $21.4 \scriptstyle \pm .3$& $17.4 \scriptstyle \pm .2$& $18.0 \scriptstyle \pm .6$& $18.5 \scriptstyle \pm .5$& $20.6 \scriptstyle \pm .4$& $18.9 \scriptstyle \pm .3$& $19.1 \scriptstyle \pm .2$\\
&& \ours{} \textsc{(Fci)} & $\textbf{3.2} \scriptstyle \pm .6$& $\textbf{2.6} \scriptstyle \pm .1$& $\textbf{5.0} \scriptstyle \pm .8$& $\textbf{4.8} \scriptstyle \pm .7$& $\textbf{6.7} \scriptstyle \pm .1$& $\textbf{6.2} \scriptstyle \pm .1$& $\textbf{9.8} \scriptstyle \pm .2$& $\textbf{11.0} \scriptstyle \pm .2$\\
\bottomrule
\end{tabular}
\end{small}
\end{center}
\vskip -0.2in
\end{table*}

\begin{table*}[t]
\setlength\tabcolsep{3.5 pt}
\caption{
mAP on synthetic graphs, between predicted vs. true DAG and undirected skeletons $E^*$.
Mean/std over 5 distinct Erd\H{o}s-R\'enyi graphs.
$\dagger$ indicates o.o.d. setting.
$*$ indicates non-parametric bootstrapping.
DAG results from Table~\ref{table:synthetic}.
}
\vspace{-0.15in}
\label{table:skeleton}
\begin{center}
\begin{small}
\begin{tabular}{ccl rr rr rr rr rr}
\toprule
$N$ & $E$ & Model 
& \multicolumn{2}{c}{Linear} 
& \multicolumn{2}{c}{NN add.}
& \multicolumn{2}{c}{Sigmoid$^\dagger$}
& \multicolumn{2}{c}{Polynomial$^\dagger$} 
\\
\cmidrule(l{\tabcolsep}){4-5}
\cmidrule(l{\tabcolsep}){6-7}
\cmidrule(l{\tabcolsep}){8-9}
\cmidrule(l{\tabcolsep}){10-11}
&&& \multicolumn{1}{c}{DAG $\uparrow$} & \multicolumn{1}{c}{$E^*$ $\uparrow$} & \multicolumn{1}{c}{DAG $\uparrow$} & \multicolumn{1}{c}{$E^*$ $\uparrow$} & \multicolumn{1}{c}{DAG $\uparrow$} & \multicolumn{1}{c}{$E^*$ $\uparrow$} & \multicolumn{1}{c}{DAG $\uparrow$} & \multicolumn{1}{c}{$E^*$ $\uparrow$}\\
\midrule

\multirow{9}{*}{10} & \multirow{9}{*}{10} & \textsc{Dcdi-G} & $0.74 \scriptstyle \pm .16$& $0.88 \scriptstyle \pm .10$& $0.79 \scriptstyle \pm .12$& $0.85 \scriptstyle \pm .12$& $0.46 \scriptstyle \pm .24$& $0.64 \scriptstyle \pm .20$& $0.41 \scriptstyle \pm .13$& $0.58 \scriptstyle \pm .15$\\
&& \textsc{Dcdi-Dsf} & $0.82 \scriptstyle \pm .20$& $0.93 \scriptstyle \pm .09$& $0.57 \scriptstyle \pm .24$& $0.88 \scriptstyle \pm .11$& $0.38 \scriptstyle \pm .21$& $0.63 \scriptstyle \pm .19$& $0.29 \scriptstyle \pm .13$& $0.61 \scriptstyle \pm .20$\\
&& \textsc{DiffAn} & $0.25 \scriptstyle \pm .06$& $0.54 \scriptstyle \pm .09$& $0.32 \scriptstyle \pm .16$& $0.62 \scriptstyle \pm .19$& $0.24 \scriptstyle \pm .10$& $0.63 \scriptstyle \pm .12$& $0.20 \scriptstyle \pm .08$& $0.50 \scriptstyle \pm .07$\\
&& \textsc{Avici} & $0.45 \scriptstyle \pm .14$& $0.93 \scriptstyle \pm .06$& $0.81 \scriptstyle \pm .15$& $\textbf{0.98} \scriptstyle \pm .03$& $0.52 \scriptstyle \pm .16$& $0.86 \scriptstyle \pm .13$& $0.31 \scriptstyle \pm .06$& $0.75 \scriptstyle \pm .13$\\
\cmidrule(l{\tabcolsep}){3-11}
&& \textsc{VarSort*} & $0.70 \scriptstyle \pm .13$& $0.88 \scriptstyle \pm .07$& $0.76 \scriptstyle \pm .13$& $0.90 \scriptstyle \pm .09$& $0.52 \scriptstyle \pm .24$& $0.88 \scriptstyle \pm .10$& $0.40 \scriptstyle \pm .14$& $0.80 \scriptstyle \pm .09$\\
&& \textsc{Fci*} & $0.52 \scriptstyle \pm .11$& $0.70 \scriptstyle \pm .20$& $0.38 \scriptstyle \pm .20$& $0.69 \scriptstyle \pm .22$& $0.56 \scriptstyle \pm .16$& $0.75 \scriptstyle \pm .18$& $0.41 \scriptstyle \pm .13$& $0.61 \scriptstyle \pm .14$\\
&& \textsc{Gies*} & $0.81 \scriptstyle \pm .12$& $\textbf{1.0} \scriptstyle \pm .00$& $0.61 \scriptstyle \pm .16$& $0.97 \scriptstyle \pm .06$& $0.70 \scriptstyle \pm .14$& $0.98 \scriptstyle \pm .03$& $0.61 \scriptstyle \pm .10$& $\textbf{0.89} \scriptstyle \pm .06$\\
\cmidrule(l{\tabcolsep}){3-11}
&& \ours{} \textsc{(Fci)} & $0.98 \scriptstyle \pm .02$& $\textbf{1.0} \scriptstyle \pm .0$& $0.88 \scriptstyle \pm .09$& $0.96 \scriptstyle \pm .05$& $0.83 \scriptstyle \pm .18$& $\textbf{0.99} \scriptstyle \pm .02$& $0.62 \scriptstyle \pm .09$& $0.87 \scriptstyle \pm .06$\\
&& \ours{} \textsc{(Gies)} & $\textbf{0.99} \scriptstyle \pm .01$& $\textbf{1.0} \scriptstyle \pm .0$& $\textbf{0.94} \scriptstyle \pm .06$& $0.97 \scriptstyle \pm .05$& $\textbf{0.85} \scriptstyle \pm .12$& $0.98 \scriptstyle \pm .03$& $\textbf{0.70} \scriptstyle \pm .11$& $0.86 \scriptstyle \pm .05$\\
\midrule\midrule
\multirow{9}{*}{20} & \multirow{9}{*}{20} & \textsc{Dcdi-G} & $0.59 \scriptstyle \pm .12$& $0.89 \scriptstyle \pm .04$& $0.78 \scriptstyle \pm .07$& $0.93 \scriptstyle \pm .07$& $0.36 \scriptstyle \pm .06$& $0.84 \scriptstyle \pm .05$& $0.42 \scriptstyle \pm .08$& $0.58 \scriptstyle \pm .10$\\
&& \textsc{Dcdi-Dsf} & $0.66 \scriptstyle \pm .16$& $0.91 \scriptstyle \pm .04$& $0.69 \scriptstyle \pm .18$& $0.93 \scriptstyle \pm .10$& $0.37 \scriptstyle \pm .04$& $0.86 \scriptstyle \pm .05$& $0.26 \scriptstyle \pm .08$& $0.55 \scriptstyle \pm .15$\\
&& \textsc{DiffAn} & $0.19 \scriptstyle \pm .09$& $0.47 \scriptstyle \pm .16$& $0.16 \scriptstyle \pm .10$& $0.44 \scriptstyle \pm .15$& $0.29 \scriptstyle \pm .11$& $0.67 \scriptstyle \pm .12$& $0.09 \scriptstyle \pm .03$& $0.32 \scriptstyle \pm .07$\\
&& \textsc{Avici} & $0.48 \scriptstyle \pm .17$& $0.84 \scriptstyle \pm .07$& $0.59 \scriptstyle \pm .09$& $0.85 \scriptstyle \pm .07$& $0.42 \scriptstyle \pm .13$& $0.78 \scriptstyle \pm .10$& $0.24 \scriptstyle \pm .08$& $0.62 \scriptstyle \pm .08$\\
\cmidrule(l{\tabcolsep}){3-11}
&& \textsc{VarSort*} & $0.81 \scriptstyle \pm .08$& $0.93 \scriptstyle \pm .04$& $0.81 \scriptstyle \pm .15$& $0.88 \scriptstyle \pm .14$& $0.50 \scriptstyle \pm .13$& $0.88 \scriptstyle \pm .09$& $0.33 \scriptstyle \pm .13$& $0.76 \scriptstyle \pm .11$\\
&& \textsc{Fci*} & $0.66 \scriptstyle \pm .07$& $0.79 \scriptstyle \pm .07$& $0.42 \scriptstyle \pm .19$& $0.54 \scriptstyle \pm .22$& $0.56 \scriptstyle \pm .08$& $0.67 \scriptstyle \pm .07$& $0.41 \scriptstyle \pm .14$& $0.61 \scriptstyle \pm .13$\\
&& \textsc{Gies*} & $0.84 \scriptstyle \pm .08$& $\textbf{0.99} \scriptstyle \pm .00$& $0.79 \scriptstyle \pm .07$& $0.94 \scriptstyle \pm .04$& $0.71 \scriptstyle \pm .10$& $0.95 \scriptstyle \pm .03$& $0.62 \scriptstyle \pm .09$& $\textbf{0.85} \scriptstyle \pm .09$\\
\cmidrule(l{\tabcolsep}){3-11}
&& \ours{} \textsc{(Fci)} & $\textbf{0.96} \scriptstyle \pm .03$& $\textbf{1.0} \scriptstyle \pm .00$& $0.91 \scriptstyle \pm .04$& $0.95 \scriptstyle \pm .03$& $\textbf{0.85} \scriptstyle \pm .09$& $\textbf{0.97} \scriptstyle \pm .02$& $\textbf{0.69} \scriptstyle \pm .09$& $\textbf{0.86} \scriptstyle \pm .09$\\
&& \ours{} \textsc{(Gies)} & $\textbf{0.97} \scriptstyle \pm .02$& $\textbf{1.0} \scriptstyle \pm .00$& $\textbf{0.94} \scriptstyle \pm .03$& $\textbf{0.96} \scriptstyle \pm .03$& $\textbf{0.84} \scriptstyle \pm .07$& $0.95 \scriptstyle \pm .03$& $\textbf{0.69} \scriptstyle \pm .12$& $0.82 \scriptstyle \pm .11$\\

\bottomrule
\end{tabular}
\end{small}
\end{center}
\vskip -0.2in
\end{table*}
\begin{table*}[ht]
\setlength\tabcolsep{3 pt}
\caption{Scaling to synthetic graphs, larger than those seen in training.
Each setting encompasses 5 distinct Erd\H{o}s-R\'enyi graphs.
All \ours{} runs in this table used $T=500$ subsets of nodes, with $b=500$ examples per batch.
For \textsc{Avici}, we took $M=2000$ samples per dataset (higher than maximum analyzed in their paper), since it performed better than $M=1000$.
Here, the mean AUC values are artificially high due to the high negative rates, as actual edges scale linearly as $N$, while the number of possible edges scales quadratically.}
\label{table:graph-size-scaling}
\begin{center}
\begin{small}
\begin{tabular}{cl rrrr c rrrr}
\toprule
$N$ & Model 
& \multicolumn{4}{c}{Linear, $E=N$} 
&
& \multicolumn{4}{c}{Linear, $E=4N$}\\
\cmidrule(l{\tabcolsep}){3-6}
\cmidrule(l{\tabcolsep}){8-11}
&&\multicolumn{1}{c}{mAP $\uparrow$} & \multicolumn{1}{c}{AUC $\uparrow$} & \multicolumn{1}{c}{SHD $\downarrow$} & \multicolumn{1}{c}{OA $\uparrow$} &
\phantom{x}&
\multicolumn{1}{c}{mAP $\uparrow$} & \multicolumn{1}{c}{AUC $\uparrow$} & \multicolumn{1}{c}{SHD $\downarrow$} & \multicolumn{1}{c}{OA $\uparrow$}
\\
\midrule
\multirow{5}{*}{100} & \textsc{InvCov} & $0.43 \scriptstyle \pm 0.0$& $0.99 \scriptstyle \pm 0.0$& $117 \scriptstyle \pm 7$& ---&& $0.30 \scriptstyle \pm 0.0$& $0.93 \scriptstyle \pm 0.0$& $512 \scriptstyle \pm 11$& ---\\
& \textsc{Corr} & $0.42 \scriptstyle \pm 0.0$& $0.99 \scriptstyle \pm 0.0$& $113 \scriptstyle \pm 5$& ---&& $0.19 \scriptstyle \pm 0.0$& $0.80 \scriptstyle \pm 0.0$& $579 \scriptstyle \pm 25$& ---\\
\cmidrule(l{\tabcolsep}){2-11}
& \textsc{Avici} & $0.03 \scriptstyle \pm 0.0$& $0.43 \scriptstyle \pm 0.1$& $109 \scriptstyle \pm 6$& $0.49 \scriptstyle \pm 0.0$&& $0.11 \scriptstyle \pm 0.0$& $0.55 \scriptstyle \pm 0.1$& $394 \scriptstyle \pm 14$& $0.58 \scriptstyle \pm 0.0$\\
\cmidrule(l{\tabcolsep}){2-11}
& \ours{} \textsc{(Fci)} & $\textbf{0.97} \scriptstyle \pm 0.0$& $\textbf{1.00} \scriptstyle \pm 0.0$& $\textbf{11.6} \scriptstyle \pm 4.3$& $\textbf{0.93} \scriptstyle \pm 0.0$&& $0.88 \scriptstyle \pm 0.0$& $0.98 \scriptstyle \pm 0.0$& $129 \scriptstyle \pm 10$& $0.94 \scriptstyle \pm 0.0$\\
& \ours{} \textsc{(Gies)} & $\textbf{0.97} \scriptstyle \pm 0.0$& $\textbf{1.00} \scriptstyle \pm 0.0$& $12.8 \scriptstyle \pm 4.7$& $0.91 \scriptstyle \pm 0.0$&& $\textbf{0.91} \scriptstyle \pm 0.0$& $\textbf{0.99} \scriptstyle \pm 0.0$& $\textbf{105} \scriptstyle \pm 6$& $\textbf{0.95} \scriptstyle \pm 0.0$\\
\bottomrule
\end{tabular}
\end{small}
\end{center}
\vskip -0.1in
\end{table*}
\begin{table*}[ht]
\setlength\tabcolsep{2.75 pt}
\caption{Causal discovery ablations by setting hidden representations to zero.
Each setting encompasses 5 distinct Erd\H{o}s-R\'enyi graphs.
The symbol $\dagger$ indicates that \ours{}
was not pretrained on this setting.
We set $T=100$.
}
\label{table:synthetic-ablation}
\begin{center}
\begin{small}
\begin{tabular}{ccl rr rr rr rr rr}
\toprule
$N$ & $E$ & Model 
& \multicolumn{2}{c}{Linear} 
& \multicolumn{2}{c}{NN add.}
& \multicolumn{2}{c}{NN non-add.}
& \multicolumn{2}{c}{Sigmoid$^\dagger$}
& \multicolumn{2}{c}{Polynomial$^\dagger$} \\
\cmidrule(l{\tabcolsep}){4-5}
\cmidrule(l{\tabcolsep}){6-7}
\cmidrule(l{\tabcolsep}){8-9}
\cmidrule(l{\tabcolsep}){10-11}
\cmidrule(l{\tabcolsep}){12-13}
&&& \multicolumn{1}{c}{mAP $\uparrow$} & \multicolumn{1}{c}{SHD $\downarrow$} & \multicolumn{1}{c}{mAP $\uparrow$} & \multicolumn{1}{c}{SHD $\downarrow$} & \multicolumn{1}{c}{mAP $\uparrow$} & \multicolumn{1}{c}{SHD $\downarrow$} & \multicolumn{1}{c}{mAP $\uparrow$} & \multicolumn{1}{c}{SHD $\downarrow$} & \multicolumn{1}{c}{mAP $\uparrow$} & \multicolumn{1}{c}{SHD $\downarrow$} \\
\midrule

\multirow{3}{*}{10} & \multirow{3}{*}{10} & \ours{} \textsc{(Fci)} & $0.97$& $1.6$& $0.95$& $2.4$& $0.92$& $2.8$& $0.83$& $3.7$& $0.69$& $6.7$\\
&& $h^\rho \leftarrow 0$ & $0.20$& $29.8$& $0.27$& $22.4$& $0.34$& $24.2$& $0.26$& $22.6$& $0.24$& $26.0$\\
&& $h^E \leftarrow 0$ & $0.61$& $24.0$& $0.71$& $28.8$& $0.76$& $27.8$& $0.49$& $26.7$& $0.51$& $26.7$\\
\cmidrule(l{\tabcolsep}){3-13}
&& \ours{} \textsc{(Gies)} & $0.99$& $1.2$& $0.94$& $2.6$& $0.91$& $3.2$& $0.85$& $4.0$& $0.70$& $5.8$\\
&& $h^\rho \leftarrow 0$ & $0.30$& $29.2$& $0.27$& $29.4$& $0.19$& $29.0$& $0.35$& $27.4$& $0.31$& $27.1$\\
&& $h^E \leftarrow 0$ & $0.85$& $6.4$& $0.82$& $10.2$& $0.78$& $13.2$& $0.63$& $10.2$& $0.55$& $13.4$\\
\bottomrule
\end{tabular}
\end{small}
\end{center}
\vskip -0.1in
\end{table*}
\begin{table*}[ht]
\caption{Our predicted graphs are highly acyclic, on synthetic ER test sets.}
\label{tab:acyclic}
\begin{small}\begin{center}
\begin{tabular}{l rrr}
\toprule
$N$ & Acyclic & Total & Proportion \\
\midrule
10 & 434 & 440 & 0.99 \\
20 & 431 & 440 & 0.98 \\
100 & 433 & 440 & 0.98 \\
\bottomrule
\end{tabular}
\end{center}\end{small}
\vskip -0.2in
\end{table*}
\begin{table*}[ht]
\setlength\tabcolsep{3 pt}
\caption{Full results on synthetic datasets (continuous metrics).
Mean/std over 5 distinct Erd\H{o}s-R\'enyi graphs.
$\dagger$ indicates o.o.d. setting.
$*$ indicates non-parametric bootstrapping.
All standard deviations within 0.03 (most within 0.01).
}
\vspace{-0.1in}
\label{table:synthetic-ER-C}
\begin{center}
\begin{small}
\begin{tabular}{ccl rr rr rr rr rr}
\toprule
$N$ & $E$ & Model 
& \multicolumn{2}{c}{Linear} 
& \multicolumn{2}{c}{NN add.}
& \multicolumn{2}{c}{NN non-add.}
& \multicolumn{2}{c}{Sigmoid$^\dagger$}
& \multicolumn{2}{c}{Polynomial$^\dagger$}
\\
\cmidrule(l{\tabcolsep}){4-5}
\cmidrule(l{\tabcolsep}){6-7}
\cmidrule(l{\tabcolsep}){8-9}
\cmidrule(l{\tabcolsep}){10-11}
\cmidrule(l{\tabcolsep}){12-13}
&&& \multicolumn{1}{c}{mAP $\uparrow$} & \multicolumn{1}{c}{AUC $\uparrow$} & \multicolumn{1}{c}{mAP $\uparrow$} & \multicolumn{1}{c}{AUC $\uparrow$} & \multicolumn{1}{c}{mAP $\uparrow$} & \multicolumn{1}{c}{AUC $\uparrow$} & \multicolumn{1}{c}{mAP $\uparrow$} & \multicolumn{1}{c}{AUC $\uparrow$} & \multicolumn{1}{c}{mAP $\uparrow$} & \multicolumn{1}{c}{AUC $\uparrow$} \\
\midrule

\multirow{12}{*}{10} & \multirow{12}{*}{10} & \textsc{Dcdi-G} & $0.74$& $0.88$& $0.79$& $0.91$& $0.89$& $0.95$& $0.46$& $0.72$& $0.41$& $0.68$\\
&& \textsc{Dcdi-Dsf} & $0.82$& $0.92$& $0.57$& $0.83$& $0.50$& $0.81$& $0.38$& $0.69$& $0.29$& $0.64$\\
&& \textsc{Dcd-Fg} & $0.45$& $0.68$& $0.41$& $0.67$& $0.59$& $0.79$& $0.40$& $0.64$& $0.50$& $0.72$\\
&& \textsc{DiffAn} & $0.25$& $0.73$& $0.32$& $0.70$& $0.12$& $0.51$& $0.24$& $0.70$& $0.20$& $0.65$\\
&& \textsc{Deci} & $0.18$& $0.63$& $0.16$& $0.63$& $0.23$& $0.71$& $0.29$& $0.72$& $0.46$& $0.83$\\
&& \textsc{Avici} & $0.45$& $0.80$& $0.81$& $0.97$& $0.65$& $0.89$& $0.52$& $0.81$& $0.31$& $0.70$\\
\bottomrule
\end{tabular}
\end{small}
\end{center}
\vskip -0.2in
\end{table*}

\begin{table*}[ht]
\setlength\tabcolsep{3 pt}
\caption{Full results on synthetic datasets (discrete metrics).
Mean/std over 5 distinct Erd\H{o}s-R\'enyi graphs.
$\dagger$ indicates o.o.d. setting.
$*$ indicates non-parametric bootstrapping.
All OA standard deviations within 0.2.
}
\vspace{-0.1in}
\label{table:synthetic-ER-D}
\begin{center}
\begin{small}
\begin{tabular}{ccl rr rr rr rr rr}
\toprule
$N$ & $E$ & Model 
& \multicolumn{2}{c}{Linear} 
& \multicolumn{2}{c}{NN add.}
& \multicolumn{2}{c}{NN non-add.}
& \multicolumn{2}{c}{Sigmoid$^\dagger$}
& \multicolumn{2}{c}{Polynomial$^\dagger$}
\\
\cmidrule(l{\tabcolsep}){4-5}
\cmidrule(l{\tabcolsep}){6-7}
\cmidrule(l{\tabcolsep}){8-9}
\cmidrule(l{\tabcolsep}){10-11}
\cmidrule(l{\tabcolsep}){12-13}
&&& \multicolumn{1}{c}{OA $\uparrow$} & \multicolumn{1}{c}{SHD $\downarrow$} & \multicolumn{1}{c}{OA $\uparrow$} & \multicolumn{1}{c}{SHD $\downarrow$} & \multicolumn{1}{c}{OA $\uparrow$} & \multicolumn{1}{c}{SHD $\downarrow$} & \multicolumn{1}{c}{OA $\uparrow$} & \multicolumn{1}{c}{SHD $\downarrow$} & \multicolumn{1}{c}{OA $\uparrow$} & \multicolumn{1}{c}{SHD $\downarrow$} \\
\midrule

\multirow{12}{*}{10} & \multirow{12}{*}{10} & \textsc{Dcdi-G} & $0.73$& $2.8 \scriptstyle \pm 2$& $0.84$& $\textbf{2.2} \scriptstyle \pm 3$& $0.88$& $\textbf{1.0} \scriptstyle \pm 1$& $0.46$& $5.8 \scriptstyle \pm 3$& $0.33$& $8.9 \scriptstyle \pm 6$\\
&& \textsc{Dcdi-Dsf} & $0.81$& $2.0 \scriptstyle \pm 3$& $0.73$& $3.0 \scriptstyle \pm 3$& $0.60$& $4.2 \scriptstyle \pm 1$& $0.43$& $6.3 \scriptstyle \pm 3$& $0.24$& $11.2 \scriptstyle \pm 5$\\
&& \textsc{Dcd-Fg} & $0.50$& $20.4 \scriptstyle \pm 3$& $0.47$& $21.2 \scriptstyle \pm 4$& $0.70$& $19.2 \scriptstyle \pm 4$& $0.43$& $19.8 \scriptstyle \pm 4$& $0.54$& $18.5 \scriptstyle \pm 5$\\
&& \textsc{DiffAn} & $0.61$& $14.0 \scriptstyle \pm 5$& $0.55$& $13.6 \scriptstyle \pm 14$& $0.26$& $21.8 \scriptstyle \pm 8$& $0.53$& $12.0 \scriptstyle \pm 5$& $0.47$& $15.0 \scriptstyle \pm 6$\\
&& \textsc{Deci} & $0.50$& $19.4 \scriptstyle \pm 5$& $0.43$& $13.8 \scriptstyle \pm 6$& $0.62$& $16.2 \scriptstyle \pm 3$& $0.63$& $13.9 \scriptstyle \pm 7$& $0.75$& $7.8 \scriptstyle \pm 4$\\
&& \textsc{Avici} & $0.58$& $8.2 \scriptstyle \pm 4$& $0.79$& $4.2 \scriptstyle \pm 3$& $0.65$& $5.6 \scriptstyle \pm 3$& $0.54$& $8.3 \scriptstyle \pm 3$& $0.35$& $9.6 \scriptstyle \pm 4$\\
\bottomrule
\end{tabular}
\end{small}
\end{center}
\vskip -0.2in
\end{table*}

\begin{table*}[ht]
\vskip -0.1in
\setlength\tabcolsep{3 pt}
\caption{Full results on synthetic datasets (continuous metrics).
Mean over 5 distinct scale-free graphs.
$\dagger$ indicates o.o.d setting.
$*$ indicates non-parametric bootstrapping.
All standard deviations were within 0.02.}
\vspace{-0.1in}
\label{table:synthetic-sf}
\begin{center}
\begin{small}
\begin{tabular}{ccl rr rr rr rr rr}
\toprule
$N$ & $E$ & Model 
& \multicolumn{2}{c}{Linear} 
& \multicolumn{2}{c}{NN add.}
& \multicolumn{2}{c}{NN non-add.}
& \multicolumn{2}{c}{Sigmoid$^\dagger$}
& \multicolumn{2}{c}{Polynomial$^\dagger$} \\
\cmidrule(l{\tabcolsep}){4-5}
\cmidrule(l{\tabcolsep}){6-7}
\cmidrule(l{\tabcolsep}){8-9}
\cmidrule(l{\tabcolsep}){10-11}
\cmidrule(l{\tabcolsep}){12-13}
&&& \multicolumn{1}{c}{mAP $\uparrow$} & \multicolumn{1}{c}{AUC $\uparrow$} & \multicolumn{1}{c}{mAP $\uparrow$} & \multicolumn{1}{c}{AUC $\uparrow$} & \multicolumn{1}{c}{mAP $\uparrow$} & \multicolumn{1}{c}{AUC $\uparrow$} & \multicolumn{1}{c}{mAP $\uparrow$} & \multicolumn{1}{c}{AUC $\uparrow$} & \multicolumn{1}{c}{mAP $\uparrow$} & \multicolumn{1}{c}{AUC $\uparrow$} \\
\midrule
\multirow{12}{*}{10} & \multirow{12}{*}{10} & \textsc{Dcdi-G} & $0.54$& $0.90$& $0.59$& $0.88$& $0.69$& $0.89$& $0.48$& $0.77$& $0.50$& $0.73$\\
&& \textsc{Dcdi-Dsf} & $0.70$& $0.92$& $0.71$& $0.88$& $0.36$& $0.83$& $0.46$& $0.75$& $0.49$& $0.76$\\
&& \textsc{Dcd-Fg} & $0.56$& $0.76$& $0.47$& $0.72$& $0.50$& $0.73$& $0.44$& $0.68$& $0.57$& $0.75$\\
&& \textsc{DiffAn} & $0.25$& $0.73$& $0.15$& $0.66$& $0.16$& $0.62$& $0.31$& $0.75$& $0.24$& $0.63$\\
&& \textsc{Deci} & $0.17$& $0.65$& $0.17$& $0.67$& $0.20$& $0.72$& $0.27$& $0.73$& $0.49$& $0.82$\\
&& \textsc{Avici} & $0.51$& $0.87$& $0.55$& $0.85$& $0.76$& $0.95$& $0.44$& $0.81$& $0.27$& $0.71$\\
\bottomrule
\end{tabular}
\end{small}
\end{center}
\vskip -0.2in
\end{table*}
\begin{table*}[ht]
\vskip -0.1in
\setlength\tabcolsep{2 pt}
\caption{Full results on synthetic datasets (discrete metrics).
Mean/std over 5 distinct scale-free graphs.
$\dagger$ indicates o.o.d. setting.
$*$ indicates non-parametric bootstrapping.
All OA standard deviations within 0.2.}
\vspace{-0.1in}
\label{table:synthetic-sf-D}
\begin{center}
\begin{small}
\begin{tabular}{ccl rr rr rr rr rr}
\toprule
$N$ & $E$ & Model 
& \multicolumn{2}{c}{Linear} 
& \multicolumn{2}{c}{NN add.}
& \multicolumn{2}{c}{NN non-add.}
& \multicolumn{2}{c}{Sigmoid$^\dagger$}
& \multicolumn{2}{c}{Polynomial$^\dagger$} \\
\cmidrule(l{\tabcolsep}){4-5}
\cmidrule(l{\tabcolsep}){6-7}
\cmidrule(l{\tabcolsep}){8-9}
\cmidrule(l{\tabcolsep}){10-11}
\cmidrule(l{\tabcolsep}){12-13}
&&& \multicolumn{1}{c}{OA $\uparrow$} & \multicolumn{1}{c}{SHD $\downarrow$} & \multicolumn{1}{c}{OA $\uparrow$} & \multicolumn{1}{c}{SHD $\downarrow$} & \multicolumn{1}{c}{OA $\uparrow$} & \multicolumn{1}{c}{SHD $\downarrow$} & \multicolumn{1}{c}{OA $\uparrow$} & \multicolumn{1}{c}{SHD $\downarrow$} & \multicolumn{1}{c}{OA $\uparrow$} & \multicolumn{1}{c}{SHD $\downarrow$} \\
\midrule

\multirow{12}{*}{10} & \multirow{12}{*}{10} & \textsc{Dcdi-G} & $0.51$& $16.6 \scriptstyle \pm 1$& $0.61$& $17.4 \scriptstyle \pm 2$& $0.71$& $16.2 \scriptstyle \pm 2$& $0.59$& $16.9 \scriptstyle \pm 5$& $0.63$& $16.6 \scriptstyle \pm 3$\\
&& \textsc{Dcdi-Dsf} & $0.70$& $16.2 \scriptstyle \pm 2$& $0.75$& $15.4 \scriptstyle \pm 3$& $0.36$& $16.8 \scriptstyle \pm 2$& $0.48$& $18.1 \scriptstyle \pm 3$& $0.67$& $18.0 \scriptstyle \pm 2$\\
&& \textsc{Dcd-Fg} & $0.60$& $16.4 \scriptstyle \pm 5$& $0.57$& $22.2 \scriptstyle \pm 4$& $0.57$& $20.0 \scriptstyle \pm 4$& $0.47$& $17.9 \scriptstyle \pm 4$& $0.59$& $16.8 \scriptstyle \pm 5$\\
&& \textsc{DiffAn} & $0.55$& $9.2 \scriptstyle \pm 4$& $0.49$& $14.6 \scriptstyle \pm 4$& $0.36$& $11.6 \scriptstyle \pm 4$& $0.59$& $7.7 \scriptstyle \pm 4$& $0.42$& $14.8 \scriptstyle \pm 6$\\
&& \textsc{Deci} & $0.51$& $17.4 \scriptstyle \pm 5$& $0.55$& $17.4 \scriptstyle \pm 5$& $0.58$& $12.0 \scriptstyle \pm 2$& $0.62$& $12.6 \scriptstyle \pm 5$& $0.74$& $8.2 \scriptstyle \pm 6$\\
&& \textsc{Avici} & $0.64$& $6.8 \scriptstyle \pm 3$& $0.59$& $5.4 \scriptstyle \pm 3$& $0.77$& $3.6 \scriptstyle \pm 1$& $0.48$& $7.8 \scriptstyle \pm 3$& $0.35$& $9.5 \scriptstyle \pm 3$\\
\bottomrule
\end{tabular}
\end{small}
\end{center}
\vskip -0.2in
\end{table*}
\begin{table*}[ht]
\setlength\tabcolsep{3 pt}
\caption{Causal discovery results on synthetic datasets with 100 nodes, continuous metrics.
Each setting encompasses 5 distinct Erd\H{o}s-R\'enyi graphs.
The symbol $\dagger$ indicates that the model
was not trained on this setting.
All standard deviations were within 0.1.}
\label{table:synthetic-100-C}
\begin{center}
\begin{small}
\begin{tabular}{ccl rr rr rr rr rr}
\toprule
$N$ & $E$ & Model 
& \multicolumn{2}{c}{Linear} 
& \multicolumn{2}{c}{NN add.}
& \multicolumn{2}{c}{NN non-add.}
& \multicolumn{2}{c}{Sigmoid$^\dagger$}
& \multicolumn{2}{c}{Polynomial$^\dagger$} \\
\cmidrule(l{\tabcolsep}){4-5}
\cmidrule(l{\tabcolsep}){6-7}
\cmidrule(l{\tabcolsep}){8-9}
\cmidrule(l{\tabcolsep}){10-11}
\cmidrule(l{\tabcolsep}){12-13}
&&& \multicolumn{1}{c}{mAP $\uparrow$} & \multicolumn{1}{c}{AUC $\uparrow$} & \multicolumn{1}{c}{mAP $\uparrow$} & \multicolumn{1}{c}{AUC $\uparrow$} & \multicolumn{1}{c}{mAP $\uparrow$} & \multicolumn{1}{c}{AUC $\uparrow$} & \multicolumn{1}{c}{mAP $\uparrow$} & \multicolumn{1}{c}{AUC $\uparrow$} & \multicolumn{1}{c}{mAP $\uparrow$} & \multicolumn{1}{c}{AUC $\uparrow$} \\
\midrule
\multirow{4}{*}{100} & \multirow{4}{*}{100} &
\textsc{Dcd-Fg} & $0.11$& $0.75$& $0.12$& $0.71$& $0.18$& $0.73$& $0.20$& $0.72$& $0.06$& $0.60$\\
\cmidrule(l{\tabcolsep}){3-13}
&& \textsc{InvCov} & $0.40$& $0.99$& $0.22$& $0.94$& $0.16$& $0.87$& $0.40$& $0.97$& $0.36$& $0.90$\\
\cmidrule(l{\tabcolsep}){3-13}
&& \ours{} \textsc{(Fci)} & $0.96$& $1.00$& $0.83$& $0.97$& $0.75$& $0.97$& $0.79$& $0.97$& $0.56$& $0.88$\\
&& \ours{} \textsc{(Gies)} & $0.97$& $1.00$& $0.82$& $0.98$& $0.74$& $0.96$& $0.80$& $0.97$& $0.54$& $0.85$\\
\midrule\midrule
\multirow{4}{*}{100} & \multirow{4}{*}{400} &
\textsc{Dcd-Fg} & $0.05$& $0.59$& $0.07$& $0.64$& $0.10$& $0.72$& $0.13$& $0.72$& $0.12$& $0.64$\\
\cmidrule(l{\tabcolsep}){3-13}
&& \textsc{InvCov} & $0.25$& $0.91$& $0.09$& $0.62$& $0.14$& $0.77$& $0.27$& $0.86$& $0.20$& $0.67$\\
\cmidrule(l{\tabcolsep}){3-13}
&& \ours{} \textsc{(Fci)} & $0.90$& $0.99$& $0.28$& $0.82$& $0.60$& $0.92$& $0.69$& $0.92$& $0.38$& $0.80$\\
&& \ours{} \textsc{(Gies)} & $0.91$& $0.99$& $0.27$& $0.82$& $0.61$& $0.92$& $0.69$& $0.91$& $0.38$& $0.78$\\
\bottomrule
\end{tabular}
\end{small}
\end{center}
\vskip -0.1in
\end{table*}
\begin{table*}[ht]
\setlength\tabcolsep{3 pt}
\caption{Causal discovery results on synthetic datasets with 100 nodes, discrete metrics.
Each setting encompasses 5 distinct Erd\H{o}s-R\'enyi graphs.
The symbol $\dagger$ indicates that the model
was not trained on this setting.}
\label{table:synthetic-100-D}
\begin{center}
\begin{small}
\begin{tabular}{ccl rr rr rr rr rr}
\toprule
$N$ & $E$ & Model 
& \multicolumn{2}{c}{Linear} 
& \multicolumn{2}{c}{NN add.}
& \multicolumn{2}{c}{NN non-add.}
& \multicolumn{2}{c}{Sigmoid$^\dagger$}
& \multicolumn{2}{c}{Polynomial$^\dagger$} \\
\cmidrule(l{\tabcolsep}){4-5}
\cmidrule(l{\tabcolsep}){6-7}
\cmidrule(l{\tabcolsep}){8-9}
\cmidrule(l{\tabcolsep}){10-11}
\cmidrule(l{\tabcolsep}){12-13}
&&& \multicolumn{1}{c}{OA $\uparrow$} & \multicolumn{1}{c}{SHD $\downarrow$} &\multicolumn{1}{c}{OA $\uparrow$} & \multicolumn{1}{c}{SHD $\downarrow$} &\multicolumn{1}{c}{OA $\uparrow$} & \multicolumn{1}{c}{SHD $\downarrow$} &\multicolumn{1}{c}{OA $\uparrow$} & \multicolumn{1}{c}{SHD $\downarrow$} &\multicolumn{1}{c}{OA $\uparrow$} & \multicolumn{1}{c}{SHD $\downarrow$} \\
\midrule
\multirow{4}{*}{100} & \multirow{4}{*}{100} &
\textsc{Dcd-Fg} & $0.63$& $3075.8$& $0.58$& $2965.0$& $0.60$& $2544.4$& $0.59$& $3808.0$& $0.34$& $1927.9$\\
\cmidrule(l{\tabcolsep}){3-13}
&& \textsc{InvCov} & --- & $124.4$& --- & $130.0$& --- & $158.8$& --- & $112.3$& --- & $106.3$\\
\cmidrule(l{\tabcolsep}){3-13}
&& \ours{} \textsc{(Fci)} & $0.91$& $13.4$& $0.90$& $34.4$& $0.91$& $47.2$& $0.78$& $40.3$& $0.69$& $59.2$\\
&& \ours{} \textsc{(Gies)} & $0.91$& $13.6$& $0.93$& $32.8$& $0.91$& $45.8$& $0.78$& $38.6$& $0.68$& $60.3$\\
\midrule\midrule
\multirow{4}{*}{100} & \multirow{4}{*}{400} &
\textsc{Dcd-Fg} & $0.46$& $3068.2$& $0.60$& $3428.8$& $0.70$& $3510.8$& $0.67$& $3601.8$& $0.53$& $3316.7$\\
\cmidrule(l{\tabcolsep}){3-13}
&& \textsc{InvCov} & --- & $557.0$& --- & $667.8$& --- & $639.0$& --- & $514.7$& --- & $539.4$ \\
\cmidrule(l{\tabcolsep}){3-13}
&& \ours{} \textsc{(Fci)} & $0.93$& $122.0$& $0.90$& $361.2$& $0.91$& $273.2$& $0.87$& $226.9$& $0.82$& $327.0$\\
&& \ours{} \textsc{(Gies)} & $0.94$& $116.6$& $0.91$& $364.4$& $0.92$& $266.8$& $0.87$& $218.3$& $0.84$& $328.0$\\
\bottomrule
\end{tabular}
\end{small}
\end{center}
\vskip -0.1in
\end{table*}

For completeness, we include additional results and analysis on the synthetic datasets.
Table~\ref{table:cpdag} contains our evaluations in the observational setting, with metrics computed over the inferred CPDAG, rather than the predicted DAG.
We find that the SHD varies slightly when evaluating DAGS or converted CPDAGs, but the difference is minimal, compared to the overall gaps in performance between methods.

Tables~\ref{table:synthetic-ER-C} and \ref{table:synthetic-ER-D} compare all baselines across all metrics and graph sizes on Erd\H{o}s-R\'enyi graphs.
Tables~\ref{table:synthetic-sf} and \ref{table:synthetic-sf-D} include the same evaluation on scale-free graphs.
Tables~\ref{table:synthetic-100-C} and \ref{table:synthetic-100-D} assess $N=100$ graphs.

Table~\ref{table:synthetic-ablation} ablates the contribution of the global and marginal features by setting their hidden representations to zero.
Note that our model has never seen this type of input during training, so drops in performance may be conflated with input distributional shift.
Overall, removing the joint statistics ($h^\rho\leftarrow 0$) leads to a higher performance drop than removing the marginal estimates ($h^E\leftarrow 0$). However, the gap between these ablation studies and our final performance may be quite large in some cases, so both inputs are important to the prediction.

Table~\ref{tab:acyclic} shows that despite omitting the DAG constraint, we find that our predicted graphs (test split) are nearly all acyclic, with a naive discretization threshold of 0.5.
Unlike \cite{enco}, which also omits the acyclicity constraint during training but optionally enforces it at inference time, we do not require any post-processing to achieve high performance.
Empirically, we found existing DAG constraints to be unstable (Lagrangian) and slow to optimize~\citep{notears, dcdi}. DAG behavior would not emerge until late in training, when the regularization term is of 1e-8 scale or smaller.

Alternatively, we could quantify the raw information content provided by these two features through the \textsc{InvCov}, \textsc{Fci*}, and \textsc{Gies*} baselines (Tables~\ref{table:synthetic-ER-C}, \ref{table:synthetic-ER-D}, \ref{table:synthetic-sf}, \ref{table:synthetic-sf-D}).
Overall, \textsc{InvCov} and \textsc{Fci*} are comparable to worse-performing baselines. \textsc{Gies*} performs very well, sometimes approaching the strongest baselines. However, there remains a large gap in performance between these ablations and our method, highlighting the value of learning non-linear transformations of these inputs.

Table~\ref{table:graph-size-scaling} and Figure~\ref{fig:graph-size} show that the current implementations of \ours{} can generalize to graphs up to $4\times$ larger than those seen during training.
During training, we did not initially anticipate testing on much larger graphs. As a result, there are two minor issues with the current implementation with respect to scaling.
First, we set an insufficient maximum subset positional embedding size of 500, so it was impossible to encode more subsets.
Second, we did not sample random starting subset indices to ensure that higher-order embeddings are updated equally. Since we never sampled up to 500 subsets during training, these higher-order embeddings were essentially random.
We anticipate that increasing the limit on the number of subsets and ensuring that all embeddings are sufficiently learned will improve the generalization capacity on larger graphs.
Nonetheless, our current model already obtains reasonable performance on larger graphs, out of the box.

Finally, we note that \textsc{Avici} scales very poorly to graphs significantly beyond the scope of their training set. For example, $N=100$ is only $2\times$ their largest training graphs, but the performance already drops dramatically.

Figure \ref{fig:runtime-scaling} depicts the model runtimes. \ours{} continues to run quickly on much larger graphs, while \textsc{Avici} runtimes increase significantly with graph size.

\begin{figure*}[ht]
    \centering
    \includegraphics[height=2.5in]{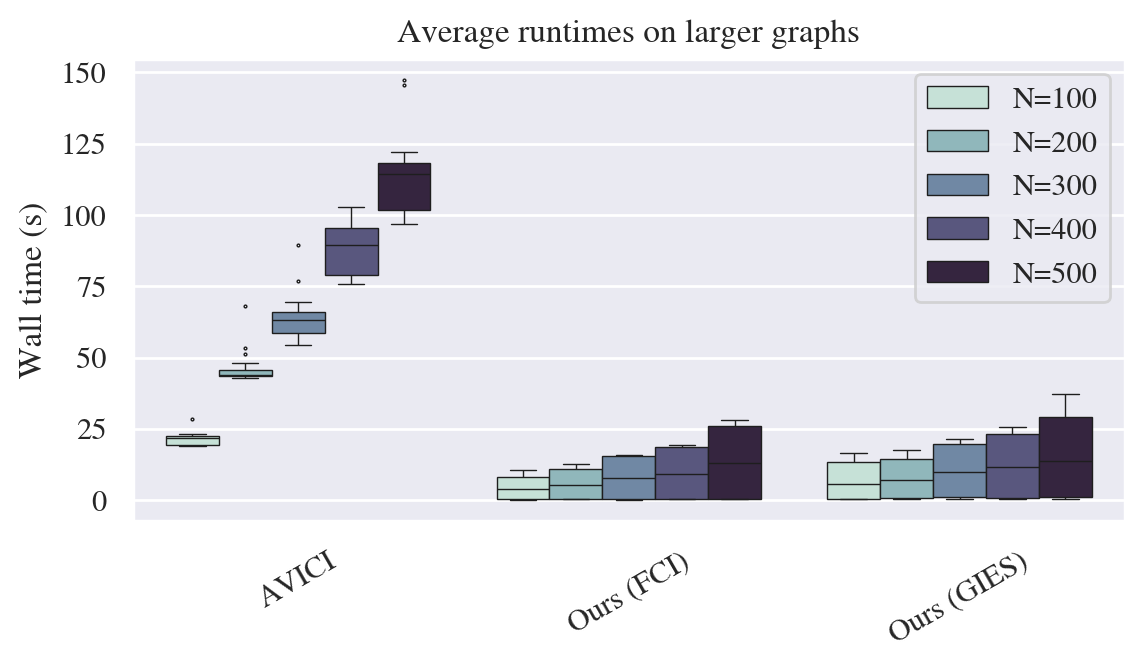}
    \caption{\ours{} scales very well in terms of runtime on much larger graphs, while \textsc{Avici} runtimes suffer as graph sizes increase.}
    \label{fig:runtime-scaling}
\end{figure*}

\textsc{Dcdi} learns a new generative model over each dataset, and its more powerful, deep sigmoidal flow variant seems to perform well in some (but not all) of these harder cases.

\begin{figure*}[ht]
    \centering
    \includegraphics[height=2in]{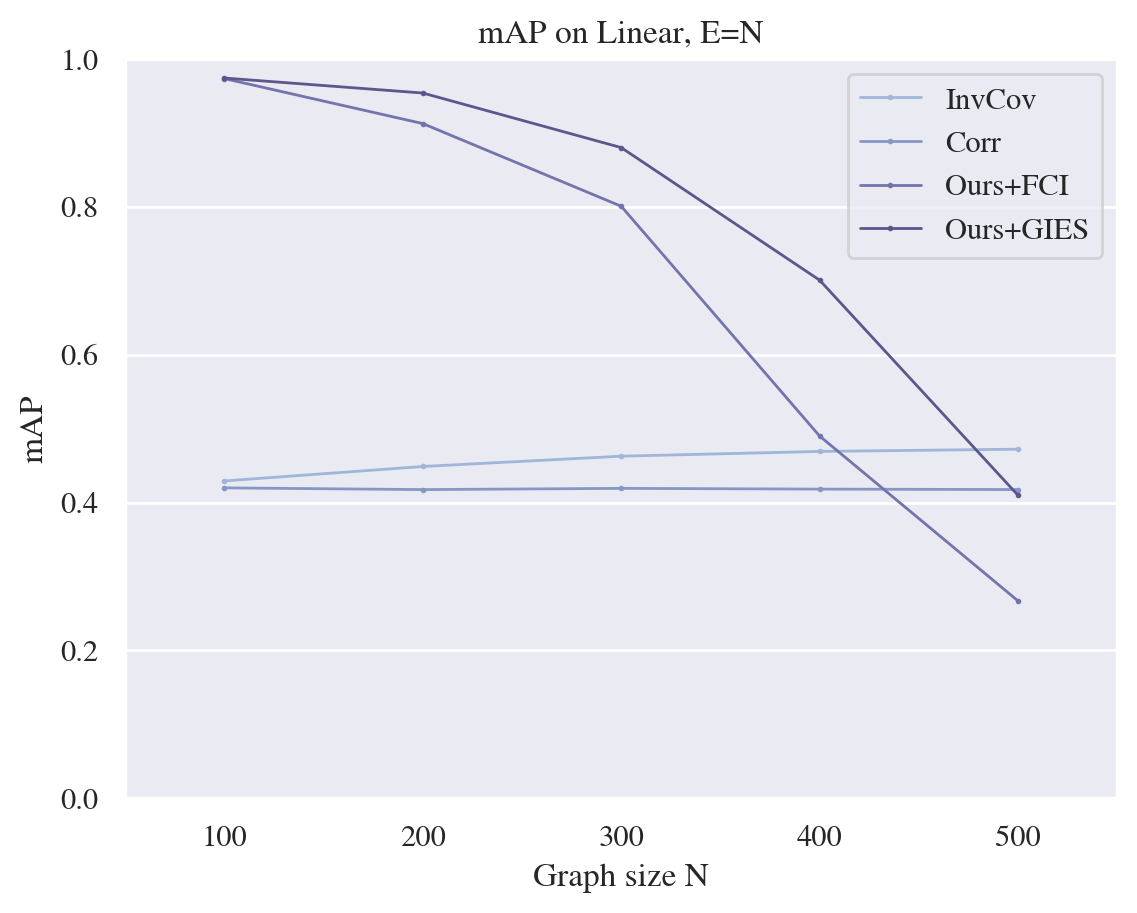}
    \includegraphics[height=2in]{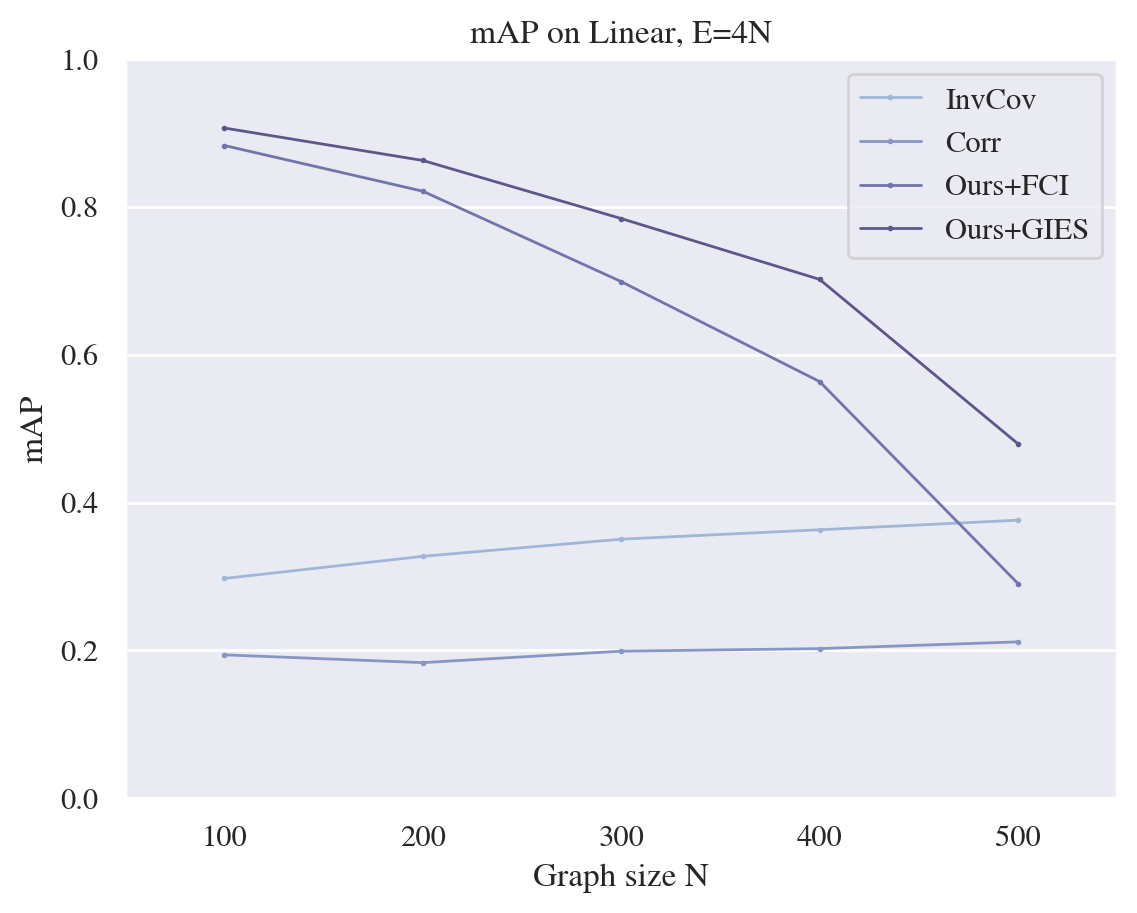}
    \caption{mAP on graphs larger than seen during training.
    During training, we only sampled a maximum of 100 subsets, so performance drop may be due to extrapolation beyond trained embeddings.
    We did not have time to finetune these embeddings for more samples.
    These values correspond to the numbers in Table~\ref{table:graph-size-scaling}.}
    \label{fig:graph-size}
\end{figure*}

\subsection{Results on real datasets}
\label{sec:additional-real}

The Sachs flow cytometry dataset~\citep{sachs} measured the expression of phosphoproteins and phospholipids at the single cell level.
We use the subset proposed by~\cite{igsp}.
The ground truth ``consensus graph'' consists of 11 nodes and 17 edges over 5,845 samples, of which 1,755 are observational and 4,091 are interventional.
The observational data were generated by a ``general perturbation'' which activated signaling pathways, and the interventional data were generated by perturbations intended to target individual proteins.
Despite the popularity of this dataset in causal discovery literature (due to lack of better alternatives), biological networks are known to be time-resolved and cyclic,
so the validity of the ground truth ``consensus'' graph has been questioned by experts~\cite{mooij2020joint}.
Nonetheless, we benchmark all methods on this dataset in Table~\ref{table:sachs-full}.

\begin{table}[ht]
\caption{Swapping to estimation algorithms with significantly different assumptions (LiNGAM) leads to a larger performance drop.
Interestingly, the GIES model seems to depend more on marginal estimates, while the FCI depends more on global statistics.
Results on $N=10, E=10$ observational setting.
}
\setlength\tabcolsep{3 pt}
\vspace{-0.1in}
\label{table:frankenstein2}
\begin{tabular}{ll rr rr rr rr rr}
\toprule

\multirow{2}{1cm}{Train}
&
\multirow{2}{1.5cm}{Inference}
& \multicolumn{2}{c}{Linear} 
& \multicolumn{2}{c}{NN add.}
& \multicolumn{2}{c}{NN non-add.}
& \multicolumn{2}{c}{Sigmoid$^\dagger$}
& \multicolumn{2}{c}{Polynomial$^\dagger$} 
\\
\cmidrule(l{\tabcolsep}){3-4}
\cmidrule(l{\tabcolsep}){5-6}
\cmidrule(l{\tabcolsep}){7-8}
\cmidrule(l{\tabcolsep}){9-10}
\cmidrule(l{\tabcolsep}){11-12}
&&\multicolumn{1}{c}{mAP $\uparrow$} & \multicolumn{1}{c}{SHD $\downarrow$} & \multicolumn{1}{c}{mAP $\uparrow$} & \multicolumn{1}{c}{SHD $\downarrow$} & \multicolumn{1}{c}{mAP $\uparrow$} & \multicolumn{1}{c}{SHD $\downarrow$} & \multicolumn{1}{c}{mAP $\uparrow$} & \multicolumn{1}{c}{SHD $\downarrow$} & \multicolumn{1}{c}{mAP $\uparrow$} & \multicolumn{1}{c}{SHD $\downarrow$} 
\\
\midrule
\multirow{5}{*}{FCI}
&
FCI
& $0.96 \scriptstyle \pm .03$& $3.2 \scriptstyle \pm .6$& $0.91 \scriptstyle \pm .04$& $5.0 \scriptstyle \pm .8$& $0.82 \scriptstyle \pm .05$& $8.8 \scriptstyle \pm .9$& $0.85 \scriptstyle \pm .09$& $6.7 \scriptstyle \pm .1$& $0.69 \scriptstyle \pm .09$& $9.8 \scriptstyle \pm .2$\\
&PC& $0.95 \scriptstyle \pm .04$& $3.6 \scriptstyle \pm .4$& $0.91 \scriptstyle \pm .05$& $4.6 \scriptstyle \pm .6$& $0.81 \scriptstyle \pm .06$& $10.0 \scriptstyle \pm .3$& $0.84 \scriptstyle \pm .07$& $7.4 \scriptstyle \pm .2$& $0.65 \scriptstyle \pm .14$& $10.8 \scriptstyle \pm .5$\\
&GES& $0.94 \scriptstyle \pm .05$& $4.4 \scriptstyle \pm .3$& $0.91 \scriptstyle \pm .05$& $4.2 \scriptstyle \pm .6$& $0.81 \scriptstyle \pm .06$& $9.6 \scriptstyle \pm .6$& $0.81 \scriptstyle \pm .10$& $8.4 \scriptstyle \pm .9$& $0.61 \scriptstyle \pm .16$& $11.5 \scriptstyle \pm .2$\\
&GRaSP& $0.94 \scriptstyle \pm .05$& $4.0 \scriptstyle \pm .0$& $0.91 \scriptstyle \pm .05$& $4.4 \scriptstyle \pm .9$& $0.81 \scriptstyle \pm .06$& $10.0 \scriptstyle \pm .0$& $0.81 \scriptstyle \pm .10$& $8.5 \scriptstyle \pm .0$& $0.61 \scriptstyle \pm .16$& $11.5 \scriptstyle \pm .2$\\
&LiNGAM& $0.88 \scriptstyle \pm .08$& $8.0 \scriptstyle \pm .5$& $0.84 \scriptstyle \pm .05$& $6.0 \scriptstyle \pm .3$& $0.74 \scriptstyle \pm .05$& $10.8 \scriptstyle \pm .3$& $0.71 \scriptstyle \pm .06$& $12.5 \scriptstyle \pm .5$& $0.59 \scriptstyle \pm .12$& $14.4 \scriptstyle \pm .6$\\

\midrule
\multirow{4}{*}{GIES}
&PC& $0.96 \scriptstyle \pm .02$& $1.8 \scriptstyle \pm .7$& $0.91 \scriptstyle \pm .05$& $2.8 \scriptstyle \pm .4$& $0.89 \scriptstyle \pm .10$& $3.2 \scriptstyle \pm .2$& $0.82 \scriptstyle \pm .14$& $4.1 \scriptstyle \pm .5$& $0.58 \scriptstyle \pm .20$& $6.7 \scriptstyle \pm .3$\\
&GES& $0.95 \scriptstyle \pm .03$& $2.0 \scriptstyle \pm .9$& $0.91 \scriptstyle \pm .05$& $2.6 \scriptstyle \pm .1$& $0.88 \scriptstyle \pm .11$& $3.4 \scriptstyle \pm .5$& $0.81 \scriptstyle \pm .15$& $4.1 \scriptstyle \pm .5$& $0.57 \scriptstyle \pm .19$& $6.8 \scriptstyle \pm .1$\\
&GRaSP& $0.95 \scriptstyle \pm .03$& $1.8 \scriptstyle \pm .7$& $0.92 \scriptstyle \pm .05$& $3.0 \scriptstyle \pm .8$& $0.88 \scriptstyle \pm .11$& $3.2 \scriptstyle \pm .2$& $0.81 \scriptstyle \pm .15$& $4.0 \scriptstyle \pm .4$& $0.57 \scriptstyle \pm .19$& $6.9 \scriptstyle \pm .0$\\
&LiNGAM& $0.60 \scriptstyle \pm .14$& $5.6 \scriptstyle \pm .8$& $0.40 \scriptstyle \pm .23$& $9.0 \scriptstyle \pm .6$& $0.40 \scriptstyle \pm .15$& $9.4 \scriptstyle \pm .2$& $0.53 \scriptstyle \pm .15$& $7.2 \scriptstyle \pm .0$& $0.51 \scriptstyle \pm .14$& $7.9 \scriptstyle \pm .0$\\

\bottomrule
\end{tabular}
\end{table}

\begin{table}[t]
\caption{
Complete results on Sachs flow cytometry dataset~\citep{sachs}, using the subset proposed by \citep{igsp}.
}
\label{table:sachs-full}
\setlength\tabcolsep{3 pt}
\begin{small}\begin{center}
\begin{tabular}[b]{l rrr}
\toprule
Model & mAP $\uparrow$ & AUC $\uparrow$ & SHD $\downarrow$ \\
\midrule

\textsc{Dcdi-G} & $0.17$& $0.55$& $21$\\
\textsc{Dcdi-Dsf} & $0.20$& $0.59$& $20$\\
\textsc{Dcd-Fg} & $0.32$& $0.59$& $27$\\
\textsc{DiffAn} & $0.14$& $0.45$& $37$\\
\textsc{Deci} & $0.21$& $0.62$& $28$\\
\textsc{Avici-L} & $0.35$& $0.78$& $20$\\
\textsc{Avici-R} & $0.29$& $0.65$& $18$\\
\textsc{Avici-L+R} & $\textbf{0.59}$& $\textbf{0.83}$& $14$\\
\midrule
\textsc{Fci*} & $0.27$& $0.59$& $18$\\
\textsc{Gies*} & $0.21$& $0.59$& $17$\\
\midrule
\ours{} \textsc{(Fci)} & $0.23$& $0.54$& $24$\\
\textsc{+Kci} & $0.33$& $0.63$& $14$\\
\textsc{+Corr} & $0.41$& $0.70$& $15$\\
\textsc{+Kci+Corr} & $0.49$& $0.71$& $\textbf{13}$\\
\ours{} \textsc{(Gies)} & $0.23$& $0.60$& $14$\\

\bottomrule
\end{tabular}
\end{center}\end{small}
\end{table}

\end{document}